\documentclass{article} 
\usepackage[accepted]{icml2026}

\usepackage{amsmath,amsfonts,bm}

\def\eqref#1{equation~\ref{#1}}

\def\1{\bm{1}}

\DeclareMathAlphabet{\mathsfit}{\encodingdefault}{\sfdefault}{m}{sl}
\SetMathAlphabet{\mathsfit}{bold}{\encodingdefault}{\sfdefault}{bx}{n}

\usepackage[utf8]{inputenc} 
\usepackage[T1]{fontenc}    
\usepackage[colorlinks=true,
            linkcolor=red,
            citecolor=blue,
            filecolor=blue,
            urlcolor=blue]{hyperref}
\usepackage{url}

\usepackage{amssymb}
\usepackage{amsthm}
\usepackage{amsmath, amssymb, mathtools, amsfonts}
\newtheorem{theorem}{Theorem}

\usepackage{caption} 
\usepackage[x11names,svgnames,dvipsnames,table]{xcolor}
\usepackage{subcaption}
\usepackage{thmtools, thm-restate}
\usepackage{cleveref} 
\usepackage{multirow}
\usepackage{enumitem}

\usepackage{amssymb}   
\usepackage{pifont}    
\usepackage{booktabs}
\usepackage{wrapfig} 
\usepackage{setspace} 
\usepackage{booktabs}   
\usepackage{tabularx}   
\usepackage{ragged2e}   
\usepackage{xstring}  
\usepackage{collcell} 
\usepackage{graphicx}
\usepackage{caption}

\usepackage{xcolor}
\usepackage[most]{tcolorbox}

\newtcolorbox{hypothesisbox}{
    enhanced,
    breakable,
    colback=gray!8,
    colframe=black!65,
    boxrule=0.9pt,
    arc=3pt,
    left=10pt,
    right=10pt,
    top=8pt,
    bottom=8pt
}

\newcommand{\Yuhang}[1]{\textcolor{red}{#1}}
\newcommand{\erdun}[1]{\textcolor{red}{#1}}

\newcommand{\cmark}{\ding{51}} 
\newcommand{\xmark}{\ding{55}} 

\makeatletter
 \def\namedlabel#1#2{\begingroup
    #2%
    \def\@currentlabel{#2}%
    \phantomsection\label{#1}\endgroup
}
\makeatother

\newcommand{\encbody}{f_{\text{enc}}}

\newcommand{\decoder}{f_{\text{dec}}}

\newcommand{\schedule}{\operatorname{Schedule}}

\definecolor{yellow}{HTML}{FFD966} 
\definecolor{perturbation}{HTML}{EA9A90}
\definecolor{obs}{HTML}{7789B7}

\crefname{section}{\S}{\S\S}
\crefname{equation}{Eq.}{Equations}
\crefname{appendix}{App.}{Apps.}
\crefname{thm}{Thm.}{Thms.}
\crefname{cor}{Cor.}{Cors.}
\crefname{prop}{Prop.}{Props.}
\crefname{asm}{Asm.}{Asms.}
\crefname{defn}{Defn.}{Defns.}
\crefname{lemma}{Lem.}{Lems.}
\crefname{exm}{Ex.}{Exs.}
\crefname{clar}{Clar.}{Clars.}
\crefname{princ}{Princ.}{Princs.}

\declaretheorem[name=Lemma,numberwithin=section]{lemma}

\declaretheorem[name=Principle,style=definition,numberwithin=section]{princ}

\usepackage[textsize=tiny]{todonotes}

\usepackage{minitoc}

\icmltitlerunning{What Makes a Representation Good for Single-Cell Perturbation Prediction?}

\begin{document}

\twocolumn[
  \icmltitle{What Makes a Representation Good for Single-Cell Perturbation Prediction?}



  \icmlsetsymbol{equal}{*}

  \begin{icmlauthorlist}
    \icmlauthor{Wenkang Jiang}{aiml}
    \icmlauthor{Yuhang Liu}{aiml,rair}
    \icmlauthor{Yichao Cai}{aiml}
    \icmlauthor{Erdun Gao}{aiml,rair}
    \icmlauthor{Jiayi Dong}{fdu}
    \icmlauthor{Ehsan Abbasnejad}{monash}
    \icmlauthor{Lina Yao}{unsw,csiro}
    \icmlauthor{Javen Qinfeng Shi}{aiml,rair}
  \end{icmlauthorlist}

  \icmlaffiliation{aiml}{Australian Institute for Machine Learning, Adelaide University, Australia}
 \icmlaffiliation{rair}{Responsible AI Research Centre, Australia}
  \icmlaffiliation{fdu}{College of Computer Science and Artificial Intelligence, Fudan University, China}
  \icmlaffiliation{monash}{Department of Data Science and AI, Monash University, Australia}
  \icmlaffiliation{unsw}{School of Computer Science and Engineering, University of New South Wales, Australia}
\icmlaffiliation{csiro}{CSIRO, Australia}
  \icmlcorrespondingauthor{Yuhang Liu}{yuhang.liu01@adelaide.edu.au}

  \icmlkeywords{Machine Learning, ICML}

  \vskip 0.3in
]



\printAffiliationsAndNotice{}  

\doparttoc 
\faketableofcontents

\begin{abstract}
Single-cell perturbation modeling is fundamental for understanding and predicting cellular responses to genetic perturbations. However, existing approaches, from causal representation learning to foundation models, often struggle with an overlooked challenge: gene expression is dominated by perturbation-invariant information, while perturbation-specific signals are intrinsically sparse. As a result, learned representations either entangle invariant and perturbation-specific information, leading to spurious and non-generalizable predictors, or suppress perturbation-specific signals altogether, rendering them ineffective for prediction. To address this, we propose \emph{PerturbedVAE}, a general framework designed to resolve this signal imbalance. The framework explicitly separates perturbation-specific information from dominant invariant structure and recovers causal representations to effectively utilize such information for prediction. We further provide an identifiability analysis that characterizes the conditions under which sparse perturbation effects can be reliably recovered, thereby clarifying how the framework can be concretely specified under such conditions. Empirically, \emph{PerturbedVAE} achieves state-of-the-art performance on a widely used benchmark across multiple evaluation settings, yielding significant gains on out-of-distribution combinatorial predictions and uncovering interpretable perturbation-response programs.

\end{abstract}


\section{Introduction}
\label{sec:introduction}

The goal of single-cell perturbation modeling is to predict how a cell’s gene expression changes in response to genetic perturbations~\citep{orgogozo2015differential}, such as CRISPR-based perturbations~\citep{jinek2012programmable,gilbert2014genome,dixit2016perturb,replogle2020combinatorial}. Unlike standard predictive tasks, this problem involves perturbations that actively modify the underlying biological system. As a result, it presents two core challenges: generalization and interpretability. The first concerns generalization to unseen perturbations, e.g., combinatorial ones~\citep{dixit2016perturb,replogle2020combinatorial}. The second concerns understanding how perturbations reshape underlying cellular programs~\citep{huang2024predicting,lotfollahi2023predicting}.

Existing work on single-cell perturbation modeling has primarily focused on two directions: (i) learning causal representations, and (ii) learning general
representations using foundation models (FMs) (See Sec.~\ref{app:relatedwork} for additional discussion). The first direction is motivated by leveraging causal mechanisms to support both generalization and interpretability~\citep{lachapelle2022disentanglement,zhang2023identifiability,lopez2022learning,de2025interpretable}. These approaches typically employ latent causal generative models to capture high-level causal variables that govern gene expression, and aim to recover such variables from data, also known as causal representation learning~\citep{scholkopf2021toward,liu2022identifying}. The second direction is inspired by the success of FMs in various domains, and emphasizes data and model scale to learn powerful, transferable representations~\citep{cui2024scgpt,hao2024large,yang2022scbert,rosen2023universal,theodoris2023transfer}. 

\vspace{0.5em}
\textbf{The Perturbation Suppression Hypothesis (Sec.~\ref{sec: psh}).}
Across both directions above,  an empirical property of single-cell perturbation data is often underemphasized: perturbation-invariant information, such as background cellular programs, typically dominates gene expression, while perturbation-specific information is sparse. This gives rise to what we term the \emph{Perturbation Suppression Hypothesis}: when invariant variation dominates the data, approaches that do not explicitly separate invariant and perturbation-specific information tend to fail systematically. In particular, such approaches either entangle invariant information into perturbation-related representations, undermining their perturbation-specific semantics and leading to poor generalization, as in existing causal representation learning methods, or suppress perturbation-induced signals that are critical for accurate perturbation prediction, as in FM-based approaches.



\vspace{0.25em}
\textbf{Perturbation-Aware Representations (Sec.~\ref{sec: par}).}
Motivated by the perturbation suppression hypothesis, we advocate for learning \emph{perturbation-aware representations}, which explicitly aims to: 1) separate perturbation-specific information from dominant, perturbation-invariant one (\emph{i.e., Extraction}), 2) organize the extracted perturbation-specific information in a manner that supports generalization, and crucially, interpretability (\emph{i.e., Utilization}):

\vspace{0.25em}
\textbf{Contributions.}
Beyond introducing the {perturbation suppression hypothesis} and {perturbation-aware representations} (Sec.~\ref{sec: obs}), we make the following contributions. 
First, guided by learning perturbation-aware representations, we propose a general framework, termed Perturbation-aware Variational Autoencoders (PerturbedVAE), which combines an alignment-based component to isolate perturbation-specific information from dominant, unperturbed cellular programs, with a latent causal model that organizes the isolated information, enabling principled generalization to unseen perturbations (Sec.~\ref{sec:framework}). Second, we provide an identifiability analysis that characterizes the conditions under which perturbation-specific information can be reliably recovered. Importantly, this analysis clarifies how the proposed framework admits theoretically grounded specifications (Sec.~\ref{sec:dgp}). Third, we empirically demonstrate substantial improvements over strong baselines across multiple single-cell perturbation benchmarks, with particularly strong gains in generalization to unseen combinatorial perturbations, while yielding interpretable representations of perturbation effects (Sec.~\ref{sec: exp}).

\section{Perturbation Suppression and Beyond}
\label{sec: obs}

We begin by introducing the {perturbation suppression hypothesis}, motivated by a fundamental yet often overlooked property of single-cell perturbation data: perturbation-invariant information typically dominates gene expression. From this perspective, we show that representations that do not explicitly distinguish perturbation-invariant from perturbation-specific information may fail to effectively utilize or preserve perturbation-induced signals. Finally, we introduce perturbation-aware representation learning to address this hypothesis.

\subsection{Perturbation Suppression Hypothesis}
\label{sec: psh}
Largely because of experimental, environmental, and cost constraints, perturbation-specific signals typically occupy only a small portion of the gene expression space, while unperturbed background programs account for the majority, even in large-scale perturbation datasets. For example, in widely used Perturb-seq benchmark datasets, the number of distinct perturbations is relatively limited (e.g., on the order of $\sim$284 in the dataset of \citet{norman2019exploring}), whereas gene expression profiles span tens of thousands of measured genes (e.g., $\sim$19{,}264 genes \citep{ahlmann2025deep}). In this setting, methods that are not explicitly designed to separate perturbation-invariant and perturbation-specific information tend to exhibit the following failure modes: either the learned perturbation-related representations become entangled with invariant information, or perturbation-induced signals are suppressed, as below.

\textbf{Implications for Learning Causal Representations.}
Although conceptually appealing, the effectiveness of learning causal representations is not closely supported by theoretical guarantees that are typically grounded in identifiability results, which ensure that the underlying causal variables can be recovered up to simple transformations. A key assumption underlying most existing identifiability results is access to sufficiently rich interventional data, i.e., that all latent causal variables
are perturbed across environments~\citep{lachapelle2022disentanglement,zhang2023identifiability,
lopez2022learning,de2025interpretable}.

In practice, however, single-cell perturbation data is often characterized by partial intervention: only a small subset of genes is perturbed, while large portions of genes remain invariant, as mentioned above. This mismatch between theoretical assumptions and available perturbation data in practice encourages invariant background information to be absorbed into perturbation-related (i.e., causal) representations. This may constitute perturbation suppression, whereby invariant information may be mixed into perturbation-related representations when such separation is not explicitly enforced. As a result, constructing causal models on such mixed representations may be inaccurate, leading to limited generalization ability.

\textbf{Implications for FMs.}
FMs for single-cell transcriptomics are typically trained to learn generic representations by emphasizing fitting the overall data distribution. Given the dominance of perturbation-invariant information, such objectives naturally prioritize encoding shared, perturbation-invariant structure. As a consequence, perturbation-specific signals, which occupy only a small fraction of the expression space, may become less accessible in the learned representations.

To empirically assess this effect, we conduct a linear probing experiment on frozen
representations learned by several FMs. Specifically, we treat perturbation conditions as surrogate labels (i.e., variable $\mathbf{u}$ defined in Sec.~\ref{sec: lcm}) and evaluate whether they are linearly decodable from the learned representations, a standard diagnostic for assessing what information is encoded in an accessible form. We compare representations from UCE~\citep{rosen2023universal}, scFoundation~\citep{hao2024large}, and
Geneformer~\citep{theodoris2023transfer} against a Principal Component Analysis (PCA) baseline applied directly to gene expression.

\begin{figure}[h]   
  \centering
\includegraphics[width=0.9\linewidth]{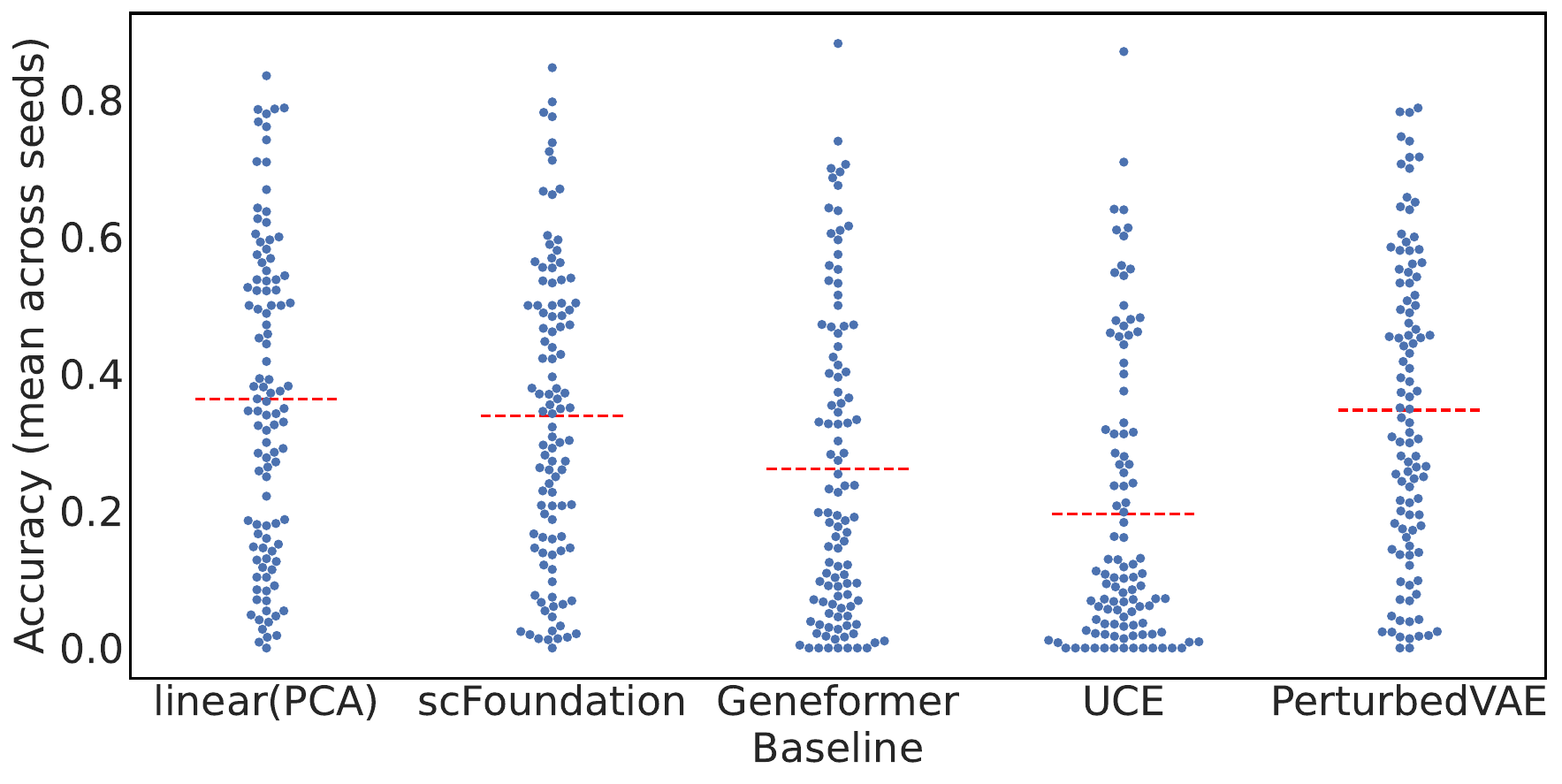}  
  \caption{Linear probe accuracy ($\uparrow$) for predicting perturbation labels. 
FM representations exhibit weaker linear decodability of perturbation labels than a PCA baseline, suggesting limited accessibility of perturbation-related information. In contrast, PerturbedVAE achieves competitive linear probing accuracy, indicating improved accessibility of perturbation-specific signals.}
  \label{fig:collapse}
\end{figure}

As shown in Fig.~\ref{fig:collapse}, representations learned by FMs exhibit weaker linear decodability of perturbation labels than a simple PCA baseline. Since PCA operates directly on the observed gene expression profiles, this result suggests that perturbation-related information is present in the data but becomes less readily accessible after being encoded by generic FM objectives. This observation provides empirical evidence consistent with the perturbation suppression effect in FM representations. In contrast, representations learned by the proposed PerturbedVAE yield competitive linear probing accuracy. See Sec.~\ref{sec:exp_fm} for further results.

\subsection{Learning Perturbation-Aware Representations}
\label{sec: par}
Taken together, the observations above suggest that the central challenge in single-cell perturbation prediction lies in how perturbation-specific information is extracted and utilized. Effective perturbation prediction requires representations that preserve sparse perturbation-induced signals, rather than suppressing them as in FM-based methods, while also preventing invariant structure from becoming entangled with perturbation-related representations, as can occur in learning causal representation methods. Moreover, such representations should organize perturbation-specific information in a form that supports generalization, typically for unseen and combinatorial perturbations.


This motivates us to learn perturbation-aware representations that \emph{explicitly} extract perturbation-specific information from dominant invariant information, i.e., \textit{extraction}, and organize the extracted information in a causally structured representation space for prediction under unseen perturbations, i.e., \textit{utilization}.
\section{PerturbedVAE: A Heuristic Framework for General Perturbation-Aware Modeling}
\label{sec:framework}

Motivated by the notion of \emph{perturbation-aware representations} introduced above, we present PerturbedVAE as a conceptual modeling framework that instantiates the design principles of \textit{Explicit Extraction} and \textit{Effective Utilization}. We emphasize that this section introduces the framework at a high level, motivated by the preceding analysis, rather than a fully specified implementation.

We will study this framework in Sec.~\ref{sec:dgp} from a theoretical perspective and derive conditions under which perturbation-specific information can be reliably identified and utilized. These theoretical insights then guide the concrete instantiation of the framework, ensuring that the resulting model is grounded in principled guarantees.

\subsection{A Latent Causal Generative Model}
\label{sec: lcm}
We begin by introducing the underlying causal generative model, which explicitly formulates perturbation-invariant background programs and perturbation-specific factors, and how they jointly give rise to the observed gene expression data. This model serves as the foundation for designing the proposed framework.

\Cref{fig:lvm-dgp1} illustrates the proposed latent causal generative model for single-cell perturbation data. We assume that the observed gene expression $\mathbf{x}$ is generated from latent variables $\mathbf{z}$ through an unknown nonlinear mapping.
To formulate genetic perturbations, we introduce a surrogate variable $\mathbf{u}$ that indexes the applied perturbation label (e.g., a one-hot encoding). Importantly, we do not assume access to the underlying biochemical intervention mechanism; it suffices to observe which perturbation condition is applied.

To reflect the distinction between unperturbed background cellular programs and perturbation-induced effects, we decompose the latent space into two components:
\begin{itemize}[leftmargin=5pt]
    \item \(\mathbf{z}_\iota\) (\emph{perturbation-invariant variables}), supported on \(\mathcal{Z}_\iota \subseteq \mathbb{R}^{d_\iota}\), represents latent factors that remain stable across perturbation conditions and capture invariant background programs shared by all cells.
    \item \(\mathbf{z}_\nu\) (\emph{perturbation-responsive variables}), supported on \(\mathcal{Z}_\nu \subseteq \mathbb{R}^{d_\nu}\), represents latent factors whose values change in response to genetic perturbations and encode perturbation-induced effects. 
    The components of \(\mathbf{z}_\nu\) are assumed to obey an unknown causal structure, constrained to be a directed acyclic graph (DAG).
\end{itemize}

Following standard causal modeling assumptions, we associate each latent causal variable $\mathbf{z}_i$ with an independent exogenous variable: $\mathbf{n}_\iota$ for $\mathbf{z}_\iota$ and $\mathbf{n}_{\nu,i}$ for each component of $\mathbf{z}_{\nu,i}$, capturing sources of information.

Under the above generative assumptions, the joint prior distribution over the latent variables factorizes as
\begin{equation}
p(\mathbf{z}_\nu, \mathbf{z}_\iota \mid \mathbf{u})
= p(\mathbf{z}_\nu \mid \mathbf{u}, \mathbf{z}_\iota)\, p(\mathbf{z}_\iota),
\label{eq: prior}
\end{equation}
where the perturbation-invariant variables $\mathbf{z}_\iota$ are independent of the perturbation condition $\mathbf{u}$, while the perturbation-responsive variables $\mathbf{z}_\nu$ depend on both the applied perturbation and the background cellular state.

\begin{figure}[t]
  \centering
\    \includegraphics[width=0.65\linewidth]{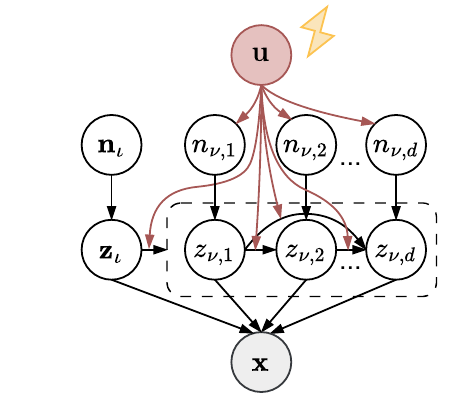}
    \caption{A latent generative model of single-cell perturbations. $\mathbf{z}_\iota$ denotes perturbation-invariant variable and $\mathbf{z}_\nu$ denotes perturbation-responsive variables. }
    \label{fig:lvm-dgp1}
\label{fig:lvms}
\end{figure}

\subsection{Variational Inference Model}
\label{sec: likelihood}
The posterior distribution over the latent variables,
$p(\mathbf{z}_\nu,\mathbf{z}_\iota \mid \mathbf{x},\mathbf{u})$,
induced by the proposed generative model is generally intractable.
We therefore adopt a variational inference to approximate the posterior. Specifically, we introduce the following structured variational distribution:
\begin{equation}
q(\mathbf{z}_\nu,\mathbf{z}_\iota \mid \mathbf{x},\mathbf{u})
= q(\mathbf{z}_\nu \mid \mathbf{x},\mathbf{u})\, q(\mathbf{z}_\iota \mid \mathbf{x}),
\label{eq: poster}
\end{equation}
which reflects our modeling assumptions: $\mathbf{z}_\iota$ captures perturbation-invariant background information inferred from the observed expression $\mathbf{x}$, while $\mathbf{z}_\nu$ captures perturbation-responsive effects conditioned on both $\mathbf{x}$ and the perturbation condition $\mathbf{u}$.

Recall that the perturbation-responsive latent variables $\mathbf{z}_\nu$ are assumed to follow an underlying DAG, which captures the structured dependencies among perturbation effects.
Accordingly, the inference model $q_\theta(\mathbf{z}_\nu \mid \mathbf{x},\mathbf{u})$ in Eq.~\ref{eq: poster} is designed to recover both the latent variables and their dependency structure, so that the learned representations encode not only perturbation effects but also their organization. This structured representation is essential for generalization to unseen perturbations.

\textbf{Evidence Lower Bound.} Given the prior factorization in Eq.~\ref{eq: prior} and the variational posterior in Eq.~\ref{eq: poster}, we optimize the conditional likelihood $\log p(\mathbf{x}\mid \mathbf{u})$ through the following evidence lower bound (ELBO):
\begin{align}
\mathcal{L}_{\text{ELBO}} \; =\; &
\mathbb{E}_{q_{\theta}(\mathbf{z}_\nu, \mathbf{z}_\iota \mid \mathbf{x}, \mathbf{u})}
\big[\log p(\mathbf{x} \mid \mathbf{z}_\nu, \mathbf{z}_\iota,\mathbf{u})\big] \nonumber\\
&- D_{\mathrm{KL}}\!\left(q(\mathbf{z}_\nu, \mathbf{z}_\iota \mid \mathbf{x}, \mathbf{u})\,\|\,p(\mathbf{z}_\nu, \mathbf{z}_\iota\mid \mathbf{u})\right).
\label{eq:elbo}
\end{align}

\textbf{Contrastive Alignment.}
Optimizing the ELBO alone may not resolve the \emph{Perturbation Suppression Hypothesis}: when invariant background programs dominate gene expression data, reconstruction can be achieved primarily by modeling this dominant component.
In this regime, the perturbation-responsive variables $\mathbf{z}_\nu$ may be under-recovered, as their contribution is overwhelmed by invariant background information, leading the learned representations to discard perturbation-induced signals.

To counteract this effect, we introduce a contrastive alignment objective that uses unperturbed controls as an explicit reference for background information.
Specifically, for each perturbed sample $(\mathbf{x}, \mathbf{u})$, we sample a control expression profile $\mathbf{x}^{(\mathbf{u}_0)}$ from the unperturbed condition and encourage their \emph{perturbation-invariant} representations to agree. Concretely, we align the invariant latents inferred from the two samples by minimizing
\begin{equation}
\mathcal{L}_{\text{contrast}}(\mathbf{x}, \mathbf{x}^{(\boldsymbol{u}_0)}) 
= \| \mathbf{z}_\iota - \mathbf{z}_\iota^{(\boldsymbol{u}_0)}\|_2^2 ,
\label{eq: align}
\end{equation}
where $\mathbf{z}_\iota \sim q_\theta(\mathbf{z}_\iota \mid \mathbf{x})$ and $\mathbf{z}_\iota^{(\mathbf{u}_0)} \sim q_\theta(\mathbf{z}_\iota \mid \mathbf{x}^{(\mathbf{u}_0)})$.

Intuitively, by enforcing consistency of $\mathbf{z}_\iota$ across perturbed and unperturbed samples, the alignment enforces dominant, perturbation-invariant variation to be explained through $\mathbf{z}_\iota$. Once this dominant background variation is accounted for in $\mathbf{z}_\iota$, it no longer needs to be explained by other latent variables during reconstruction. Consequently, the perturbation-responsive variables $\mathbf{z}_\nu$ are freed from modeling background structure and are instead driven to capture the residual, perturbation-induced changes.

\begin{figure*}[!t]   
  \centering
  \includegraphics[width=0.9\linewidth]{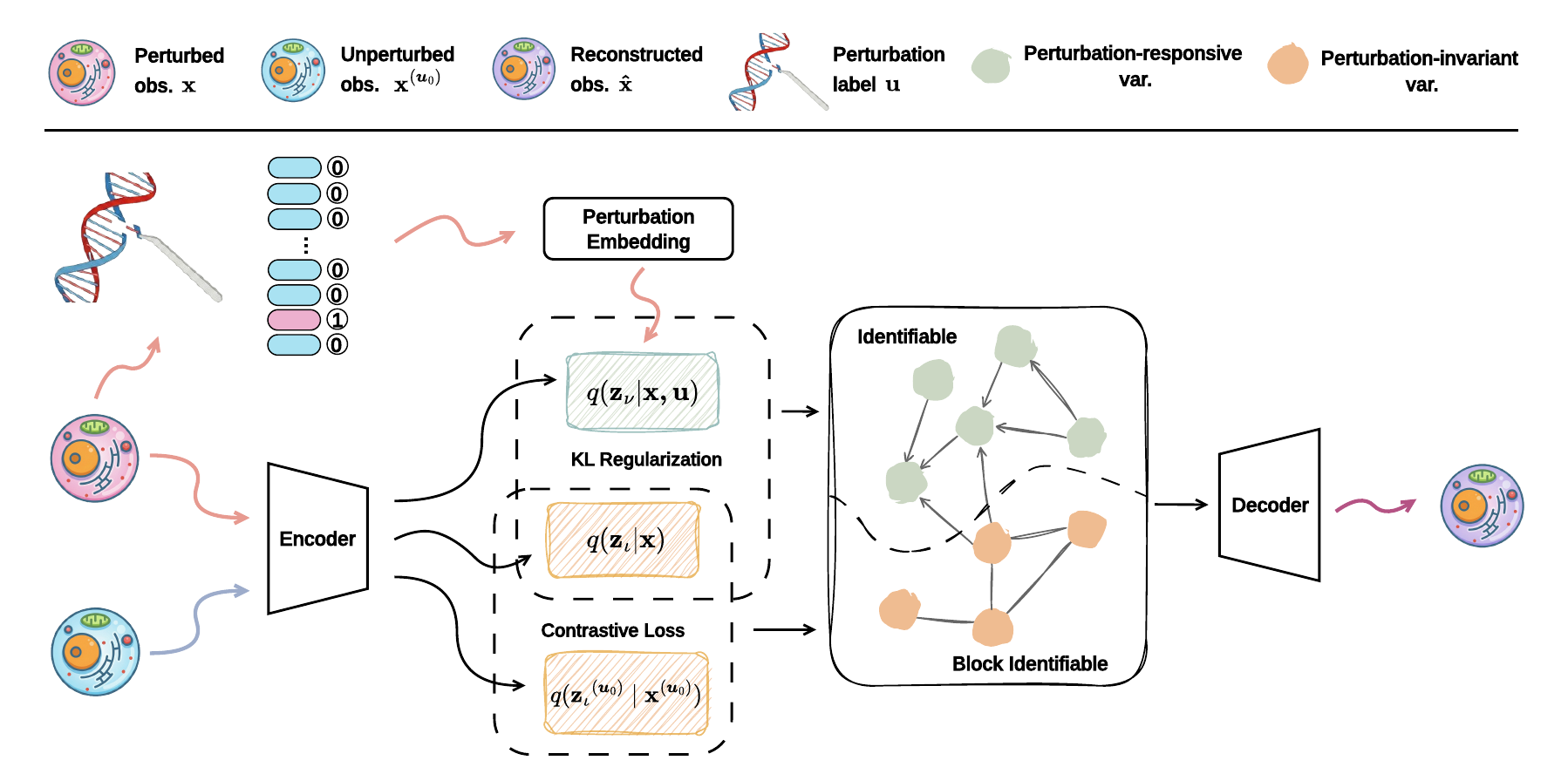}  
  \caption{Framework of the proposed PerturbedVAE. Perturbed $\mathbf{x}$ are used to learn the perturbation-responsive block $\mathbf{z}_\nu$, which captures the effects of perturbations indexed by $\mathbf{u}$. In parallel, unperturbed $\mathbf{x}^{(\boldsymbol{u}_0)}$ are leveraged for contrastive alignment of the perturbation-invariant block $\mathbf{z}_\iota$, encouraging invariant background programs to be separated from perturbation-induced variation.
}
  \label{fig:cdagvae}
\end{figure*}

\subsection{Learning Objective}

The final training objective combines the variational objective with the contrastive alignment term:
\begin{equation}
\mathcal{L}
= - \mathcal{L}_{\mathrm{ELBO}}
+ \alpha\, \mathcal{L}_{\mathrm{contrast}},
\label{eq:final_obj}
\end{equation}
where $\alpha$ controls the strength of the alignment regularization. We refer to the resulting model as \emph{Perturbation-aware Variational Autoencoders} (PerturbedVAE). Figure~\ref{fig:cdagvae} provides an overview of the proposed framework.

\textbf{Objective Interpretation.} The objective in Eq.~\ref{eq:final_obj} balances several complementary goals. The reconstruction term in the ELBO ensures that the latent variables retain sufficient information to explain the observed gene expression. The KL regularization on $\mathbf{z}_\iota$ constrains the perturbation-invariant component, while the KL term on $\mathbf{z}_\nu$ encourages a compact representation of perturbation-responsive variation. The contrastive alignment term explicitly enforces the separation between perturbation-invariant and perturbation-responsive factors by anchoring background variation in $\mathbf{z}_\iota$. Together, these components enable PerturbedVAE to both extract perturbation-induced signals and organize them in a structured latent space that supports generalization to unseen perturbations.

\textbf{Test-time Prediction.}
At test time, PerturbedVAE predicts a target perturbation without using post-perturbation expression as input. Given an unperturbed/control cell $\mathbf{x}^{(\mathbf{u}_0)}$ and a perturbation-condition vector $\mathbf{u}$, we infer the invariant latent variable $\mathbf{z}_\iota \sim q_\phi(\mathbf{z}_\iota \mid \mathbf{x}^{(\mathbf{u}_0)})$, generate the responsive latent variable $\mathbf{z}_\nu$ from the learned perturbation-conditioned mechanism $p_\theta(\mathbf{z}_\nu \mid \mathbf{z}_\iota, \mathbf{u})$, and decode $(\mathbf{z}_\iota,\mathbf{z}_\nu)$ to the predicted expression profile. For double-gene OOD prediction, we directly feed unseen two-hot perturbation vectors into the same learned $\mathbf{u}$-conditioned mechanism trained from single-gene interventions.

\section{A Theoretically-Grounded Construction of the PerturbedVAE Framework}
\label{sec:dgp}

The framework introduced in Sec.~\ref{sec:framework} is deliberately kept implementation-agnostic and does not commit to specific parametric choices (e.g., the concrete form of the prior in Eq.~\ref{eq: prior}). In this section, we provide a formal identifiability analysis that places the framework on a rigorous theoretical footing. By introducing explicit assumptions and parameterizations of the data-generating process, we establish theoretical guarantees for the proposed framework and, in turn, clarify how it can be realized in practice while preserving identifiability guarantees.

Without additional assumptions, exactly recovering latent variables is in general impossible from the observed variables $\mathbf{x}$ and $\mathbf{u}$ alone, even in the simple nonlinear ICA setting~\cite{hyvarinen1999nonlinear,hyvarinen2016unsupervised,hyvarinen2017nonlinear}. To enable the theoretical analysis that follows, we therefore impose additional structure and parameterize the proposed causal generative model as follows.
\begin{align}
    &\mathbf{z}_\iota := \boldsymbol{\lambda}_{\iota\iota}\,\mathbf{z}_\iota + \mathbf{n}_\iota, 
    \qquad \mathbf{n}_\iota \sim \mathcal{N}\!\big(\boldsymbol{\mu}_\iota, \operatorname{diag}\,\boldsymbol{\beta}_\iota\big), 
    \label{eq:lcm_inv}\\
    &\mathbf{z}_\nu := \boldsymbol{\lambda}_{\nu\iota}(\mathbf{u})\,\mathbf{z}_\iota + \boldsymbol{\lambda}_{\nu\nu}(\mathbf{u})\,\mathbf{z}_\nu + \mathbf{n}_\nu, 
    \\
    &\mathbf{n}_\nu \sim \mathcal{N}\!\big(\boldsymbol{\mu}_\nu(\mathbf{u}), \operatorname{diag}\,\boldsymbol{\beta}_\nu(\mathbf{u})\big),
    \label{eq:lcm_nu}\\
    &\mathbf{x} := \mathrm{g}(\mathbf{z}), \label{eq:dgp}
\end{align}
where, 
\begin{itemize}[leftmargin=2em]
    \item $\mathbf{n}_\iota\in\mathbb{R}^{d_\iota}$ and $\mathbf{n}_\nu\in\mathbb{R}^{d_\nu}$ are latent noise variables, sampled from Gaussian, i.e., $\mathcal{N}\!\big(\boldsymbol{\mu}_\iota, \operatorname{diag}\,\boldsymbol{\beta}_\iota\big)$ and $\mathcal{N}\!\big(\boldsymbol{\mu}_\nu(\mathbf{u}), \operatorname{diag}\,\boldsymbol{\beta}_\nu(\mathbf{u})\big)$, respectively.
    \item The matrices $\boldsymbol{\lambda}_{\iota\iota}$, $\boldsymbol{\lambda}_{\nu\nu}(\mathbf{u})$, and $\boldsymbol{\lambda}_{\nu\iota}(\mathbf{u})$ are weight matrices and are assumed to be strictly lower triangular to satisfy the DAG constraint.\footnote{Without loss of generality, we fix a common acyclic ordering across environments, due to permutation indeterminacy in the latent space~\citep{squires2023linear,liu2022identifying}.} Here we assume linear causal relationships among the latent variables, primarily for simplicity and practical applicability, following common practice in prior work~\citep{zhang2023identifiability,de2025interpretable,lopez2022learning}.
    \item $\mathbf{z} = (\mathbf{z}_\iota, \mathbf{z}_\nu)$ and $\mathrm{g}$ denotes an unknown nonlinear mapping from $\mathbf{z}$ to $\mathbf{x}$.
\end{itemize}


Given the parameterization of the latent causal generative model above, we now introduce the following result.

\begin{theorem}[Identifiability Results]
\label{thm: identifiability11}
Suppose the observed variable $\mathbf{x}$ and latent causal variables $\mathbf{z} = (\mathbf{z}_\iota, \mathbf{z}_\nu)$ follow the generative model defined in Eqs.~\ref{eq:lcm_inv}--\ref{eq:dgp}, parameterized by $\boldsymbol{\theta} = (\mathrm{g}, \boldsymbol{\lambda}, \boldsymbol{\mu}, \boldsymbol{\beta})$. Let $\boldsymbol{\hat{\theta}} = (\mathrm{\hat{g}}, \hat{\boldsymbol{\lambda}}, \hat{\boldsymbol{\mu}}, \hat{\boldsymbol{\beta}})$ be the estimated parameters obtained by matching the conditional data distribution $p_\theta(\mathbf{x}|\mathbf{u}) = p_{\hat{\theta}}(\mathbf{x}|\mathbf{u})$ and minimizing the alignment loss. Assume the following conditions hold:

\begin{itemize}[leftmargin=2em]
    \item[\namedlabel{itm1:smooth}{(i)}] \textbf{Invertibility and Smoothness:} The unknown nonlinear mapping $\mathrm{g}$ is smooth and invertible.
    
    \item[\namedlabel{itm1:rank}{(ii)}] \textbf{Environmental Sufficiency:} There exist $2d_\nu$ distinct environments $\{\mathbf{u}_1, \dots, \mathbf{u}_m\}$ relative to a reference $\mathbf{u}_0$ such that the matrix 
    \begin{equation}
        \mathbf{L}^\top = [\Delta\boldsymbol{\eta}(\mathbf{u}_1), \dots, \Delta\boldsymbol{\eta}(\mathbf{u}_{2d_\nu})]^\top \in \mathbb{R}^{2d_\nu \times 2d_\nu}
    \end{equation}
    has full column rank $2d_\nu$, where (elementwise divisions)
\begin{equation}
\Delta\eta(\mathbf{u})
:=
\begin{pmatrix}
\frac{\boldsymbol{\mu}_\nu(\mathbf{u})}{\boldsymbol{\beta}_\nu(\mathbf{u})}
-\frac{\boldsymbol{\mu}_\nu(\mathbf{u}_0)}{\boldsymbol{\beta}_\nu(\mathbf{u}_0)}\\[2mm]
-\frac12\Big(\frac{1}{\boldsymbol{\beta}_\nu(\mathbf{u})}-\frac{1}{\boldsymbol{\beta}_\nu(\mathbf{u}_0)}\Big)
\end{pmatrix}\in\mathbb{R}^{2d_\nu}.
\end{equation}
    \item[\namedlabel{itm1:align}{(iii)}] \textbf{Optimal Alignment:} The alignment loss, e.g., Eq.~\ref{eq: align} attains its global minimum such that $\mathbf{f}_\iota(\mathbf{x}^{(\mathbf{u})}) = \mathbf{f}_\iota(\mathbf{x}^{(\mathbf{u}_0)})$ almost surely for any $\mathbf{u}, \mathbf{u}_0$, where $\mathbf{f}_\iota = \mathrm{\hat{g}}^{-1}_\iota$.
    \item[\namedlabel{itm1:lambda} {(iv)}] \textbf{Intervention Sufficiency:} The function class of $\boldsymbol{\lambda}$ satisfies the following condition: there exists $\mathbf{u}_{i}$, such that, for all parent nodes $z_j \in\mathrm{pa}_i $ of $z_i$, $\boldsymbol{\lambda}_{j,i} =0$.
\end{itemize}
Then, the true latent causal variables $\mathbf{z}$ are related to the variables $\hat{\mathbf{z}}$ estimated by matching likelihood (via the ELBO objective in Eq.~\ref{eq:elbo}) as follows:
\begin{enumerate}[leftmargin=2em]
    \item $\mathbf{z}_\nu$ is identified up to permutation and scaling, i.e., $\mathbf{z}_\nu = \mathbf{P}_\nu \hat{\mathbf{z}}_\nu + \mathbf{c}_\nu$, here $\mathbf{P}_\nu$ is a permutation matrix with scaling.
    \item $\mathbf{z}_\iota$ is identified up to a linear block transformation, i.e., $\mathbf{z}_\iota = \mathbf{A}_\iota \hat{\mathbf{z}}_\iota + \mathbf{c}_\iota$, where $\mathbf{A}_\iota$ is a non-singular matrix.
\end{enumerate}
\end{theorem}
\begin{proof}
    Proof can be found in App.~\ref{app: identi}.
\end{proof}
Refer to Sec.~\ref{app: justify} for a justification of the assumptions.

\textbf{Discussion.}
The identifiability result in Theorem~\ref{thm: identifiability11} characterizes the conditions under which perturbation-specific causal variables can be reliably recovered from partial-intervention data. Importantly, these conditions directly inform how the PerturbedVAE framework is instantiated.

\textbf{Contrastive Loss in Eq.~\ref{eq:final_obj}.} Theorem~\ref{thm: identifiability11} highlights the importance of preventing perturbation-induced variation from being absorbed into invariant representations, by assumption~\ref{itm1:align}. However, this requirement is often overlooked in prior causal representation learning approaches~\citep{zhang2023identifiability,de2025interpretable,lopez2022learning}. As a result, these methods may struggle in practice under partial-intervention settings. Further, assumption~\ref{itm1:align} is aligned with the \emph{contrastive alignment} objective, which explicitly enforces Assumption~\ref{itm1:align} by encouraging the inferred invariant representation $\hat{\mathbf{z}}_\iota = \mathbf{f}_\iota(\mathbf{x})$ to remain identical across perturbed and unperturbed conditions. As a result, environment-dependent information cannot be explained away by $\mathbf{z}_\iota$ and is instead forced to be captured by the perturbation-responsive variables $\mathbf{z}_\nu$.

\textbf{ELBO Loss in Eq.~\ref{eq:final_obj}.} Theorem~\ref{thm: identifiability11} not only clarifies the role of environmental diversity in achieving identifiability, but also directly constrains how the latent variable model should be parameterized, thereby guiding the concrete construction of the ELBO objective in Eq.~\ref{eq:final_obj}. Specifically, under the proposed latent causal generative model (Eqs.~\ref{eq:lcm_inv}-\ref{eq:dgp}), both the prior and posterior of the perturbation-responsive variables $\mathbf{z}_\nu$ are parameterized as linear Gaussian structural causal models conditioned on the perturbation label $\mathbf{u}$. This environment-dependent parameterization is required by the identifiability conditions and is explicitly reflected in the ELBO through perturbation-conditioned distributional parameters. In contrast, the invariant variables $\mathbf{z}_\iota$ are identifiable only up to a linear block transformation. Accordingly, we model $\mathbf{z}_\iota$ using an i.i.d.\ Gaussian form and omit additional structural constraints for simplicity. See Sec.~\ref{app:elbo} for a specific implementation of the proposed PerturbedVAE.

\section{Empirical Findings}
\label{sec: exp}

\subsection{Numerical Simulation}
We first conduct simulations to verify our theoretical results under controllable assumptions. To this end, we generate synthetic data according to our latent causal generative model in Eqs.~\ref{eq:lcm_inv}-~\ref{eq:dgp}. More details can be found in App.~\ref{app:simulation}. This setup allows us to systematically assess the recovery of the latent subspace over $\mathbf{z}_\iota$, and causal structure recovery over latent perturbed variables $\mathbf{z}_\nu$. For evaluation, following \citet{sorrenson2020disentanglement,khemakhem2020variational}, we use the mean correlation coefficient (MCC) to quantify component-wise recovery of $\mathbf{z}_\nu$. Specifically, MCC measures the correlation between each learned component of $\mathbf{z}_\nu$ and its corresponding ground-truth component, with a value of 1 indicating perfect recovery. For block identifiability evaluation of $\mathbf{z}_\iota$, we report the regression $R^2$, following \citet{von2021self}, which measures the correlation between the learned block and its ground-truth counterpart. Values closer to 1 are better.

\begin{table} 
\centering
\caption{Results on simulated data. 
MCC evaluates component-wise recovery of $\mathbf{z}_\nu$, and $R^2$ evaluates block-level recovery of $\mathbf{z}_\iota$. 
Contrastive alignment consistently improves identifiability and disentanglement.}
\resizebox{0.85\linewidth}{!}{
\renewcommand{\arraystretch}{1.2} 
\begin{tabular}{@{}c ccc@{}}
\toprule
\multirow{3}{*}{\begin{tabular}[c]{@{}c@{}}\\\textbf{Contrastive}\\ \textbf{Alignment}\end{tabular}}
  & \textbf{MCC} & \multicolumn{2}{c}{$\boldsymbol{R^2}$} \\
\cmidrule(l){2-4}
  & \begin{tabular}[c]{@{}c@{}}\textbf{Var.} $\mathbf{z}_\nu$\\ (identifiable)\end{tabular} &
    \begin{tabular}[c]{@{}c@{}}\textbf{Var.} $\mathbf{z}_\nu$\\ (block-identifiable)\end{tabular} &
    \begin{tabular}[c]{@{}c@{}}\textbf{Inv.} $\mathbf{z}_\iota$\\ (block-identifiable)\end{tabular} \\
\midrule
\xmark & $0.81_{\pm 0.0306}$ & $0.93_{\pm 0.0120}$ & $0.66_{\pm 0.0281}$ \\
\cmark & $\mathbf{0.86}_{\pm 0.0285}$ & $\mathbf{0.95}_{\pm 0.0020}$ & $\mathbf{0.97}_{\pm 0.0077}$ \\
\bottomrule
\end{tabular}}
\label{tab:simu}
\end{table}

Table~\ref{tab:simu} shows that the contrastive alignment term substantially improves identifiability. For the variant block $\mathbf{z}_\nu$, MCC increases from $0.81$ to $0.86$ and block-wise $R^2$ from $0.93$ to $0.95$, indicating more accurate recovery of intervention-specific factors. The effect is even more pronounced for the invariant block $\mathbf{z}_\iota$, whose $R^2$ rises from $0.66$ to $0.97$, highlighting the crucial role of contrastive alignment in disentangling invariant programs from perturbation-induced effects. These results confirm our theoretical claims as stated in Theorem~\ref{thm: identifiability11}.

\subsection{Experiments on Real Data}
\label{sec:exp_fm}
For real-world perturbation data, we consider the large-scale Perturb-seq dataset from~\citep{norman2019exploring} (See App.~\ref{app: realdata} for more details), following previous works \cite{zhang2023identifiability,de2025interpretable,lopez2022learning,bereket2023modelling}. We design two types of comparisons. First, we compare our method with FMs to evaluate whether the proposed method provides advantages over large pretrained models. Second, we compare with existing methods based on causal representation learning, to assess the benefits of our contrastive and causal modeling components beyond causal modeling alone.

\begin{table}[h]

\centering
\caption{Double-gene perturbation prediction results. RMSE ($\downarrow$) and $R^2$ ($\uparrow$) demonstrate that our method achieves improved generalization to combinatorial perturbations compared to others.}
\setlength{\tabcolsep}{8pt}
\renewcommand{\arraystretch}{1.2}
\resizebox{\linewidth}{!}{
\begin{tabular}{ccc}
\hline
\multirow{2}{*}{\textbf{Method}} & \multicolumn{2}{c}{\textbf{Metrics}}            \\ \cline{2-3} 
                                 & \textbf{RMSE}         & {$\mathbf{R^2}$} \\ \hline
\textbf{scFoundation}~\citep{hao2024large} & $0.5714_{\pm 0.0105}$                    & $0.9844_{\pm 0.006}$                     \\
\textbf{UCE}~\citep{rosen2023universal}               & $0.5634_{\pm 0.0039}$                    & $\underline{0.9857_{\pm 0.0006}}$                    \\
\textbf{Geneformer}~\citep{theodoris2023transfer}          & $0.6132_{\pm 0.0322}$                    & $0.9728_{\pm 0.0015}$                   \\ 
\textbf{STATE}~\citep{adduri2025predicting}
& $0.4981_{\pm 0.0046}$ 
& $0.9475_{\pm 0.0021}$ \\ 
\hline
\textbf{RandomForest}~\citep{breiman2001random}            & $0.4931_{\pm 0.0003}$ & $0.9800_{\pm 0.0005}$   \\
{\textbf{ElasticNet}}~\citep{zou2005regularization}        & $0.4929_{\pm 0.0000}$ & $0.9795_{\pm 0.0000}$   \\
\textbf{KNN}~\citep{cover1967nearest}                     &  $\underline{0.4894_{\pm 0.0000}}$ & $0.9843_{\pm 0.000}$  \\ \hline
\textbf{PerturbedVAE (Ours)}                           & $\mathbf{{0.4474}_{\pm 0.0007}}$ & $\mathbf{{0.9865}_{\pm 0.0009}}$ \\ \hline
\end{tabular}}
\label{tab:FM_performance}
\end{table}

\textbf{Compared With FMs.}
We benchmark {PerturbedVAE} against widely used single-cell FMs, including {scFoundation}~\citep{hao2024large}, {UCE}~\citep{rosen2023universal}, and {Geneformer}~\citep{theodoris2023transfer}, {STATE}~\citep{adduri2025predicting}. For each model, we follow the recommended fine-tuning and inference protocol. All methods are evaluated under the same data splits and metrics. Under cell-wise evaluation, Table~\ref{tab:FM_performance} shows that FMs degrade substantially in the setting of double-gene perturbations, whereas {PerturbedVAE} achieves consistently better performance across five random seeds. Together with our diagnostic analysis in Sec.~\ref{sec: obs} (Figure~\ref{fig:collapse}), these results support that FM representations retain limited perturbation-specific information, which in turn constrains their ability to support accurate perturbation prediction. 
In contrast, the proposed method explicitly preserves perturbation-specific information, and more importantly, effectively utilizes it by causal modeling, leading to improved performance.

\textbf{Compared With Simple Baselines.} We also compared against simple baselines, including RandomForest~\citep{breiman2001random}, 
ElasticNet~\citep{zou2005regularization}, and KNN~\citep{cover1967nearest}. 
Recent benchmarking studies have shown that, in certain settings, FMs do not consistently outperform such simple baselines~\citep{csendes2025benchmarking}. In contrast, our proposed method achieves superior performance compared to these baselines, as shown in Table~\ref{tab:FM_performance}, suggesting that incorporating inductive biases and learning perturbation-aware representations may lead to improved generalization.

\textbf{Perspective on the Additive Linear Baseline.}
Recent work by \citet{ahlmann2025deep} reports that a classical additive linear model can achieve surprisingly strong performance in combinatorial perturbation settings. We therefore compare PerturbedVAE with the additive model, with full results reported in App.~\ref{app:perspective}. Consistent with prior observations~\citep{ahlmann2025deep}, the additive model performs competitively on this dataset, suggesting that many combinatorial perturbation effects are approximately linear and exhibit weak interactions.

At the cell level, however, the additive model exhibits a different trade-off. As shown in Table~\ref{tab:additive_cell_main}, it attains a slightly lower RMSE on genome-wide expression profiles, but its cell-level $R^2$ is negative and therefore not informative. In contrast, PerturbedVAE achieves a valid and high cell-level $R^2$, suggesting better preservation of cell-level explained variance under double-gene OOD prediction.

Nevertheless, such linear compositionality is not guaranteed to hold in general, particularly in the presence of nonlinear gene--gene interactions or context-dependent effects. Our analysis therefore clarifies both the strengths and limitations of simple additive models, and positions PerturbedVAE as a structured alternative that remains applicable when linearity assumptions break down. For completeness, we also report supplementary PCA/scVI-based additive controls in App.~\ref{app:pca_scvi_additive}.

\begin{table}[t]
\centering
\caption{
Comparison with the additive linear baseline on genome-wide expression profiles under double-gene perturbation prediction results.
A dash (--) indicates negative $R^2$.}
\label{tab:additive_cell_main}
\small
\setlength{\tabcolsep}{8pt}
\begin{tabular}{lcc}
\toprule
Method & RMSE $\downarrow$ & $R^2$ $\uparrow$ \\
\midrule
Additive Linear
& $\mathbf{0.4424_{\pm 0.0000}}$
& -- \\
PerturbedVAE
& $0.4494_{\pm 0.0008}$
& $\mathbf{0.9840_{\pm 0.0011}}$ \\
\bottomrule
\end{tabular}
\end{table}

\subsection{Compared with Existing Latent Causal Models}
\textbf{Experimental Setup.}
To compare with existing methods that learn causal representations, we follow the experimental protocol of \citet{zhang2023identifiability} (See App.~\ref{app: realdata} for more details). We benchmark {PerturbedVAE} against four representative baselines, Discrepancy-VAE~\citep{zhang2023identifiability}, SENA-discrepancy-VAE (SENA)~\citep{de2025interpretable}, sVAE+~\citep{lopez2022learning}, {SAMS-VAE}~\citep{bereket2023modelling}, reporting results averaged over five random seeds for each model. We also implement a variant of the proposed {PerturbedVAE}, namely PerturbedVAE (w/o Align), which excludes the contrastive term.

\textbf{Single-Gene Perturbation.} 
We first assess whether the compared latent causal representation models can reproduce observed single-gene perturbation responses. On Norman2019, we evaluate generative fidelity using the 14 single-gene conditions with more than 800 cells. For each condition, we generate 96 synthetic cells from the trained model and compare them to 96 held-out real cells not used during training. Performance is evaluated using population-level $R^2$ and RMSE across all genes (Table~\ref{tab:rmse_single_double}). Our model achieves consistently high fidelity, with an average $R^2$ of 0.99 across the 14 conditions (Left in Figure~\ref{fig:gene_r2_pair}), indicating that it reproduces the mean perturbation response. 

We further evaluate the latent causal model comparison on \citet{replogle2022mapping} as an additional cross-dataset single-gene i.i.d. check. PerturbedVAE again shows the strongest representation-learning performance, with full results in App.~\ref{app:replogle_iid}.

\begin{figure}[H]
  \centering
  \includegraphics[width=0.5\linewidth]{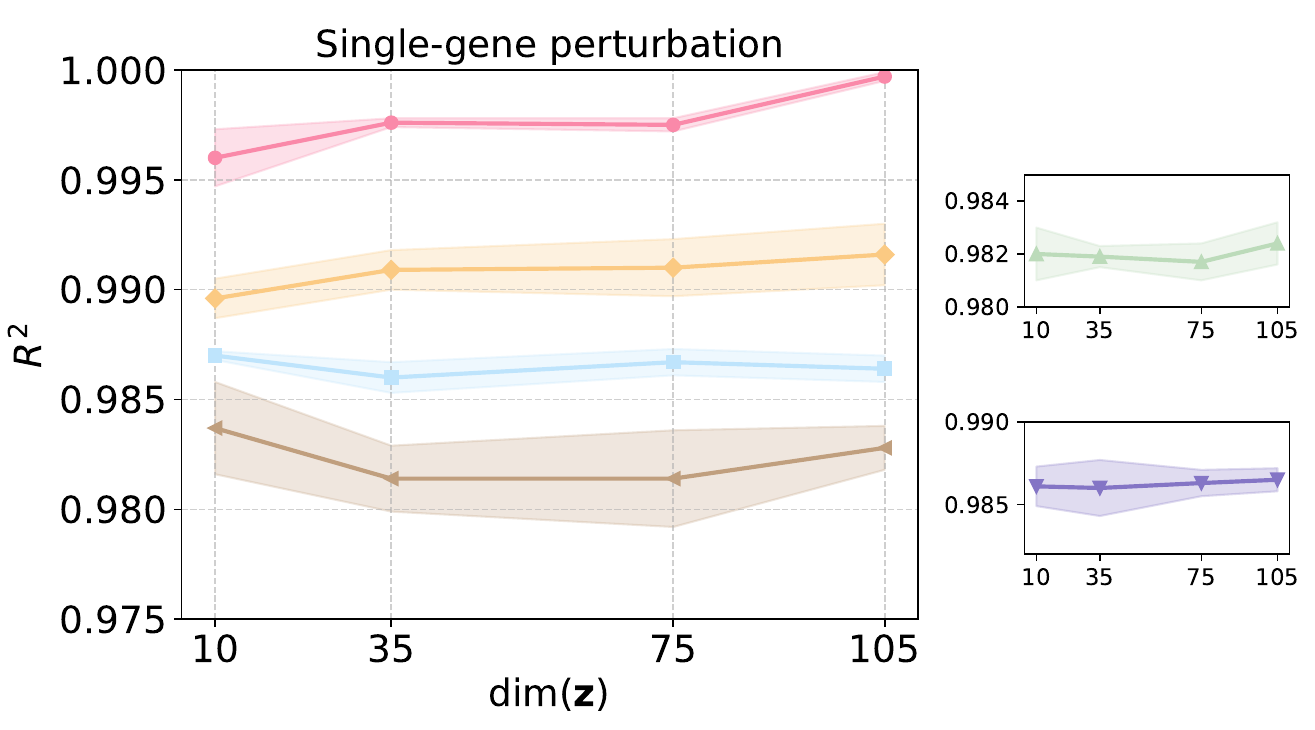}\hfill
  \includegraphics[width=0.5\linewidth]{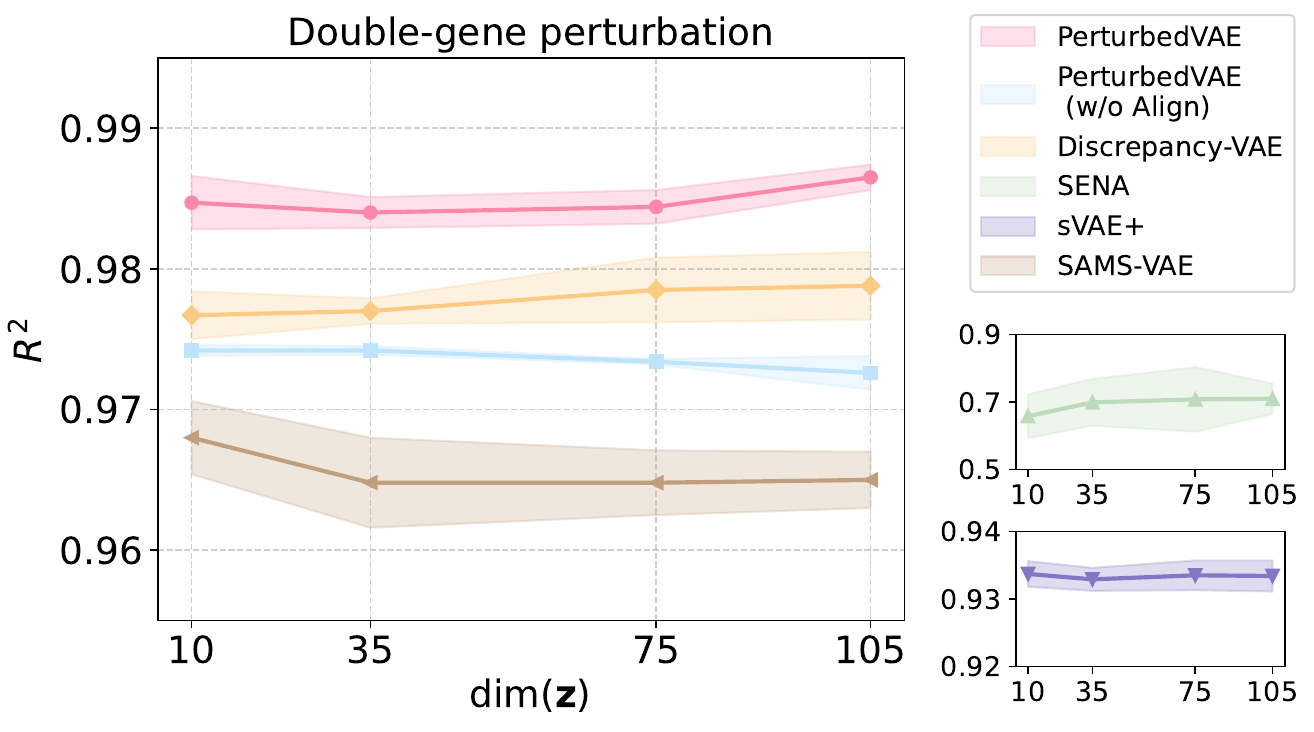}
  \caption{$R^2$ scores for genetic perturbation prediction across different latent dimensions $\mathrm{dim}(\mathbf{z})$.
\textbf{Left:} single-gene perturbations.
\textbf{Right:} double-gene perturbations.
The proposed method consistently achieves higher $R^2$ and remains stable as $\mathrm{dim}(\mathbf{z})$ increases. Shaded areas denote variation across runs.}
  \label{fig:gene_r2_pair}
\end{figure}

\begin{table*}[t]
\centering
\caption{RMSE (lower is better) for single- and double-gene perturbation prediction across different numbers of perturbed genes. Mean ± standard deviation over multiple runs are reported. Compared to existing latent causal representation methods, PerturbedVAE yields consistently lower errors, and the alignment mechanism further improves generalization to unseen double-gene perturbations.}
\label{tab:rmse_single_double}
\footnotesize
\setlength{\tabcolsep}{3pt}
\renewcommand{\arraystretch}{1.15}
\resizebox{\linewidth}{!}{
\begin{tabular}{lcccccccc}
\toprule
\multirow{2}{*}{\textbf{Method}} 
& \multicolumn{4}{c}{\textbf{Single-gene}} 
& \multicolumn{4}{c}{\textbf{Double-gene}} \\
\cmidrule(lr){2-5} \cmidrule(lr){6-9}
& \textbf{10} & \textbf{35} & \textbf{75} & \textbf{105}
& \textbf{10} & \textbf{35} & \textbf{75} & \textbf{105} \\
\midrule
Discrepancy-VAE~\citep{zhang2023identifiability} 
& ${0.5603}_{\pm 0.0030}$ & ${0.5560}_{\pm 0.0027}$ & ${0.5582}_{\pm 0.0038}$ & ${0.5558}_{\pm 0.0022}$
& ${0.6084}_{\pm 0.0045}$ & ${0.6037}_{\pm 0.0025}$ & ${0.6075}_{\pm 0.0072}$ & ${0.6082}_{\pm 0.0045}$ \\

SENA~\citep{de2025interpretable} 
& ${0.5839}_{\pm 0.0021}$ & ${0.5837}_{\pm 0.0086}$ & ${0.5778}_{\pm 0.0109}$ & ${0.5837}_{\pm 0.0074}$
& ${0.8573}_{\pm 0.0205}$ & ${0.8514}_{\pm 0.0248}$ & ${0.8507}_{\pm 0.0396}$ & ${0.8483}_{\pm 0.0248}$ \\

sVAE+~\citep{lopez2022learning} 
& ${0.5012}_{\pm 0.0018}$ & ${0.5005}_{\pm 0.0025}$ & ${0.5003}_{\pm 0.0024}$ & ${0.5002}_{\pm 0.0022}$
& ${0.5663}_{\pm 0.0009}$ & ${0.5667}_{\pm 0.0008}$ & ${0.5665}_{\pm 0.0011}$ & ${0.5664}_{\pm 0.0012}$ \\

SAMS-VAE~\citep{bereket2023modelling} 
& ${0.4114}_{\pm 0.0020}$ & ${0.4136}_{\pm 0.0019}$ & ${0.4140}_{\pm 0.0022}$ & $\underline{{0.4123}_{\pm 0.0029}}$
& ${0.4605}_{\pm 0.0020}$ & ${0.4631}_{\pm 0.0024}$ & ${0.4632}_{\pm 0.0017}$ & ${0.4629}_{\pm 0.0014}$ \\

\midrule
\textbf{PerturbedVAE (w/o Align) (ours)} 
& $\underline{{0.4098}_{\pm 0.0001}}$ & $\underline{{0.4115}_{\pm 0.0008}}$ & $\underline{{0.4115}_{\pm 0.0005}}$ & ${{0.4155}_{\pm 0.0038}}$
& $\underline{{0.4557}_{\pm 0.0005}}$ & $\underline{{0.4563}_{\pm 0.0005}}$ & $\underline{{0.4577}_{\pm 0.0005}}$ & $\underline{{0.4623}_{\pm 0.0041}}$ \\

\textbf{PerturbedVAE (ours)} 
& $\mathbf{{0.4027}_{\pm 0.0028}}$ & $\mathbf{{0.3998}_{\pm 0.0013}}$ & $\mathbf{{0.3997}_{\pm 0.0013}}$ & $\mathbf{{0.3995}_{\pm 0.0013}}$
& $\mathbf{{0.4493}_{\pm 0.0019}}$ & $\mathbf{{0.4494}_{\pm 0.0008}}$ & $\mathbf{{0.4489}_{\pm 0.0009}}$ & $\mathbf{{0.4474}_{\pm 0.0007}}$ \\

\bottomrule
\end{tabular}}
\end{table*}

\textbf{Double-Gene Perturbation.}
We further evaluate 112 double-gene perturbations, which constitute a zero-shot prediction setting as no combinatorial conditions are observed during training. For each perturbation, we compare the population-average expression profile of generated cells with that of the held-out real cells. Despite this challenge, PerturbedVAE achieves strong performance, with an average $R^2$ of 0.98 across all genes (right panel of Figure~\ref{fig:gene_r2_pair}), indicating that it successfully composes knowledge from single-gene interventions to predict unseen combinatorial effects. 
Consistently low RMSE values in Table~\ref{tab:rmse_single_double} further show that the model preserves absolute expression magnitudes under out-of-distribution conditions.

Taken together, these results demonstrate that PerturbedVAE not only outperforms existing methods under observed single-gene perturbations, but also generalizes reliably to unseen combinatorial perturbations, highlighting the benefit of explicitly isolating perturbation-specific signals and aligning with our theoretical analysis.

\subsection{Representation Diagnostics and Interpretability}
\textbf{Unperturbed Latent Subspace.} 
For the invariant block $\mathbf{z}_\iota$, we examine whether its representation remains stable across perturbations, as a qualitative check of the intended block-level invariance property. Specifically, we analyze all perturbation conditions in the test set and visualize the inferred $\mathbf{z}_\iota$ representations, which are provided in App.~\ref{app:latent space}.

\textbf{Perturbed Latent Subspace: Biological Plausibility Check.}
Complementary to the unperturbed latent subspace analysis above, we perform an analysis of perturbed latent space; we first provide the learned causal graph by {PerturbedVAE} to assess whether it captures biologically meaningful regulatory patterns. Following \citet{zhang2023identifiability}, we derive a program-level directed acyclic graph by assigning each latent program to a target gene using a simple heuristic based on maximal intervention effect. The resulting graph (Figure~\ref{fig:structure_gene_a}) recovers several well-established regulatory interactions, including the TGFBR2$\to$SNAI1 axis involved in epithelial–mesenchymal transition (EMT)~\citep{vincent2009snail1, fan2025identification}, the canonical TP73$\to$CDKN1A tumor suppressor pathway governing cell-cycle arrest~\citep{schmidt2021p53}, and the inhibitory regulation of JUN by DUSP9 in MAPK signaling~\citep{emanuelli2008overexpression}. While this analysis is not intended as a comprehensive biological validation, the recovery of these known mechanisms provides supportive evidence that the learned latent structure is not arbitrary and reflects meaningful aspects of underlying regulatory processes, consistent with the goal of learning structured and causally interpretable representations. More details and further analysis can be found in App.~\ref{app:structure learning}.

\textbf{Perturbed Latent Subspace: Information Preservation.}
We further evaluate whether perturbation-specific information is preserved by analyzing performance on the top 20 differentially expressed (DE) genes under perturbation, which form a perturbation-enriched 20-dimensional subspace. This serves as a diagnostic probe of whether perturbation-induced variation is retained rather than suppressed by invariant background programs. Results are reported in App.~\ref{app:de genes}.

\begin{figure}[htbp]
\centering
\includegraphics[width=0.75\linewidth]{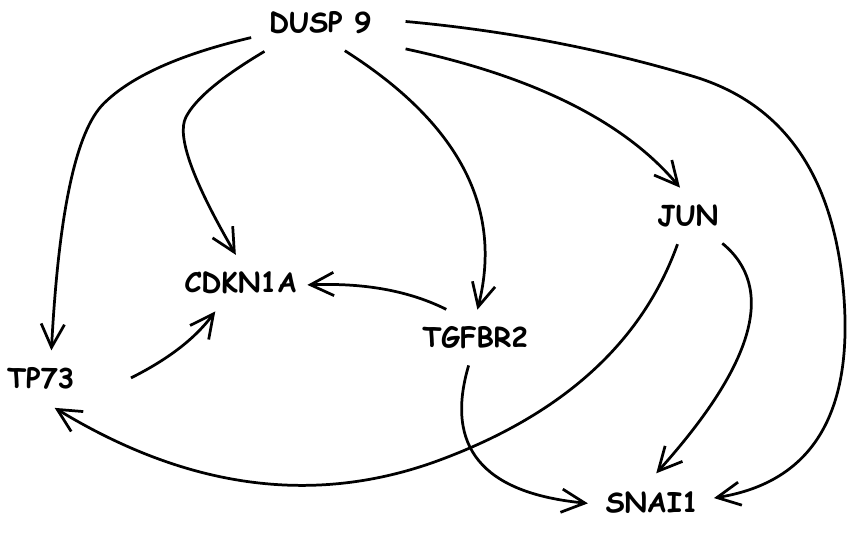}
\caption{Learned causal structure over the perturbation-responsive latent variables $\mathbf{z}_\nu$. Nodes correspond to latent programs assigned to target genes via maximal intervention effect, and directed edges indicate inferred causal dependencies. Several known regulatory relationships are recovered.}
\label{fig:structure_gene_a}
\end{figure}


\section{Conclusion}
We showed that the failure of foundation models in perturbation prediction is not primarily a matter of scale, but stems from a mismatch between generic training objectives and the sparse, intervention-driven structure of perturbation data. Motivated by this, we proposed PerturbedVAE, which introduces explicit inductive biases to disentangle invariant and perturbation-responsive factors and to organize the latter via a latent causal structure. We provided theoretical identifiability results and empirical evidence on both synthetic and real datasets demonstrating accurate prediction and interpretable structure. Our results suggest that a good representation for single-cell perturbation prediction is one that is perturbation-aware and causally structured, enabling reliable generalization under interventions.

\section*{Acknowledgment}
The work is partly supported by the Responsible AI Research Centre and ARC Discovery Project DP240103278. 

\section*{Impact Statement}

This paper presents work whose goal is to advance the field of machine learning. There are many potential societal consequences of our work, none of which we feel must be specifically highlighted here.

\bibliography{reference}

@article{norman2019exploring,
  title={Exploring genetic interaction manifolds constructed from rich single-cell phenotypes},
  author={Norman, Thomas M and Horlbeck, Max A and Replogle, Joseph M and Ge, Alex Y and Xu, Albert and Jost, Marco and Gilbert, Luke A and Weissman, Jonathan S},
  journal={Science},
  volume={365},
  number={6455},
  pages={786--793},
  year={2019},
  publisher={American Association for the Advancement of Science}
}

@article{gilbert2014genome,
  title={Genome-scale CRISPR-mediated control of gene repression and activation},
  author={Gilbert, Luke A and Horlbeck, Max A and Adamson, Britt and Villalta, Jacqueline E and Chen, Yuwen and Whitehead, Evan H and Guimaraes, Carla and Panning, Barbara and Ploegh, Hidde L and Bassik, Michael C and others},
  journal={Cell},
  volume={159},
  number={3},
  pages={647--661},
  year={2014},
  publisher={Elsevier}
}

@article{ahlmann2025deep,
  title={Deep-learning-based gene perturbation effect prediction does not yet outperform simple linear baselines},
  author={Ahlmann-Eltze, Constantin and Huber, Wolfgang and Anders, Simon},
  journal={Nature Methods},
  volume={22},
  number={8},
  pages={1657--1661},
  year={2025},
  publisher={Nature Publishing Group US New York}
}

@article{breiman2001random,
  title={Random forests},
  author={Breiman, Leo},
  journal={Machine learning},
  volume={45},
  number={1},
  pages={5--32},
  year={2001},
  publisher={Springer}
}

@article{cover1967nearest,
  title={Nearest neighbor pattern classification},
  author={Cover, Thomas and Hart, Peter},
  journal={IEEE transactions on information theory},
  volume={13},
  number={1},
  pages={21--27},
  year={1967},
  publisher={IEEE}
}

@article{zou2005regularization,
  title={Regularization and variable selection via the elastic net},
  author={Zou, Hui and Hastie, Trevor},
  journal={Journal of the Royal Statistical Society Series B: Statistical Methodology},
  volume={67},
  number={2},
  pages={301--320},
  year={2005},
  publisher={Oxford University Press}
}

@article{csendes2025benchmarking,
  title={Benchmarking foundation cell models for post-perturbation RNA-seq prediction},
  author={Csendes, Gerold and Sanz, Gema and Szalay, Krist{\'o}f Z and Szalai, Bence},
  journal={BMC genomics},
  volume={26},
  number={1},
  pages={393},
  year={2025},
  publisher={Springer}
}

@article{theodoris2023transfer,
  title={Transfer learning enables predictions in network biology},
  author={Theodoris, Christina V and Xiao, Ling and Chopra, Anant and Chaffin, Mark D and Al Sayed, Zeina R and Hill, Matthew C and Mantineo, Helene and Brydon, Elizabeth M and Zeng, Zexian and Liu, X Shirley and others},
  journal={Nature},
  volume={618},
  number={7965},
  pages={616--624},
  year={2023},
  publisher={Nature Publishing Group UK London}
}

@article{yang2022scbert,
  title={scBERT as a large-scale pretrained deep language model for cell type annotation of single-cell RNA-seq data},
  author={Yang, Fan and Wang, Wenchuan and Wang, Fang and Fang, Yuan and Tang, Duyu and Huang, Junzhou and Lu, Hui and Yao, Jianhua},
  journal={Nature Machine Intelligence},
  volume={4},
  number={10},
  pages={852--866},
  year={2022},
  publisher={Nature Publishing Group UK London}
}

@article{hao2024large,
  title={Large-scale foundation model on single-cell transcriptomics},
  author={Hao, Minsheng and Gong, Jing and Zeng, Xin and Liu, Chiming and Guo, Yucheng and Cheng, Xingyi and Wang, Taifeng and Ma, Jianzhu and Zhang, Xuegong and Song, Le},
  journal={Nature methods},
  volume={21},
  number={8},
  pages={1481--1491},
  year={2024},
  publisher={Nature Publishing Group US New York}
}

@inproceedings{saunshi2019theoretical,
  title={A theoretical analysis of contrastive unsupervised representation learning},
  author={Saunshi, Nikunj and Plevrakis, Orestis and Arora, Sanjeev and Khodak, Mikhail and Khandeparkar, Hrishikesh},
  booktitle={International conference on machine learning},
  pages={5628--5637},
  year={2019},
  organization={PMLR}
}

@inproceedings{wang2020understanding,
  title={Understanding contrastive representation learning through alignment and uniformity on the hypersphere},
  author={Wang, Tongzhou and Isola, Phillip},
  booktitle={International conference on machine learning},
  pages={9929--9939},
  year={2020},
  organization={PMLR}
}

@article{zhang2024causal,
  title={Causal representation learning from multiple distributions: A general setting},
  author={Zhang, Kun and Xie, Shaoan and Ng, Ignavier and Zheng, Yujia},
  journal={arXiv preprint arXiv:2402.05052},
  year={2024}
}

@article{cui2024scgpt,
  title={scGPT: toward building a foundation model for single-cell multi-omics using generative AI},
  author={Cui, Haotian and Wang, Chloe and Maan, Hassaan and Pang, Kuan and Luo, Fengning and Duan, Nan and Wang, Bo},
  journal={Nature methods},
  volume={21},
  number={8},
  pages={1470--1480},
  year={2024},
  publisher={Nature Publishing Group US New York}
}

@article{zhang2023identifiability,
  title={Identifiability guarantees for causal disentanglement from soft interventions},
  author={Zhang, Jiaqi and Greenewald, Kristjan and Squires, Chandler and Srivastava, Akash and Shanmugam, Karthikeyan and Uhler, Caroline},
  journal={Advances in Neural Information Processing Systems},
  volume={36},
  pages={50254--50292},
  year={2023}
}

@article{lippe2022causal,
  title={Causal representation learning for instantaneous and temporal effects in interactive systems},
  author={Lippe, Phillip and Magliacane, Sara and L{\"o}we, Sindy and Asano, Yuki M and Cohen, Taco and Gavves, Efstratios},
  journal={arXiv preprint arXiv:2206.06169},
  year={2022}
}

@inproceedings{lippe2022citris,
  title={Citris: Causal identifiability from temporal intervened sequences},
  author={Lippe, Phillip and Magliacane, Sara and L{\"o}we, Sindy and Asano, Yuki M and Cohen, Taco and Gavves, Stratis},
  booktitle={International Conference on Machine Learning},
  pages={13557--13603},
  year={2022},
  organization={PMLR}
}

@article{brehmer2022weakly,
  title={Weakly supervised causal representation learning},
  author={Brehmer, Johann and De Haan, Pim and Lippe, Phillip and Cohen, Taco S},
  journal={Advances in Neural Information Processing Systems},
  volume={35},
  pages={38319--38331},
  year={2022}
}

@article{von2021self,
  title={Self-supervised learning with data augmentations provably isolates content from style},
  author={Von K{\"u}gelgen, Julius and Sharma, Yash and Gresele, Luigi and Brendel, Wieland and Sch{\"o}lkopf, Bernhard and Besserve, Michel and Locatello, Francesco},
  journal={Advances in neural information processing systems},
  volume={34},
  pages={16451--16467},
  year={2021}
}

@article{hyvarinen1999nonlinear,
  title={Nonlinear independent component analysis: Existence and uniqueness results},
  author={Hyv{\"a}rinen, Aapo and Pajunen, Petteri},
  journal={Neural networks},
  volume={12},
  number={3},
  pages={429--439},
  year={1999},
  publisher={Elsevier}
}

@incollection{hyvarinen2001independent,
  title={Independent component analysis},
  author={Hyv{\"a}rinen, Aapo and Hurri, Jarmo and Hoyer, Patrik O},
  booktitle={Natural Image Statistics: A Probabilistic Approach to Early Computational Vision},
  pages={151--175},
  year={2001},
  publisher={Springer}
}

@article{liu2022identifying,
  title={Identifying weight-variant latent causal models},
  author={Liu, Yuhang and Zhang, Zhen and Gong, Dong and Gong, Mingming and Huang, Biwei and van den Hengel, Anton and Zhang, Kun and Shi, Javen Qinfeng},
  journal={Journal of Machine Learning Research},
  volume={27},
  number={4},
  pages={1--49},
  year={2026}
}

@inproceedings{
liu2026i,
title={I Predict Therefore I Am: Is Next Token Prediction Enough to Learn Human-Interpretable Concepts from Data?},
author={Yuhang Liu and Dong Gong and Yichao Cai and Erdun Gao and Zhen Zhang and Biwei Huang and Mingming Gong and Anton van den Hengel and Javen Qinfeng Shi},
booktitle={The Fourteenth International Conference on Learning Representations},
year={2026}
}

@inproceedings{
liu2026beyond,
title={Beyond {DAG}s: A Latent Partial Causal Model for Multimodal Learning},
author={Yuhang Liu and Zhen Zhang and Dong Gong and Erdun Gao and Biwei Huang and Mingming Gong and Anton van den Hengel and Kun Zhang and Javen Qinfeng Shi},
booktitle={The Fourteenth International Conference on Learning Representations},
year={2026},
}

@article{
liu2025latent,
title={Latent Covariate Shift: Unlocking Partial Identifiability for Multi-Source Domain Adaptation},
author={Yuhang Liu and Zhen Zhang and Dong Gong and Mingming Gong and Biwei Huang and Anton van den Hengel and Kun Zhang and Javen Qinfeng Shi},
journal={Transactions on Machine Learning Research},
issn={2835-8856},
year={2025},
}

@article{liu2024towards,
  title={Towards Identifiable Latent Additive Noise Models},
  author={Liu, Yuhang and Zhang, Zhen and Gong, Dong and Gao, Erdun and Huang, Biwei and Gong, Mingming and Hengel, Anton van den and Zhang, Kun and Shi, Javen Qinfeng},
  journal={arXiv preprint arXiv:2403.15711},
  year={2024}
}

@article{hyvarinen2016unsupervised,
  title={Unsupervised feature extraction by time-contrastive learning and nonlinear ica},
  author={Hyvarinen, Aapo and Morioka, Hiroshi},
  journal={Advances in neural information processing systems},
  volume={29},
  year={2016}
}

@article{huang2024predicting,
  title={Predicting single-cell cellular responses to perturbations using cycle consistency learning},
  author={Huang, Wei and Liu, Hui},
  journal={Bioinformatics},
  volume={40},
  number={Supplement\_1},
  pages={i462--i470},
  year={2024},
  publisher={Oxford University Press}
}

@inproceedings{squires2023linear,
  title={Linear causal disentanglement via interventions},
  author={Squires, Chandler and Seigal, Anna and Bhate, Salil S and Uhler, Caroline},
  booktitle={International conference on machine learning},
  pages={32540--32560},
  year={2023},
  organization={PMLR}
}

@inproceedings{hyvarinen2017nonlinear,
  title={Nonlinear ICA of temporally dependent stationary sources},
  author={Hyvarinen, Aapo and Morioka, Hiroshi},
  booktitle={Artificial intelligence and statistics},
  pages={460--469},
  year={2017},
  organization={PMLR}
}

@inproceedings{
liu2023identifiable,
title={Identifiable Latent Polynomial Causal Models through the Lens of Change},
author={Yuhang Liu and Zhen Zhang and Dong Gong and Mingming Gong and Biwei Huang and Anton van den Hengel and Kun Zhang and Javen Qinfeng Shi},
booktitle={The Twelfth International Conference on Learning Representations},
year={2024}
}

@article{sorrenson2020disentanglement,
  title={Disentanglement by nonlinear ica with general incompressible-flow networks (gin)},
  author={Sorrenson, Peter and Rother, Carsten and K{\"o}the, Ullrich},
  journal={arXiv preprint arXiv:2001.04872},
  year={2020}
}

@inproceedings{khemakhem2020variational,
  title={Variational autoencoders and nonlinear ica: A unifying framework},
  author={Khemakhem, Ilyes and Kingma, Diederik and Monti, Ricardo and Hyvarinen, Aapo},
  booktitle={International conference on artificial intelligence and statistics},
  pages={2207--2217},
  year={2020},
  organization={PMLR}
}

@misc{kingma2013auto,
  title={Auto-encoding variational bayes},
  author={Kingma, Diederik P and Welling, Max and others},
  year={2013},
  publisher={Banff, Canada}
}

@inproceedings{rezende2014stochastic,
  title={Stochastic backpropagation and approximate inference in deep generative models},
  author={Rezende, Danilo Jimenez and Mohamed, Shakir and Wierstra, Daan},
  booktitle={International conference on machine learning},
  pages={1278--1286},
  year={2014},
  organization={PMLR}
}

@article{mao2024learning,
  title={Learning identifiable factorized causal representations of cellular responses},
  author={Mao, Haiyi and Lopez, Romain and Liu, Kai and Huetter, Jan-Christian and Richmond, David and Benos, Panayiotis and Qiu, Lin},
  journal={Advances in Neural Information Processing Systems},
  volume={37},
  pages={121630--121669},
  year={2024}
}

@article{von2024nonparametric,
  title={Nonparametric identifiability of causal representations from unknown interventions},
  author={von K{\"u}gelgen, Julius and Besserve, Michel and Wendong, Liang and Gresele, Luigi and Keki{\'c}, Armin and Bareinboim, Elias and Blei, David and Sch{\"o}lkopf, Bernhard},
  journal={Advances in Neural Information Processing Systems},
  volume={36},
  year={2023}
}

@article{lotfollahi2019scgen,
  title={scGen predicts single-cell perturbation responses},
  author={Lotfollahi, Mohammad and Wolf, F Alexander and Theis, Fabian J},
  journal={Nature methods},
  volume={16},
  number={8},
  pages={715--721},
  year={2019},
  publisher={Nature Publishing Group US New York}
}

@article{lotfollahi2023predicting,
  title={Predicting cellular responses to complex perturbations in high-throughput screens},
  author={Lotfollahi, Mohammad and Klimovskaia Susmelj, Anna and De Donno, Carlo and Hetzel, Leon and Ji, Yuge and Ibarra, Ignacio L and Srivatsan, Sanjay R and Naghipourfar, Mohsen and Daza, Riza M and Martin, Beth and others},
  journal={Molecular systems biology},
  volume={19},
  number={6},
  pages={e11517},
  year={2023}
}

@article{roohani2023gears,
  title={Predicting transcriptional outcomes of novel multigene perturbations},
  author={Roohani, Yusuf and Kazerouni, Abbas and Xie, Zhen and Yosef, Nir},
  journal={Nature Biotechnology},
  volume={42},
  pages={927--935},
  year={2024},
  publisher={Nature Publishing Group}
}

@article{bereket2023modelling,
  title={Modelling cellular perturbations with the sparse additive mechanism shift variational autoencoder},
  author={Bereket, Michael and Karaletsos, Theofanis},
  journal={Advances in Neural Information Processing Systems},
  volume={36},
  pages={1--12},
  year={2023}
}

@article{de2025interpretable,
  title={Interpretable Causal Representation Learning for Biological Data in the Pathway Space},
  author={de la Fuente, Jesus and Lehmann, Robert and Ruiz-Arenas, Carlos and Voges, Jan and Marin-Go{\~n}i, Irene and Martinez-de-Morentin, Xabier and Gomez-Cabrero, David and Ochoa, Idoia and Tegner, Jesper and Lagani, Vincenzo and others},
  journal={arXiv preprint arXiv:2506.12439},
  year={2025}
}

@article{wang2023multi,
  title={Multi-ContrastiveVAE disentangles perturbation effects in single cell images from optical pooled screens},
  author={Wang, Zitong Jerry and Lopez, Romain and H{\"u}tter, Jan-Christian and Kudo, Takamasa and Yao, Heming and Hanslovsky, Philipp and H{\"o}ckendorf, Burkhard and Moran, Rahul and Richmond, David and Regev, Aviv},
  journal={bioRxiv},
  pages={2023--11},
  year={2023},
  publisher={Cold Spring Harbor Laboratory}
}

@inproceedings{ahuja2023interventional,
  title={Interventional causal representation learning},
  author={Ahuja, Kartik and Mahajan, Divyat and Wang, Yixin and Bengio, Yoshua},
  booktitle={International conference on machine learning},
  pages={372--407},
  year={2023},
  organization={PMLR}
}

@article{hetzel2022predicting,
  title={Predicting cellular responses to novel drug perturbations at a single-cell resolution},
  author={Hetzel, Leon and Boehm, Simon and Kilbertus, Niki and G{\"u}nnemann, Stephan and Theis, Fabian and others},
  journal={Advances in Neural Information Processing Systems},
  volume={35},
  pages={26711--26722},
  year={2022}
}

@article{replogle2022mapping,
  title={Mapping information-rich genotype-phenotype landscapes with genome-scale Perturb-seq},
  author={Replogle, Joseph M and Saunders, Reuben A and Pogson, Angela N and Hussmann, Jeffrey A and Lenail, Alexander and Guna, Alina and Mascibroda, Lauren and Wagner, Eric J and Adelman, Karen and Lithwick-Yanai, Gila and others},
  journal={Cell},
  volume={185},
  number={14},
  pages={2559--2575},
  year={2022},
  publisher={Elsevier}
}

@inproceedings{
lopez2022learning,
title={Learning Causal Representations of Single Cells via Sparse Mechanism Shift Modeling},
author={Romain Lopez and Natasa Tagasovska and Stephen Ra and Kyunghyun Cho and Jonathan Pritchard and Aviv Regev},
booktitle={NeurIPS 2022 Workshop on Causality for Real-world Impact},
year={2022},
url={https://openreview.net/forum?id=gdTXCy7fZf7}
}

@inproceedings{
lachapelle2022disentanglement,
title={Disentanglement via Mechanism Sparsity Regularization: A New Principle for Nonlinear {ICA}},
author={Sebastien Lachapelle and Pau Rodriguez and Yash Sharma and Katie E Everett and R{\'e}mi LE PRIOL and Alexandre Lacoste and Simon Lacoste-Julien},
booktitle={First Conference on Causal Learning and Reasoning},
year={2022},
url={https://openreview.net/forum?id=dHsFFekd_-o}
}

@article{scholkopf2021toward,
  title={Toward causal representation learning},
  author={Sch{\"o}lkopf, Bernhard and Locatello, Francesco and Bauer, Stefan and Ke, Nan Rosemary and Kalchbrenner, Nal and Goyal, Anirudh and Bengio, Yoshua},
  journal={Proceedings of the IEEE},
  volume={109},
  number={5},
  pages={612--634},
  year={2021},
  publisher={IEEE}
}

@article{replogle2020combinatorial,
  title={Combinatorial single-cell CRISPR screens by direct guide RNA capture and targeted sequencing},
  author={Replogle, Joseph M and Norman, Thomas M and Xu, Albert and Hussmann, Jeffrey A and Chen, Jin and Cogan, J Zachery and Meer, Elliott J and Terry, Jessica M and Riordan, Daniel P and Srinivas, Niranjan and others},
  journal={Nature biotechnology},
  volume={38},
  number={8},
  pages={954--961},
  year={2020},
  publisher={Nature Publishing Group US New York}
}

@inproceedings{Cai2024CLAP,
  title     = {CLAP: Isolating Content from Style through Contrastive Learning with Augmented Prompts},
  author    = {Yichao Cai and Yuhang Liu and Zhen Zhang and Javen Qinfeng Shi},
  booktitle = {European Conference on Computer Vision (ECCV)},
  pages     = {130--147},
  year      = {2024}
}

@article{dixit2016perturb,
  title={Perturb-Seq: dissecting molecular circuits with scalable single-cell RNA profiling of pooled genetic screens},
  author={Dixit, Atray and Parnas, Oren and Li, Biyu and Chen, Jenny and Fulco, Charles P and Jerby-Arnon, Livnat and Marjanovic, Nemanja D and Dionne, Danielle and Burks, Tyler and Raychowdhury, Raktima and others},
  journal={cell},
  volume={167},
  number={7},
  pages={1853--1866},
  year={2016},
  publisher={Elsevier}
}

@inproceedings{kong2022partial,
  title={Partial disentanglement for domain adaptation},
  author={Kong, Lingjing and Xie, Shaoan and Yao, Weiran and Zheng, Yujia and Chen, Guangyi and Stojanov, Petar and Akinwande, Victor and Zhang, Kun},
  booktitle={International conference on machine learning},
  pages={11455--11472},
  year={2022},
  organization={PMLR}
}

@article{gao2025domain,
  title={Domain generalization via content factors isolation: a two-level latent variable modeling approach},
  author={Gao, Erdun and Bondell, Howard and Huang, Shaoli and Gong, Mingming},
  journal={Machine Learning},
  volume={114},
  number={4},
  pages={1--33},
  year={2025},
  publisher={Springer}
}

@article{jinek2012programmable,
  title={A programmable dual-RNA--guided DNA endonuclease in adaptive bacterial immunity},
  author={Jinek, Martin and Chylinski, Krzysztof and Fonfara, Ines and Hauer, Michael and Doudna, Jennifer A and Charpentier, Emmanuelle},
  journal={science},
  volume={337},
  number={6096},
  pages={816--821},
  year={2012},
  publisher={American Association for the Advancement of Science}
}

@article{orgogozo2015differential,
  title={The differential view of genotype--phenotype relationships},
  author={Orgogozo, Virginie and Morizot, Baptiste and Martin, Arnaud},
  journal={Frontiers in genetics},
  volume={6},
  pages={179},
  year={2015},
  publisher={Frontiers Media SA}
}

@inproceedings{chen2020simple,
  title={A simple framework for contrastive learning of visual representations},
  author={Chen, Ting and Kornblith, Simon and Norouzi, Mohammad and Hinton, Geoffrey},
  booktitle={International conference on machine learning},
  pages={1597--1607},
  year={2020},
  organization={PmLR}
}

@article{grill2020bootstrap,
  title={Bootstrap your own latent-a new approach to self-supervised learning},
  author={Grill, Jean-Bastien and Strub, Florian and Altch{\'e}, Florent and Tallec, Corentin and Richemond, Pierre and Buchatskaya, Elena and Doersch, Carl and Avila Pires, Bernardo and Guo, Zhaohan and Gheshlaghi Azar, Mohammad and others},
  journal={Advances in neural information processing systems},
  volume={33},
  pages={21271--21284},
  year={2020}
}

@inproceedings{radford2021learning,
  title={Learning transferable visual models from natural language supervision},
  author={Radford, Alec and Kim, Jong Wook and Hallacy, Chris and Ramesh, Aditya and Goh, Gabriel and Agarwal, Sandhini and Sastry, Girish and Askell, Amanda and Mishkin, Pamela and Clark, Jack and others},
  booktitle={International conference on machine learning},
  pages={8748--8763},
  year={2021},
  organization={PmLR}
}

@inproceedings{
cai2025value,
title={On the Value of Cross-Modal Misalignment in Multimodal Representation Learning},
author={Yichao Cai and Yuhang Liu and Erdun Gao and Tianjiao Jiang and Zhen Zhang and Anton van den Hengel and Javen Qinfeng Shi},
booktitle={The Thirty-ninth Annual Conference on Neural Information Processing Systems},
year={2025}
}

@inproceedings{tschannen2020mutual,
  title={On Mutual Information Maximization for Representation Learning},
  author={Tschannen, M and Djolonga, J and Rubenstein, P and Gelly, S and Lucic, M},
  booktitle={Eighth International Conference on Learning Representations},
  year={2020},
  organization={OpenReview. net}
}

@article{aliee2023conditionally,
  title={Conditionally invariant representation learning for disentangling cellular heterogeneity},
  author={Aliee, Hananeh and Kapl, Ferdinand and Hediyeh-Zadeh, Soroor and Theis, Fabian J},
  journal={arXiv preprint arXiv:2307.00558},
  year={2023}
}

@article{tu2024supervised,
  title={A supervised contrastive framework for learning disentangled representations of cell perturbation data},
  author={Tu, Xinming and H{\"u}tter, Jan-Christian and Wang, Zitong Jerry and Kudo, Takamasa and Regev, Aviv and Lopez, Romain},
  journal={BioRxiv},
  pages={2024--01},
  year={2024},
  publisher={Cold Spring Harbor Laboratory}
}

@article{weinberger2023isolating,
  title={Isolating salient variations of interest in single-cell data with contrastiveVI},
  author={Weinberger, Ethan and Lin, Chris and Lee, Su-In},
  journal={Nature Methods},
  volume={20},
  number={9},
  pages={1336--1345},
  year={2023},
  publisher={Nature Publishing Group US New York}
}

@article{emanuelli2008overexpression,
  title={Overexpression of the dual-specificity phosphatase MKP-4/DUSP-9 protects against stress-induced insulin resistance},
  author={Emanuelli, Brice and Eberl{\'e}, Delphine and Suzuki, Ryo and Kahn, C Ronald},
  journal={Proceedings of the National Academy of Sciences},
  volume={105},
  number={9},
  pages={3545--3550},
  year={2008},
  publisher={National Academy of Sciences}
}

@article{zhang2009non,
  title={Non-Smad pathways in TGF-$\beta$ signaling},
  author={Zhang, Ying E},
  journal={Cell research},
  volume={19},
  number={1},
  pages={128--139},
  year={2009},
  publisher={Nature Publishing Group}
}

@article{koeppel2011crosstalk,
  title={Crosstalk between c-Jun and TAp73$\alpha$/$\beta$ contributes to the apoptosis--survival balance},
  author={Koeppel, Max and van Heeringen, Simon J and Kramer, Daniela and Smeenk, Leonie and Janssen-Megens, Eva and Hartmann, Marianne and Stunnenberg, Hendrik G and Lohrum, Marion},
  journal={Nucleic acids research},
  volume={39},
  number={14},
  pages={6069--6085},
  year={2011},
  publisher={Oxford University Press}
}

@article{lim1998stress,
  title={Stress-induced immediate-early gene, egr-1, involves activation of p38/JNK1},
  author={Lim, Cheh Peng and Jain, Neeraj and Cao, Xinmin},
  journal={Oncogene},
  volume={16},
  number={22},
  pages={2915--2926},
  year={1998},
  publisher={Nature Publishing Group}
}

@article{sundqvist2013specific,
  title={Specific interactions between Smad proteins and AP-1 components determine TGF$\beta$-induced breast cancer cell invasion},
  author={Sundqvist, Anders and Zieba, Agata and Vasilaki, Eleftheria and Herrera Hidalgo, Carmen and S{\"o}derberg, Ola and Koinuma, D and Miyazono, Kohei and Heldin, Carl-Henrik and Landegren, Ulf and ten Dijke, Peter and others},
  journal={Oncogene},
  volume={32},
  number={31},
  pages={3606--3615},
  year={2013},
  publisher={Nature Publishing Group}
}

@article{ikushima2010tgfbeta,
  title={TGF$\beta$ signalling: a complex web in cancer progression},
  author={Ikushima, Hiroaki and Miyazono, Kohei},
  journal={Nature reviews cancer},
  volume={10},
  number={6},
  pages={415--424},
  year={2010},
  publisher={Nature Publishing Group UK London}
}

@article{fan2025identification,
  title={Identification of a SNAI1 enhancer RNA that drives cancer cell plasticity},
  author={Fan, Chuannan and Wang, Qian and Krijger, Peter HL and Cats, Davy and Selle, Miriam and Khorosjutina, Olga and Dhanjal, Soniya and Schmierer, Bernhard and Mei, Hailiang and de Laat, Wouter and others},
  journal={Nature Communications},
  volume={16},
  number={1},
  pages={2890},
  year={2025},
  publisher={Nature Publishing Group UK London}
}

@article{schmidt2021p53,
  title={The p53/p73-p21CIP1 tumor suppressor axis guards against chromosomal instability by restraining CDK1 in human cancer cells},
  author={Schmidt, Ann-Kathrin and Pudelko, Karoline and Boekenkamp, Jan-Eric and Berger, Katharina and Kschischo, Maik and Bastians, Holger},
  journal={Oncogene},
  volume={40},
  number={2},
  pages={436--451},
  year={2021},
  publisher={Nature Publishing Group UK London}
}

@article{vincent2009snail1,
  title={A SNAIL1--SMAD3/4 transcriptional repressor complex promotes TGF-$\beta$ mediated epithelial--mesenchymal transition},
  author={Vincent, Theresa and Neve, Etienne PA and Johnson, Jill R and Kukalev, Alexander and Rojo, Federico and Albanell, Joan and Pietras, Kristian and Virtanen, Ismo and Philipson, Lennart and Leopold, Philip L and others},
  journal={Nature cell biology},
  volume={11},
  number={8},
  pages={943--950},
  year={2009},
  publisher={Nature Publishing Group UK London}
}

@article{ragione2003p21cip1,
  title={p21Cip1 gene expression is modulated by Egr1: a novel regulatory mechanism involved in the resveratrol antiproliferative effect.},
  author={Ragione, Fulvio Della and Cucciolla, Valeria and Criniti, Vittoria and Indaco, Stefania and Borriello, Adriana and Zappia, Vincenzo},
  journal={The Journal of Biological Chemistry},
  volume={278},
  number={26},
  pages={23360--23368},
  year={2003}
}

@article{gao2025causal,
  title={Causal disentanglement for single-cell representations and controllable counterfactual generation},
  author={Gao, Yicheng and Dong, Kejing and Shan, Caihua and Li, Dongsheng and Liu, Qi},
  journal={Nature communications},
  volume={16},
  number={1},
  pages={6775},
  year={2025},
  publisher={Nature Publishing Group UK London}
}

@article{rosen2023universal,
  title={Universal Cell Embeddings: A Foundation Model for Cell Biology},
  author={Rosen, Yanay and Roohani, Yusuf and Agrawal, Ayush and Samotorcan, Leon and Consortium, Tabula Sapiens and Quake, Stephen R and Leskovec, Jure},
  journal={bioRxiv},
  pages={2023--11},
  year={2023},
  publisher={Cold Spring Harbor Laboratory}
}

@inproceedings{yang2021causalvae,
  title={Causalvae: Disentangled representation learning via neural structural causal models},
  author={Yang, Mengyue and Liu, Furui and Chen, Zhitang and Shen, Xinwei and Hao, Jianye and Wang, Jun},
  booktitle={Proceedings of the IEEE/CVF conference on computer vision and pattern recognition},
  pages={9593--9602},
  year={2021}
}

@article{adduri2025predicting,
  title={Predicting cellular responses to perturbation across diverse contexts with State},
  author={Adduri, Abhinav K and Gautam, Dhruv and Bevilacqua, Beatrice and Imran, Alishba and Shah, Rohan and Naghipourfar, Mohsen and Teyssier, Noam and Ilango, Rajesh and Nagaraj, Sanjay and Dong, Mingze and others},
  journal={BioRxiv},
  pages={2025--06},
  year={2025},
  publisher={Cold Spring Harbor Laboratory}
}

@article{bendidi2024benchmarking,
  title={Benchmarking Transcriptomics Foundation Models for Perturbation Analysis: one PCA still rules them all},
  author={Bendidi, Ihab and Whitfield, Shawn and Kenyon-Dean, Kian and Yedder, Hanene Ben and Mesbahi, Yassir El and Noutahi, Emmanuel and Denton, Alisandra K},
  journal={arXiv preprint arXiv:2410.13956},
  year={2024}
}
\bibliographystyle{icml2026}

\clearpage                  
\onecolumn  
\appendix

\newpage
\appendix
\part{Appendix} 
\parttoc
\newpage

\section{Related Work}
\label{app:relatedwork}
\paragraph{Disentangling Single-Cell Perturbation Effects.} A central challenge in single-cell perturbation modeling is to separate intervention effects from intrinsic cellular variability. Deep generative approaches have shown strong performance on this task. scGen~\citep{lotfollahi2019scgen} models perturbations as additive shifts in a latent space, while CPA~\citep{lotfollahi2023predicting} factorizes each cell into basal state and perturbation effect. chemCPA~\citep{hetzel2022predicting} extends CPA with chemical structure embeddings and dosage information, enabling zero-shot predictions for unseen compounds. Other methods incorporate biological priors or contrastive objectives: GEARS~\citep{roohani2023gears} uses gene-gene interaction graphs for improved generalization across perturbation combinations, and contrastive VAEs have been applied in optical pooled screening to disentangle stable identity from perturbation-driven variation~\citep{wang2023multi}. Despite empirical successes, most of these models treat disentanglement statistically rather than causally, which limits interpretability. Recent work has incorporated sparsity into latent-variable models to encourage identifiable and interpretable representations. CausCell~\citep{gao2025causal} enables counterfactual generation via SCM-guided diffusion, but critically depends on a predefined causal graph, limiting its applicability when causal structures are unknown or hard to specify.  sVAE+~\citep{lopez2022learning}, SAMS-VAE~\citep{bereket2023modelling} impose sparse structure or mechanism shifts in the latent space to model perturbation-induced variation. 
Recent advances such as discrepancy-VAE~\citep{zhang2023identifiability}, and its interpretable variant~\citep{de2025interpretable} align latent-variable models with identifiable causal semantics, pointing toward representations that are both intervention-sensitive and explanatory. Building on these advances, our approach moves beyond purely statistical factorization, ensuring that the learned representations reflect genuine causal effects of perturbations.


\paragraph{Identifiable Causal Representations.} A key aim in modeling complex systems is to learn low-dimensional latent variables $\mathbf{z}$ from high-dimensional data $\mathbf{x}$ that match the true generative factors~\citep{hyvarinen2001independent,liu2026i}. Nonlinear ICA showed that such components are not identifiable from i.i.d.\ data without extra assumptions~\citep{hyvarinen1999nonlinear}. Identifiable variants address this by introducing an auxiliary variable $\mathbf{u}$ so that latent factors $\{z_i\}_{i=1}^p$ are conditionally independent given $\mathbf{u}$~\citep{hyvarinen2016unsupervised,hyvarinen2017nonlinear}. The iVAE framework~\citep{khemakhem2020variational}, built on VAEs~\citep{kingma2013auto,rezende2014stochastic}, proves identifiability of both $\mathbf{z}$ and $p(\mathbf{x}\mid \mathbf{z})$ under mild conditions. Recent approaches impose structure in latent space: DAG-based models enforce acyclicity~\citep{yang2021causalvae, lippe2022citris,liu2022identifying,liu2023identifiable,ahuja2023interventional,liu2024towards,liu2025latent}, while factorized designs split latent variables into invariant, intervention-specific, and interaction parts~\citep{von2021self,kong2022partial,gao2025domain}. While prior methods establish identifiability via auxiliary conditioning or broad structural constraints, our model ties perturbations directly to latent mechanisms. This design moves beyond heuristic augmentations or globally factorized latents, making our framework specifically tailored to single-cell perturbation.

\paragraph{Contrastive Representation Learning.}
Contrastive multi-view learning learns invariances across views or modalities (e.g., SimCLR, BYOL, CLIP-style training) but typically relies on heuristic augmentations whose invariants need not align with causal structure~\citep{chen2020simple,grill2020bootstrap,radford2021learning,Cai2024CLAP,cai2025value,liu2026beyond,tschannen2020mutual,von2021self}. \citet{aliee2023conditionally} learn conditionally invariant representations by leveraging variability across observational environments (patients, batches, platforms) to suppress domain-specific artifacts while preserving biological signal. In single-cell analysis, \citet{weinberger2023isolating} contrast background and target datasets—extending to multi-omics—to isolate salient structure, but provide no identifiability guarantees. For perturbation screens, supervised contrastive VAEs use guide labels with HSIC to isolate perturbation effects from background heterogeneity~\citep{tu2024supervised}. Concurrently, \citet{mao2024learning} posit a three-way factorization (covariate, treatment, interaction) and promote independence via structural constraints and adversarial training; while principled, this fixed design may underfit non-classical responses, and its identifiability hinges on stringent experimental designs. Unlike contrastive or domain-invariant models, we obtain block identifiability for the perturbation-invariant block and component-wise identifiability for the perturbation-responsive block under a weight-variant latent SCM, thereby performing CRL in the latent space and recovering the latent causal graph among responsive variables.

\newpage
\section{Limitations}

A fundamental limitation of this work, shared with most studies on latent causal representation learning, lies in the assumptions introduced in our theoretical analysis. Nevertheless, identifying latent causal variables without assumptions is impossible in general. As a result, like most identifiability analyses in the causal representation learning literature, we necessarily impose a set of assumptions to characterize when latent causal representations can be recovered. These assumptions are not specific to our framework; rather, they are commonly used in nonlinear ICA and causal representation learning, as discussed in Sec.~\ref{app: justify}.

We further acknowledge that these assumptions are generally untestable in real-world biological data, similar to most of the work in the causal representation learning community, since the true generative process and latent causal variables are unknown. Nevertheless, our empirical results on real perturbation datasets suggest that our approach is effective, demonstrating that the insights derived from our theoretical analysis may meaningfully guide practical modeling.

Finally, our OOD claim is restricted to combinatorial generalization over genes observed during training. Specifically, the model is trained on single-gene perturbations and evaluated on unseen double-gene combinations, where each individual gene has already appeared as a perturbation target. This differs from unseen-gene generalization, where the model must predict the effect of perturbing a gene that was never observed during training. Our current perturbation input $\mathbf{u}$ is an identifier-style condition vector; it provides no biological geometry between genes and no information about the function, pathway membership, regulatory context, or sequence properties of an unseen target. Consequently, a new gene coordinate in $\mathbf{u}$ has no learned support, making unseen-gene prediction statistically underdetermined without additional side information. Addressing this setting would require augmenting $\mathbf{u}$ with biologically meaningful gene-level features, such as GO/pathway annotations, gene embeddings, regulatory-network features, or sequence-derived representations. We therefore do not claim to address unseen-gene or cross-cell-type perturbation generalization in this work.

\section{Difference from Existing Causal Representation Learning Methods}

Learning causal representations for single-cell perturbation prediction is a promising direction. A central challenge in this line of work lies in theoretical support from identifiability analysis, which ensures that the true latent causal variables can be reliably recovered, thereby enabling meaningful learning of the causal relationships among them~\citep{lachapelle2022disentanglement,zhang2023identifiability,lopez2022learning,de2025interpretable}. Most existing identifiability results, however, rely on the assumption that \emph{all} latent causal variables are intervened across environments~\citep{zhang2023identifiability,lopez2022learning,de2025interpretable}. In real cellular perturbation experiments, this assumption is rarely satisfied. Comprehensive perturbation of all genes is often prohibitively expensive or technically infeasible, and practical datasets typically contain interventions on only a small subset of genes, as we highlighted throughout this work. As a consequence, a large subspace of genes, and the corresponding latent causal variables governing their expression, remains unperturbed across all observed conditions.

This mismatch between theoretical assumptions and experimental reality poses a fundamental challenge. As a result, existing identifiability theory may not be directly applicable under such partial-intervention settings. In turn, methods built upon these theoretical results~\citep{lopez2022learning,zhang2023identifiability,de2025interpretable} may struggle to perform effectively in practice, as invariant background variation can dominate the learned representations and obscure perturbation-specific signals.

In contrast to prior work, our approach explicitly embraces the partial-intervention nature of real cellular data. Rather than assuming that all latent causal variables are perturbed, we focus on separating perturbation-responsive factors from dominant perturbation-invariant structure. This perspective motivates our theoretical analysis, which, in turn, provides model design for the realistic setting of limited and partial perturbations, bridging the gap between identifiability theory and real-world single-cell perturbation data.

\newpage

\section{Lemmas}
\label{aop: lemmas}
\begin{lemma}[Unit-Jacobian mapping from latent causal variables to latent noise variables]
\label{lem:unit_jacobian_noise}
Consider the structural Eqs.~\ref{eq:lcm_inv}-\ref{eq:lcm_nu}. Let
$\mathbf{z}=(\mathbf{z}_\iota,\mathbf{z}_\nu)\in\mathbb{R}^{d_\iota+d_\nu}$ and
$\mathbf{n}=(\mathbf{n}_\iota,\mathbf{n}_\nu)\in\mathbb{R}^{d_\iota+d_\nu}$.
Define the block matrix
\begin{equation}
\boldsymbol{\Lambda}(\mathbf{u})
:=
\begin{pmatrix}
\boldsymbol{\lambda}_{\iota\iota} & \mathbf{0}\\
\boldsymbol{\lambda}_{\nu\iota}(\mathbf{u}) & \boldsymbol{\lambda}_{\nu\nu}(\mathbf{u})
\end{pmatrix}.
\label{eq:Lambda_def}
\end{equation}
Then we have:
\begin{equation}
\mathbf{z} = \boldsymbol{\Lambda}(\mathbf{u})\,\mathbf{z} + \mathbf{n},
\qquad\text{equivalently}\qquad
\mathbf{n} = (\mathbf{I}-\boldsymbol{\Lambda}(\mathbf{u}))\,\mathbf{z}.
\label{eq:z_Lambda_n}
\end{equation}
Moreover, the mapping $\mathbf{z}\mapsto \mathbf{n}$ is bijective and has unit Jacobian determinant:
\begin{equation}
\big|\det (\mathbf{I}-\boldsymbol{\Lambda}(\mathbf{u}))\big| = 1\qquad\text{for all }\mathbf{u}.
\label{eq:det_unit}
\end{equation}
\end{lemma}

\begin{proof}
By DAG assumption, $\boldsymbol{\lambda}_{\iota\iota}$ and $\boldsymbol{\lambda}_{\nu\nu}(\mathbf{u})$
are strictly lower triangular, hence have zeros on their diagonals. Therefore
$\boldsymbol{\Lambda}(\mathbf{u})$ has zeros on its diagonal as well, and
$I-\boldsymbol{\Lambda}(\mathbf{u})$ is block lower triangular with diagonal blocks
$\mathbf{I}-\boldsymbol{\lambda}_{\iota\iota}$ and $\mathbf{I}-\boldsymbol{\lambda}_{\nu\nu}(\mathbf{u})$.
Since each diagonal block is lower triangular with ones on the diagonal, we have
\begin{equation}
\det(\mathbf{I}-\boldsymbol{\lambda}_{\iota\iota})=1,
\qquad
\det(\mathbf{I}-\boldsymbol{\lambda}_{\nu\nu}(\mathbf{u}))=1.
\end{equation}
Hence
\begin{equation}
\det(\mathbf{I}-\boldsymbol{\Lambda}(\mathbf{u}))
=
\det(\mathbf{I}-\boldsymbol{\lambda}_{\iota\iota})\cdot \det(\mathbf{I}-\boldsymbol{\lambda}_{\nu\nu}(\mathbf{u}))=1,
\end{equation}
which proves Eq.~\ref{eq:det_unit}. In particular, $\mathbf{I}-\boldsymbol{\Lambda}(\mathbf{u})$ is invertible, so
$\mathbf{z}\mapsto\mathbf{n}$ is bijective.
\end{proof}

\begin{lemma}[Component-wise identifiability of $\mathbf{n}_\nu$]
\label{lem:n_nu_perm}
Under the proposed latent generative model Eqs.~\ref{eq:lcm_inv}-\ref{eq:dgp} with an invertible $g$. Suppose there exist environments
$\{\mathbf{u}_0,\ldots,\mathbf{u}_{2d_\nu}\}$ such that the matrix
\begin{equation}
\mathbf{L}^\top :=
\begin{bmatrix}
\Delta\eta(\mathbf{u}_1)^\top\\
\vdots\\
\Delta\eta(\mathbf{u}_{2d_\nu})^\top
\end{bmatrix}
\in\mathbb{R}^{2d_\nu\times 2d_\nu}
\quad\text{has full column rank } 2d_\nu,
\end{equation}
where (elementwise divisions)
\begin{equation}
\Delta\boldsymbol{\eta}(\mathbf{u})
:=
\begin{pmatrix}
\frac{\boldsymbol{\mu}_\nu(\mathbf{u})}{\boldsymbol{\beta}_\nu(\mathbf{u})}
-\frac{\boldsymbol{\mu}_\nu(\mathbf{u}_0)}{\boldsymbol{\beta}_\nu(\mathbf{u}_0)}\\[2mm]
-\frac12\Big(\frac{1}{\boldsymbol{\beta}_\nu(\mathbf{u})}-\frac{1}{\boldsymbol{\beta}_\nu(\mathbf{u}_0)}\Big)
\end{pmatrix}\in\mathbb{R}^{2d_\nu}.
\end{equation}
If two parameter sets $\boldsymbol{\theta},\boldsymbol{\hat\theta}$ induce the same $p(\mathbf{x}\mid\mathbf{u})$ for all $(\mathbf{x},\mathbf{u})$,
then there exist a permutation $\pi$ and nonzero scalars $\{a_j\}$ and constants $\{b_j\}$ such that
\begin{equation}
 n_{\nu,j} = a_j\,\hat n_{\nu,\pi(j)} + b_j,\qquad j=1,\ldots,d_\nu,
\end{equation}
where $\hat n_{\nu,\pi(j)}$ is estimated by the likelihood matching.
\end{lemma}
\paragraph{Proof Sketch.} Let 
$\boldsymbol{\theta} := \big(\mathrm{g},\;\boldsymbol{\lambda}_{\iota\iota},\;\boldsymbol{\lambda}_{\nu\iota}(\cdot),\;\boldsymbol{\lambda}_{\nu\nu}(\cdot),\;
\boldsymbol{\mu}_{\iota},\;\boldsymbol{\beta}_{\iota},\;\boldsymbol{\mu}_{\nu}(\cdot),\;\boldsymbol{\beta}_{\nu}(\cdot)\big)$
denote the collection of all model parameters, and let $\hat{\theta}$ denote another parameter set from the same model class. \emph{(Step 1)} Using the change-of-variables formula and the unit Jacobian of the structural map $\mathbf{z}\mapsto\mathbf{n}$, equality $p_{\boldsymbol{\theta}}(\mathbf{x}\mid\mathbf{u})\equiv p_{\boldsymbol{\hat \theta}}(\mathbf{x}\mid\mathbf{u})$ implies equality of $\log p_{\boldsymbol{\theta}}(\mathbf{n}\mid\mathbf{u})$ and $\log p_{\boldsymbol{\hat \theta}}(\hat{\mathbf{n}}\mid\mathbf{u})$ up to an $\mathbf{u}$-independent term, which vanishes upon differencing across environments; since $p_{\boldsymbol{\theta}}(\mathbf{n}\mid\mathbf{u})=p(\mathbf{n}_\iota)\,p(\mathbf{n}_\nu\mid\mathbf{u})$ with $p(\mathbf{n}_\iota)$ invariant, this leaves an identity involving only $\mathbf{n}_\nu$.
\emph{(Step 2)} For diagonal Gaussians, the environment log-ratio is linear in the sufficient statistics $\mathbf{T}_\nu(\mathbf{n}_\nu)=(\mathbf{n}_\nu,\mathbf{n}_\nu^{\odot 2})$, and stacking over sufficiently many environments yields an affine relation $\mathbf{T}_\nu(\mathbf{n}_\nu)=\mathbf{A}_\nu \mathbf{\hat T}(\hat{\mathbf{n}})+\mathbf{c}_\nu$ with $\mathbf{A}_\nu$ full rank.
\emph{(Step 3)} The affine relation implies each $n_{\nu,j}$ (and $n_{\nu,j}^2$) is a low-degree polynomial in $\hat{\mathbf{n}}$, and independence and Gaussianity then force each $n_{\nu,j}$ to depend on exactly one coordinate of $\hat{\mathbf{n}}$, implying identifiability up to permutation and component-wise scaling/shift.

\begin{proof}

\textbf{Step 1.}
By Lemma~\ref{lem:unit_jacobian_noise}, for any $\mathbf{u}$ the map
$\mathbf{z}\mapsto \mathbf{n}=(\mathbf{I}-\boldsymbol{\Lambda}(\mathbf{u}))\mathbf{z}$
is bijective and satisfies $\big|\det(\mathbf{I}-\boldsymbol{\Lambda}(\mathbf{u}))\big|=1$.
Since $\mathbf{x}=\mathrm{g}(\mathbf{z})$ with $\mathrm{g}$ invertible, we can write
\begin{align}
\mathbf{n}=(\mathbf{I}-\boldsymbol{\Lambda}(\mathbf{u}))\,\mathrm{g}^{-1}(\mathbf{x}).
\end{align}
By the change-of-variables formula,
\begin{align}
\log p_{\boldsymbol{\theta}}(\mathbf{x}\mid\mathbf{u})
&= \log p_{\boldsymbol{\theta}}(\mathbf{n}\mid\mathbf{u})
    + \log\left|\det(\mathbf{I}-\boldsymbol{\Lambda}(\mathbf{u}))\right|
    + \log\left|\det J_{\mathrm{g}^{-1}}(\mathbf{x})\right| \nonumber\\
&= \log p_{\boldsymbol{\theta}}(\mathbf{n}\mid\mathbf{u})
    + \log\left|\det J_{\mathrm{g}^{-1}}(\mathbf{x})\right|,
\label{eq:cov_x_to_n}
\end{align}
where the last equality uses $\big|\det(\mathbf{I}-\boldsymbol{\Lambda}(\mathbf{u}))\big|=1$.

We assume that, in the limit of infinite data, any two parameterizations $\boldsymbol{\theta}$ and $\boldsymbol{\hat \theta}$ from the same model class
induce the same conditional distribution over the observations, i.e.,
\begin{equation}
p_{\boldsymbol{\theta}}(\mathbf{x}\mid\mathbf{u}) \equiv p_{\boldsymbol{\hat \theta}}(\mathbf{x}\mid\mathbf{u}) \qquad \text{for all } (\mathbf{x},\mathbf{u}).
\label{eq:same_px}
\end{equation}
Applying Eq.~\ref{eq:cov_x_to_n} to $\boldsymbol{\hat \theta}$ yields
\begin{equation}
\log p_{{\boldsymbol{\hat \theta}}}(\mathbf{x}\mid\mathbf{u})
=
\log p_{{\boldsymbol{\hat \theta}}}(\hat{\mathbf{n}}\mid\mathbf{u}) + \log\left|\det J_{\mathrm{\hat g}^{-1}}(\mathbf{x})\right|.
\end{equation}
Subtracting these two expressions and using Eq.~\ref{eq:same_px} gives
\begin{equation}
\log p_{\boldsymbol{\theta}}(\mathbf{n}\mid\mathbf{u}) - \log p_{{\boldsymbol{\hat \theta}}}(\hat{\mathbf{n}}\mid\mathbf{u})
=
\log\left|\det J_{\mathrm{\hat g}^{-1}}(\mathbf{x})\right| - \log\left|\det J_{\mathrm{g}^{-1}}(\mathbf{x})\right|
=: \boldsymbol{\phi}(\mathbf{x}),
\label{eq:match}
\end{equation}
where $\boldsymbol{\phi}(\mathbf{x})$ does not depend on $\mathbf{u}$.

For any environment $\mathbf{u}_\ell$, subtracting Eq.~\ref{eq:match} at $\mathbf{u}_0$ from that at $\mathbf{u}_\ell$ gives
\[
\log\frac{p_{\boldsymbol{\theta}}(\mathbf{n}\mid\mathbf{u}_\ell)}{p_{\boldsymbol{\theta}}(\mathbf{n}\mid\mathbf{u}_0)}
=
\log\frac{p_{{\boldsymbol{\hat \theta}}}(\hat{\mathbf{n}}\mid\mathbf{u}_\ell)}{p_{{\boldsymbol{\hat \theta}}}(\hat{\mathbf{n}}\mid\mathbf{u}_0)}.
\]
Since $p_{\boldsymbol{\theta}}(\mathbf{n}\mid\mathbf{u}) = p(\mathbf{n}_\iota)\,p(\mathbf{n}_\nu\mid\mathbf{u})$ and $p(\mathbf{n}_\iota)$ is invariant in $\mathbf{u}$,
the $\mathbf{n}_\iota$-terms cancel from the left side, yielding
\begin{equation}
\log\frac{p(\mathbf{n}_\nu\mid\mathbf{u}_\ell)}{p(\mathbf{n}_\nu\mid\mathbf{u}_0)}
=
\log\frac{p_{{\boldsymbol{\hat \theta}}}(\hat{\mathbf{n}}\mid\mathbf{u}_\ell)}{p_{{\boldsymbol{\hat \theta}}}(\hat{\mathbf{n}}\mid\mathbf{u}_0)}.
\label{eq:nu_ratio_only}
\end{equation}

\textbf{Step 2:}
Starting from the ratio identity Eq.~\ref{eq:nu_ratio_only} and using the diagonal Gaussian log-density
\begin{equation}
\log p(\mathbf{n}_\nu\mid\mathbf{u})
= -\frac12\sum_{j=1}^{d_\nu}\left[\log\beta_{\nu,j}(\mathbf{u})+\frac{(n_{\nu,j}-\mu_{\nu,j}(\mathbf{u}))^2}{\beta_{\nu,j}(\mathbf{u})}\right]+C,
\end{equation}
a direct expansion of the quadratic term shows that Eq.~\ref{eq:nu_ratio_only} is equivalent, for each $\ell=1,\ldots,2d_\nu$, to
\begin{equation}
\Delta\boldsymbol{\eta}(\mathbf{u}_\ell)^\top
\underbrace{\begin{pmatrix}\mathbf{n}_\nu\\ \mathbf{n}_\nu^{\odot 2}\end{pmatrix}}_{=:\mathbf{T}_\nu(\mathbf{n}_\nu)}
=
\Delta\boldsymbol{\hat\eta}(\mathbf{u}_\ell)^\top
\underbrace{\begin{pmatrix}\hat{\mathbf{n}}\\ \hat{\mathbf{n}}^{\odot 2}\end{pmatrix}}_{=:\mathbf{\hat T}(\hat{\mathbf{n}})}
+ {b}_\ell,
\label{eq:inner_l}
\end{equation}
where $\Delta\eta(\cdot)$ and $\Delta\hat\eta(\cdot)$ collect the environment-dependent coefficients, and $b_\ell$ is a scalar independent of $\mathbf{n}_\nu$.

Stacking Eq.~\ref{eq:inner_l} over $\ell=1,\ldots,2d_\nu$ yields
\begin{equation}
\mathbf{L}^\top T_\nu(\mathbf{n}_\nu) = \mathbf{\hat L}^\top \hat T(\hat{\mathbf{n}}) + \mathbf{b}.
\label{eq:stacked}
\end{equation}
By the full-rank assumption on $\mathbf{L}^\top$, left-multiplying Eq.~\ref{eq:stacked} by a left inverse of $\mathbf{L}^\top$ yields
\begin{equation}
\mathbf{T}_\nu(\mathbf{n}_\nu)=\mathbf{A}_\nu \mathbf{\hat T}(\hat{\mathbf{n}})+\mathbf{c}_\nu.
\label{eq:affine_Tnu}
\end{equation}

We now argue that $\mathbf{A}_\nu \in \mathbb{R}^{2d_\nu \times 2(d_\iota+d_\nu)}$ has full row rank.
Since $\mathbf{T}_\nu(\mathbf{n}_\nu) = (\mathbf{n}_\nu, \mathbf{n}_\nu^{\odot 2})$ is component-wise and consists of $k=2$ univariate sufficient statistics per coordinate, the standard rank argument for nonlinear ICA with exponential family conditionals applies.
In particular, by the construction in Appendix~A of \citet{sorrenson2020disentanglement}, there exist $k=2$ evaluation points at which the concatenated Jacobians of $T_\nu$ form an invertible matrix, which implies that the linear operator relating sufficient statistics must be full rank.
We therefore conclude that $\mathbf{A}_\nu$ has full row rank.

\textbf{Step 3:}
Eq.~\ref{eq:affine_Tnu} implies that each $n_{\nu,j}$ is a polynomial of degree at most $2$ in $\hat{\mathbf{n}}$, and
simultaneously $n_{\nu,j}^2$ is also a polynomial of degree at most $2$ in $\hat{\mathbf{n}}$.
Since the coordinates of $\mathbf{n}_\nu$ are mutually independent and Gaussian, any cross-dependence on multiple coordinates
of $\hat{\mathbf{n}}$ would introduce statistical dependence, yielding a contradiction. Therefore each $n_{\nu,j}$ depends on
at most one coordinate of $\hat{\mathbf{n}}$, and hence there exist a permutation $\pi$, nonzero scalars $a_j$, and constants
$b_j$ such that $n_{\nu,j}=a_j\hat n_{\pi(j)}+b_j$ for all $j$; see, e.g., \citet{sorrenson2020disentanglement}.
\end{proof}

\begin{lemma}[Block Identifiability of $\mathbf n_\iota$ via Alignment]
\label{lem:n_iota_block}
Write $\mathbf{f}(\mathbf x)=(\mathbf{f}_\iota(\mathbf x),\mathbf{f}_\nu(\mathbf x))$, where
$\mathbf f_\iota$ is the $d_\iota$-dimensional component intended to estimate the invariant
noise $\mathbf n_\iota$, denoted by $\hat{\mathbf n}_\iota$.
Consider the latent generative model in Eqs.~\ref{eq:lcm_inv}--\ref{eq:dgp}. Let $g$ be the
true decoder and $f=\hat g^{-1}$ be the estimated encoder, both smooth and invertible.

Assume that for any fixed $\mathbf n_\iota$ and any environments $\mathbf u,\mathbf u_0$,
if $\mathbf N_\nu^{(\mathbf u)}$ and $\mathbf N_\nu^{(\mathbf u_0)}$ are independent draws
from the corresponding noise distributions, then
\begin{equation}
\mathbf{f}_\iota\!\big(\mathrm{g}(\mathbf z(\mathbf n_\iota,\mathbf N_\nu^{(\mathbf u)},\mathbf u))\big)
=
\mathbf{f}_\iota\!\big(\mathrm{g}(\mathbf z(\mathbf n_\iota,\mathbf N_\nu^{(\mathbf u_0)},\mathbf u_0))\big)
\quad\text{a.s.}
\label{eq:align_as_iota_coupled}
\end{equation}
Then there exists a smooth function $\mathbf h_\iota$ such that
\begin{equation}
\mathbf{f}_\iota\!\big(\mathrm{g}(\mathbf z(\mathbf n_\iota,\mathbf n_\nu,\mathbf u))\big)
=
\mathbf h_\iota(\mathbf n_\iota)
\quad\text{a.s.},
\end{equation}
i.e., $\hat{\mathbf n}_\iota$ does not depend on $\mathbf n_\nu$.
\end{lemma}
\begin{proof}
Define the deterministic map
\begin{equation}
F(\mathbf n_\iota,\mathbf n_\nu,\mathbf u)
:=
\mathbf{f}_\iota\!\big(\mathrm{g}(\mathbf z(\mathbf n_\iota,\mathbf n_\nu,\mathbf u))\big),
\end{equation}
where $\mathbf z(\cdot)$ denotes the unique solution of the linear structural equations
given $(\mathbf n_\iota,\mathbf n_\nu,\mathbf u)$.
Since $\mathbf f$, $\mathrm g$ are smooth and invertible and $\mathbf z(\cdot)$ is linear
in $(\mathbf n_\iota,\mathbf n_\nu)$, the map $F$ is smooth in
$(\mathbf n_\iota,\mathbf n_\nu)$ for each fixed $\mathbf u$.

By assumption Eq.~\ref{eq:align_as_iota_coupled}, for any fixed $\mathbf n_\iota$ and any
$\mathbf u,\mathbf u_0$, the random variables
$F(\mathbf n_\iota,\mathbf N_\nu^{(\mathbf u)},\mathbf u)$ and
$F(\mathbf n_\iota,\mathbf N_\nu^{(\mathbf u_0)},\mathbf u_0)$ are almost surely equal.
In particular, for each fixed $\mathbf n_\iota$ and $\mathbf u$, the random variable
$F(\mathbf n_\iota,\mathbf N_\nu^{(\mathbf u)},\mathbf u)$ is almost surely constant
with respect to $\mathbf N_\nu^{(\mathbf u)}$.

We show that $F(\mathbf n_\iota,\mathbf n_\nu,\mathbf u)$ cannot depend on $\mathbf n_\nu$.
Suppose otherwise. Then there exists an environment $\mathbf u$ and an index
$l\in\{1,\dots,d_\nu\}$ such that
\begin{equation}
\frac{\partial F}{\partial n_{\nu,l}}(\mathbf n_\iota^\star,\mathbf n_\nu^\star,\mathbf u)\neq 0
\end{equation}
for some $(\mathbf n_\iota^\star,\mathbf n_\nu^\star)$.
By continuity of the partial derivative, there exists an open neighborhood
$U=U_\iota\times U_\nu$ of $(\mathbf n_\iota^\star,\mathbf n_\nu^\star)$ on which
$\partial F/\partial n_{\nu,l}$ has a fixed nonzero sign.

Fix any $\mathbf n_\iota\in U_\iota$ and $\mathbf n_{\nu,-l}\in U_{\nu,-l}$.
Then the function
\begin{equation}
t \;\mapsto\; F(\mathbf n_\iota,(t,\mathbf n_{\nu,-l}),\mathbf u)
\end{equation}
is strictly monotone on $U_{\nu,l}$.
Since $\mathbf N_\nu^{(\mathbf u)}$ has a non-degenerate Gaussian distribution,
$\mathbb P(\mathbf N_\nu^{(\mathbf u)}\in U_\nu)>0$, and conditional on this event,
two independent draws have distinct $l$-th coordinates with probability one.
Therefore,
\begin{equation}
\mathbb P\!\left(
F(\mathbf n_\iota,\mathbf N_\nu^{(\mathbf u)},\mathbf u)
\neq
F(\mathbf n_\iota,\tilde{\mathbf N}_\nu^{(\mathbf u)},\mathbf u)
\right)>0,
\end{equation}
which contradicts the almost-sure invariance implied by
Eq.~\ref{eq:align_as_iota_coupled}.
Hence $F$ cannot depend on $\mathbf n_\nu$.

Therefore, for each $\mathbf u$ there exists a function $\mathbf h_{\iota,\mathbf u}$
such that
$F(\mathbf n_\iota,\mathbf n_\nu,\mathbf u)=\mathbf h_{\iota,\mathbf u}(\mathbf n_\iota)$
almost surely.
The alignment equality across environments forces these functions to agree almost surely,
yielding a single smooth function $\mathbf h_\iota$ such that
\begin{equation}
F(\mathbf n_\iota,\mathbf n_\nu,\mathbf u)=\mathbf h_\iota(\mathbf n_\iota)
\quad\text{a.s. for all }\mathbf u.
\end{equation}
\end{proof}

\begin{lemma}[Linear Block Identifiability of $\mathbf{n}_\iota$]
\label{lem:n_iota_linear}
Under the conditions of Lemma~\ref{lem:n_iota_block} and assuming the exogenous noises $\mathbf{n}_\iota$ and $\hat{\mathbf{n}}_\iota$ follow non-degenerate Gaussian distributions, the block-wise mapping $\mathbf{h}_\iota$ must be an \textbf{affine transformation}:
\begin{equation}
    \hat{\mathbf{n}}_\iota = \mathbf{A}_\iota \mathbf{n}_\iota + \mathbf{b}_\iota,
\end{equation}
where $\mathbf{A}_\iota \in \mathbb{R}^{d_\iota \times d_\iota}$ is a constant non-singular matrix and $\mathbf{b}_\iota \in \mathbb{R}^{d_\iota}$ is a bias vector.
\end{lemma}

\begin{proof}
By Lemma~\ref{lem:n_iota_block}, there exists a smooth bijection
$\mathbf h_\iota:\mathbb R^{d_\iota}\to\mathbb R^{d_\iota}$ such that
$\hat{\mathbf n}_\iota=\mathbf h_\iota(\mathbf n_\iota)$ almost surely.
Assume $\mathbf n_\iota\sim\mathcal N(\boldsymbol\mu_\iota,\boldsymbol\Sigma_\iota)$ and
$\hat{\mathbf n}_\iota\sim\mathcal N(\hat{\boldsymbol\mu}_\iota,\hat{\boldsymbol\Sigma}_\iota)$
with $\boldsymbol\Sigma_\iota,\hat{\boldsymbol\Sigma}_\iota\succ 0$.

Let $p$ and $\hat p$ denote the densities of $\mathbf n_\iota$ and $\hat{\mathbf n}_\iota$, respectively.
By change of variables,
\begin{equation}
\log p(\mathbf n)
=
\log \hat p(\mathbf h_\iota(\mathbf n))
+ \log\big|\det J_{\mathbf h_\iota}(\mathbf n)\big|.
\label{eq:cov_gauss}
\end{equation}
Taking the gradient w.r.t. $\mathbf n$ on both sides yields
\begin{equation}
\nabla_{\mathbf n}\log p(\mathbf n)
=
J_{\mathbf h_\iota}(\mathbf n)^\top \nabla_{\hat{\mathbf n}}\log \hat p(\hat{\mathbf n})
\big|_{\hat{\mathbf n}=\mathbf h_\iota(\mathbf n)}
\;+\;
\nabla_{\mathbf n}\log\big|\det J_{\mathbf h_\iota}(\mathbf n)\big|.
\label{eq:grad_id}
\end{equation}
For a non-degenerate Gaussian, the score function is affine:
\[
\nabla_{\mathbf n}\log p(\mathbf n)= -\boldsymbol\Sigma_\iota^{-1}(\mathbf n-\boldsymbol\mu_\iota),
\qquad
\nabla_{\hat{\mathbf n}}\log \hat p(\hat{\mathbf n})
= -\hat{\boldsymbol\Sigma}_\iota^{-1}(\hat{\mathbf n}-\hat{\boldsymbol\mu}_\iota).
\]
Substituting these into \eqref{eq:grad_id} gives, for all $\mathbf n$,
\begin{equation}
-\boldsymbol\Sigma_\iota^{-1}(\mathbf n-\boldsymbol\mu_\iota)
=
- J_{\mathbf h_\iota}(\mathbf n)^\top \hat{\boldsymbol\Sigma}_\iota^{-1}
(\mathbf h_\iota(\mathbf n)-\hat{\boldsymbol\mu}_\iota)
\;+\;
\nabla_{\mathbf n}\log\big|\det J_{\mathbf h_\iota}(\mathbf n)\big|.
\label{eq:score_eq}
\end{equation}

Now differentiate \eqref{eq:score_eq} once more w.r.t. $\mathbf n$.
The left-hand side has constant Jacobian $-\boldsymbol\Sigma_\iota^{-1}$.
On the right-hand side, any nonzero second derivatives of $\mathbf h_\iota$ would generate terms
depending on $\mathbf n$ through $(\mathbf h_\iota(\mathbf n)-\hat{\boldsymbol\mu}_\iota)$ and through
the derivatives of $\log|\det J_{\mathbf h_\iota}(\mathbf n)|$.
Since the equality holds for all $\mathbf n$ and the left-hand side is constant,
it follows that $\nabla^2 \mathbf h_\iota(\mathbf n)\equiv 0$, i.e., $\mathbf h_\iota$ has constant Jacobian.
Therefore $\mathbf h_\iota$ is affine:
\[
\mathbf h_\iota(\mathbf n)=\mathbf A_\iota \mathbf n + \mathbf b_\iota,
\]
with $\mathbf A_\iota$ nonsingular because $\mathbf h_\iota$ is bijective.
\end{proof}

\section{Justification and Clarification for Assumptions in Theorem~\ref{thm: identifiability}}
\label{app: justify}
Assumptions~\ref{itm1:smooth}-\ref{itm1:rank} are originally developed by nonlinear ICA \citep{hyvarinen2016unsupervised,hyvarinen2017nonlinear,khemakhem2020variational,sorrenson2020disentanglement}, and have also been adopted in several recent works on causal representation learning, with different forms~\cite{zhang2024causal,liu2022identifying,liu2023identifiable}. Intuitively, assumption~\ref{itm1:smooth} requires that the mapping from the latent space to the observation space be information-preserving, in the sense that no latent information is irreversibly lost during the generation process. If this condition were violated, exact recovery of the latent variables would in general be impossible, regardless of the learning algorithm. Assumption~\ref{itm1:rank} ensures that the auxiliary variable (environment or intervention) induces sufficiently significant and diverse changes in the distributions of the latent components. This condition guarantees that different latent factors respond distinctly across environments, providing the variation for disentanglement.

Assumption~\ref{itm1:align} is inspired by analyses in contrastive and invariant representation learning, where optimal alignment objectives are assumed to recover representations that are invariant across different views or environments. Related theoretical analyses have been provided in recent contrastive learning works~\citep{wang2020understanding,saunshi2019theoretical,von2021self}.  In particular, \citet{von2021self} (e..g, Theorems 4.3 and 4.4) explicitly adopt a global optimality assumption to establish identifiability guarantees for self-supervised representations. Following this line of work, we likewise impose a population-level optimality assumption to connect the alignment objective with identifiability. We emphasize that such optimality assumptions are standard in the identifiability literature, and are commonly introduced to characterize what is theoretically achievable under idealized conditions, rather than to describe practical optimization behavior~\citep{hyvarinen2016unsupervised,hyvarinen2017nonlinear,khemakhem2020variational,sorrenson2020disentanglement,von2021self}.

Assumption~\ref{itm1:lambda} is aligned with identifiability analyses in causal representation learning, where variations across environments are required to disentangle causal mechanisms~\citep{lippe2022citris,lippe2022causal,brehmer2022weakly,liu2022identifying,liu2023identifiable,von2024nonparametric}. 
Intuitively, this assumption requires the existence of at least one environment $\mathbf{u}$ in which the influence of parent variables on a given node is effectively removed, thereby isolating the corresponding latent noise component. 
This can be interpreted as a hard intervention on the target variable, and is introduced to rule out degenerate causal structures in which parent effects cannot be separated from intrinsic noise.

\newpage
\section{Proof of Theorem~\ref{thm: identifiability}}
\label{app: identi}
Based on the preceding lemmas, we now establish the following identifiability results for the latent variables $\mathbf{z} = (\mathbf{z}_\iota, \mathbf{z}_\nu)$. 
\setcounter{theorem}{0} 

\begin{theorem}[Identifiability Results]
\label{thm: identifiability}
Suppose the observed variable $\mathbf{x}$ and latent causal variables $\mathbf{z} = (\mathbf{z}_\iota, \mathbf{z}_\nu)$ follow the generative model defined in Eqs.~\ref{eq:lcm_inv}--\ref{eq:dgp}, parameterized by $\boldsymbol{\theta} = (\mathrm{g}, \boldsymbol{\lambda}, \boldsymbol{\mu}, \boldsymbol{\beta})$. Let $\boldsymbol{\hat{\theta}} = (\mathrm{\hat{g}}, \hat{\boldsymbol{\lambda}}, \hat{\boldsymbol{\mu}}, \hat{\boldsymbol{\beta}})$ be the estimated parameters obtained by matching the conditional data distribution $p_\theta(\mathbf{x}|\mathbf{u}) = p_{\hat{\theta}}(\mathbf{x}|\mathbf{u})$ and minimizing the alignment loss. Assume the following conditions hold:

\begin{itemize}
    \item[\namedlabel{itm:smooth}{(i)}] \textbf{Invertibility and Smoothness:} The unknown nonlinear mapping $\mathrm{g}$ is smooth and invertible.
    
    \item[\namedlabel{itm:rank}{(ii)}] \textbf{Environmental Sufficiency:} There exist $2d_\nu$ distinct environments $\{\mathbf{u}_1, \dots, \mathbf{u}_m\}$ relative to a reference $\mathbf{u}_0$ such that the matrix 
    \begin{equation}
        \mathbf{L}^\top = [\Delta\boldsymbol{\eta}(\mathbf{u}_1), \dots, \Delta\boldsymbol{\eta}(\mathbf{u}_{2d_\nu})]^\top \in \mathbb{R}^{2d_\nu \times 2d_\nu}
    \end{equation}
    has full column rank $2d_\nu$, where (elementwise divisions)
\begin{equation}
\Delta\eta(\mathbf{u})
:=
\begin{pmatrix}
\frac{\boldsymbol{\mu}_\nu(\mathbf{u})}{\boldsymbol{\beta}_\nu(\mathbf{u})}
-\frac{\boldsymbol{\mu}_\nu(\mathbf{u}_0)}{\boldsymbol{\beta}_\nu(\mathbf{u}_0)}\\[2mm]
-\frac12\Big(\frac{1}{\boldsymbol{\beta}_\nu(\mathbf{u})}-\frac{1}{\boldsymbol{\beta}_\nu(\mathbf{u}_0)}\Big)
\end{pmatrix}\in\mathbb{R}^{2d_\nu}.
\end{equation}.
    \item[\namedlabel{itm:align}{(iii)}] \textbf{Optimal Alignment:} The alignment loss, e.g., Eq.~\ref{eq: align} attains its global minimum such that $\mathbf{f}_\iota(\mathbf{x}^{(\mathbf{u})}) = \mathbf{f}_\iota(\mathbf{x}^{(\mathbf{u}_0)})$ almost surely for any $\mathbf{u}, \mathbf{u}_0$, where $\mathbf{f}_\iota = \mathrm{\hat{g}}^{-1}_\iota$.
    \item[\namedlabel{itm:lambda} {(iv)}] \textbf{Intervention Sufficiency:} The function class of $\boldsymbol{\lambda}$ satisfies the following condition: there exists $\mathbf{u}_{i}$, such that, for all parent nodes $z_j \in\mathrm{pa}_i $ of $z_i$, $\boldsymbol{\lambda}_{j,i} =0$.
\end{itemize}

Then, the true latent causal variables $\mathbf{z}$ are related to the variables $\hat{\mathbf{z}}$ estimated by matching likelihood (i.e., the upper bound of ELBO Eq.~\ref{eq:elbo}) as follows:
\begin{enumerate}
    \item $\mathbf{z}_\nu$ is identified up to permutation and scaling, i.e., $\mathbf{z}_\nu = \mathbf{P}_\nu \hat{\mathbf{z}}_\nu + \mathbf{c}_\nu$, where $\mathbf{P}_\nu$ is a permutation matrix with scaling.
    \item $\mathbf{z}_\iota$ is identified up to a linear block transformation, i.e., $\mathbf{z}_\iota = \mathbf{A}_\iota \hat{\mathbf{z}}_\iota + \mathbf{c}_\iota$, where $\mathbf{A}_\iota \in \mathbb{R}^{d_\iota \times d_\iota}$ is a non-singular matrix.
\end{enumerate}
\end{theorem}

\begin{proof}

\textbf{Step 1: Identification Results of Latent Noise $\mathbf{n}$.}
From Lemma~\ref{lem:n_nu_perm}, the environmental noise $\mathbf{n}_\nu$ is identified component-wise due to sufficient variance across environments: $\mathbf{n}_\nu = \mathbf{P}_n \hat{\mathbf{n}}_\nu + \mathbf{b}_\nu$, where $\mathbf{P}_n$ is a permutation matrix with scaling. From Lemma~\ref{lem:n_iota_linear}, the invariant noise $\mathbf{n}_\iota$ is identified as a linear block: $\mathbf{n}_\iota = \mathbf{A}_n \hat{\mathbf{n}}_\iota + \mathbf{b}_\iota$. Combining these, the global exogenous noise vector satisfies:
\begin{equation}
    \mathbf{n} = \mathbf{A}_{global} \hat{\mathbf{n}} + \mathbf{b}, \quad \text{with } \mathbf{A}_{global} = 
    \begin{pmatrix} 
    \mathbf{A}_n & \mathbf{0} \\ 
    \mathbf{0} & \mathbf{P}_n 
    \end{pmatrix}.
    \label{eq:global}
\end{equation}

\textbf{Step 2: Derivation of Linear Latent Relationship.}
Substituting the structural equations $\mathbf{n} = (\mathbf{I} - \boldsymbol{\Lambda}(\mathbf{u}))\mathbf{z}$ and $\hat{\mathbf{n}} = (\mathbf{I} - \hat{\boldsymbol{\Lambda}}(\mathbf{u}))\hat{\mathbf{z}}$ (See Eq.~\ref{eq:z_Lambda_n}) into the noise relationship yields:
\begin{equation}
    (\mathbf{I} - \boldsymbol{\Lambda}(\mathbf{u})) \mathbf{z} = \mathbf{A}_{global} (\mathbf{I} - \hat{\boldsymbol{\Lambda}}(\mathbf{u})) \hat{\mathbf{z}} + \mathbf{b}_n.
    \label{eq:substitution}
\end{equation}
Solving Eq.~\ref{eq:substitution} for $\mathbf{z}$ thus implies a constant affine mapping:
\begin{equation}
    \mathbf{z} = \mathbf{M} \hat{\mathbf{z}} + \mathbf{c},
\end{equation}
where the transformation matrix
\begin{equation}
    \mathbf{M} = (\mathbf{I} - \boldsymbol{\Lambda}(\mathbf{u}))^{-1} \mathbf{A}_{global} (\mathbf{I} - \hat{\boldsymbol{\Lambda}}(\mathbf{u}))
    \label{eq:structural_identity}.
\end{equation}
Since $\mathbf{z} = \mathrm{g}^{-1}(\mathbf{x})$ and $\hat{\mathbf{z}} = \mathrm{\hat{g}}^{-1}(\mathbf{x})$, the mapping between $\mathbf{z}$ and $\hat{\mathbf{z}}$ is given by $\mathbf{z} = (\mathrm{g}^{-1} \circ \mathrm{\hat{g}})(\hat{\mathbf{z}})$. Because both $g$ and $\mathrm{\hat{g}}$ are assumed to be independent of the environment $\mathbf{u}$, the transformation from $\hat{\mathbf{z}}$ to $\mathbf{z}$ must also be $\mathbf{u}$-invariant. As a result, $\mathbf{M}$ is independent of $\mathbf{u}$. 

\textbf{Step 3: Identifiability Results of Latent $\mathbf{z}$.} 
\paragraph{Linear Block Identifiability of $\mathbf z_\iota$.}
From Lemma~\ref{lem:n_iota_linear}, there exist a non-singular matrix $\mathbf A_\iota$ and a vector $\mathbf b_\iota$ such that
\begin{equation}
\hat{\mathbf n}_\iota=\mathbf A_\iota \mathbf n_\iota+\mathbf b_\iota.
\end{equation}
By the structural equations for the invariant block,
\begin{equation}
\mathbf n_\iota=(\mathbf I-\boldsymbol{\lambda}_{\iota\iota})\mathbf z_\iota,
\qquad
\hat{\mathbf n}_\iota=(\mathbf I-\hat{\boldsymbol{\lambda}}_{\iota\iota})\hat{\mathbf z}_\iota.
\end{equation}
Combining these yields
\begin{equation}
(\mathbf I-\hat{\boldsymbol{\lambda}}_{\iota\iota})\hat{\mathbf z}_\iota
=
\mathbf A_\iota(\mathbf I-\boldsymbol{\lambda}_{\iota\iota})\mathbf z_\iota+\mathbf b_\iota,
\end{equation}
which implies that $\mathbf z_\iota$ is identifiable up to an invertible linear transformation (and a shift).

\subparagraph{Component-Wise Identifiability of $\mathbf{z}_\nu$}
To determine the internal structure of $\mathbf z_\nu$, we temporarily treat the invariant block
$\mathbf z_\iota$ as an aggregated variable and do not resolve its internal coordinates.
That is, for the purpose of analyzing the $\nu$-block, we regard $\mathbf z_\iota$ as a single latent
variable $z_\iota$ and correspondingly $\mathbf n_\iota$ as a single noise variable $n_\iota$.

Under this reduced representation, the block-wise linear relation
$\hat{\mathbf n}_\iota = \mathbf A_\iota \mathbf n_\iota + \mathbf b_\iota$ from Lemma~\ref{lem:n_iota_linear}
reduces to a scalar affine relation $\hat n_\iota = s\, n_\iota + b$ for some $s\neq 0$.

Consequently, in this reduced view, the global linear transformation in Eq.~\ref{eq:global}
takes the form
\begin{equation}
 \mathbf A'_{\mathrm{global}} =
 \begin{pmatrix}
 s & \mathbf 0 \\
 \mathbf 0 & \mathbf P_n
 \end{pmatrix},
\end{equation}
where $\mathbf P_n$ is a diagonal scaling-permutation matrix on the $\nu$-coordinates.

Therefore, by Lemma~\ref{lem:n_nu_perm}, the reduced noise vector $(n_\iota,\mathbf n_\nu)$ is identified
up to a joint permutation and component-wise scaling, and in particular $\mathbf n_\nu$ is identified
up to permutation and scaling independently of the internal parameterization of $\mathbf n_\iota$.

Then, under Assumption~\ref{itm:lambda}, the proof can be completed by following Step~III in the proof of Theorem~1 of \citet{liu2022identifying}. For completeness, we briefly adapt the argument to our setting below. The key idea in Step~III is to examine the block structure of the linear mapping between the true and estimated latents, exploiting the fact that (i) $\mathbf{I}-\boldsymbol{\Lambda}(\mathbf u)$ and $\mathbf{I}-\hat{\boldsymbol{\Lambda}}(\mathbf u)$ are strictly lower triangular with unit diagonal, and (ii) the $\nu$-block of the noise transformation is diagonal up to permutation and scaling by Lemma~\ref{lem:n_nu_perm}. By comparing corresponding entries of the resulting matrix identity across carefully chosen environments where individual structural coefficients vanish (as guaranteed by Assumption~\ref{itm:lambda}), one shows that all off-diagonal entries of the $\nu\nu$ submatrix must be zero, yielding identification of $\mathbf z_\nu$ up to permutation and scaling.
\end{proof}

\newpage
\section{Implementation of the Evidence Lower Bound}
\label{app:elbo}

In this appendix, we provide a general derivation of the Evidence Lower Bound (ELBO) for our generative model, valid for any intervention vector $\mathbf{u}$.

\paragraph{Generative Model.}
For an observation $\mathbf{x}$ under intervention $\mathbf{u}$, the generative model factorizes as:
\begin{equation}
p(\mathbf{x}, \mathbf{z}_\nu, \mathbf{z}_\iota \mid \mathbf{u})
= p(\mathbf{x} \mid \mathbf{z}_\nu, \mathbf{z}_\iota)\; 
  p(\mathbf{z}_\nu \mid \mathbf{u}, \mathbf{z}_\iota)\; 
  p(\mathbf{z}_\iota),
\label{eq:app_gen_model}
\end{equation}
where $\mathbf{z}_\nu$ denotes the \emph{variant} (intervention-specific) latents and $\mathbf{z}_\iota$ the \emph{invariant} latents.  
The variational posterior adopts the structured mean-field factorization:
\begin{equation}
q(\mathbf{z}_\nu, \mathbf{z}_\iota \mid \mathbf{x}, \mathbf{u})
= q(\mathbf{z}_\nu \mid \mathbf{x}, \mathbf{u})\; 
  q(\mathbf{z}_\iota \mid \mathbf{x}).
\label{eq:app_posterior}
\end{equation}

\paragraph{Derivation.}
The marginal likelihood is
\[
\log p(\mathbf{x} \mid \mathbf{u})
= \log \int 
   \frac{p(\mathbf{x}, \mathbf{z}_\nu, \mathbf{z}_\iota \mid \mathbf{u})}
        {q(\mathbf{z}_\nu, \mathbf{z}_\iota \mid \mathbf{x}, \mathbf{u})}
   q(\mathbf{z}_\nu, \mathbf{z}_\iota \mid \mathbf{x}, \mathbf{u})
   \, d\mathbf{z}_\nu d\mathbf{z}_\iota.
\]
Applying Jensen’s inequality to the logarithm yields the ELBO:
\begin{align}
\mathcal{L}_{\text{ELBO}}(\mathbf{x}, \mathbf{u})
&= \mathbb{E}_{q(\mathbf{z}_\nu, \mathbf{z}_\iota \mid \mathbf{x}, \mathbf{u})}
   \big[\log p(\mathbf{x} \mid \mathbf{z}_\nu, \mathbf{z}_\iota)\big] -  D_{\mathrm{KL}}\!\left(q(\mathbf{z}_\nu \mid \mathbf{x}, \mathbf{u}) 
   \,\|\, p(\mathbf{z}_\nu \mid \mathbf{u}, \mathbf{z}_\iota)\right)  -  D_{\mathrm{KL}}\!\left(q(\mathbf{z}_\iota \mid \mathbf{x}) 
   \,\|\, p(\mathbf{z}_\iota)\right).
\label{eq:app_elbo}
\end{align}
Considering the trade-off between reconstruction and regularization and motivated by $\beta$-VAE, we implement the ELBO as
\begin{align}
\mathcal{L}_{\text{ELBO}}(\mathbf{x}, \mathbf{u})
&= \mathbb{E}_{q(\mathbf{z}_\nu, \mathbf{z}_\iota \mid \mathbf{x}, \mathbf{u})}
   \big[\log p_\phi(\mathbf{x} \mid \mathbf{z}_\nu, \mathbf{z}_\iota)\big] - \beta_\nu\, D_{\mathrm{KL}}\!\left(q_\theta(\mathbf{z}_\nu \mid \mathbf{x}, \mathbf{u}) 
   \,\|\, p_\phi(\mathbf{z}_\nu \mid \mathbf{u}, \mathbf{z}_\iota)\right)  - \beta_\iota\, D_{\mathrm{KL}}\!\left(q_\theta(\mathbf{z}_\iota \mid \mathbf{x}) 
   \,\|\, p_\phi(\mathbf{z}_\iota)\right).
\label{eq:app_elbo1}
\end{align}

\paragraph{Priors.}
Following the latent causal generative model Eqs.~\ref{eq:lcm_inv}-\ref{eq:lcm_nu} in theoretical analysis in Sec.~\ref{sec:dgp}, we implement the priors appearing in the ELBO in Eq.~\ref{eq:app_elbo} as follows: we specify (i) a Gaussian prior for the invariant block $\mathbf{z}_\iota$, and (ii) a DAG-structured linear-Gaussian conditional prior for the perturbation-responsive block $\mathbf{z}_\nu$ conditioned on the intervention label $\mathbf{u}$ and $\mathbf{z}_\iota$.

\paragraph{Invariant Prior $p(\mathbf{z}_\iota)$.}
We place a standard diagonal Gaussian prior on the invariant latent variables:
\begin{equation}
p(\mathbf{z}_\iota) = \mathcal{N}(\mathbf{0}, \mathbf{I}).
\label{eq:prior_ziota}
\end{equation}
This choice matches the third term in Eq.~\ref{eq:app_elbo} and yields a closed-form KL divergence with the diagonal-Gaussian variational posterior $q(\mathbf{z}_\iota\mid \mathbf{x})$.

\paragraph{Variant Prior $p(\mathbf{z}_\nu \mid \mathbf{u}, \mathbf{z}_\iota)$.}
Following the linear structural causal model in Eqs.~\ref{eq:lcm_inv}--\ref{eq:lcm_nu}, we parameterize the conditional prior over $\mathbf{z}_\nu$ via a DAG-structured linear-Gaussian model with a fixed causal ordering.
Let $\boldsymbol{\Lambda}_{\nu\nu}(\mathbf{u})$ be strictly lower triangular (acyclic), and define
\begin{equation}
\mathbf{A}(\mathbf{u}) := \mathbf{I} - \boldsymbol{\Lambda}_{\nu\nu}(\mathbf{u}).
\end{equation}
We implement latent noise
\begin{equation}
\mathbf{n}_\nu \sim \mathcal{N}\!\big(\boldsymbol{\mu}_\nu(\mathbf{u}), \mathrm{diag}(\boldsymbol{\beta}_\nu(\mathbf{u}))\big),
\end{equation}
and a linear dependence on $\mathbf{z}_\iota$ through $\boldsymbol{\Lambda}_{\nu\iota}(\mathbf{u})$:
\begin{equation}
\mathbf{z}_\nu = \boldsymbol{\Lambda}_{\nu\nu}(\mathbf{u})\,\mathbf{z}_\nu
               + \boldsymbol{\Lambda}_{\nu\iota}(\mathbf{u})\,\mathbf{z}_\iota
               + \mathbf{n}_\nu.
\label{eq:scm_znu}
\end{equation}
Equivalently, since $\mathbf{A}(\mathbf{u})$ is invertible under acyclicity, this induces the conditional Gaussian prior
\begin{equation}
p(\mathbf{z}_\nu \mid \mathbf{u}, \mathbf{z}_\iota)
= \mathcal{N}\!\Big(
\mathbf{A}(\mathbf{u})^{-1}\big(\boldsymbol{\Lambda}_{\nu\iota}(\mathbf{u})\,\mathbf{z}_\iota + \boldsymbol{\mu}_\nu(\mathbf{u})\big),\;
\mathbf{A}(\mathbf{u})^{-1}\,\mathrm{diag}(\boldsymbol{\beta}_\nu(\mathbf{u}))\,\mathbf{A}(\mathbf{u})^{-\top}
\Big).
\label{eq:prior_znu}
\end{equation}
In practice, we implement this prior by sampling $\mathbf{n}_\nu$ and solving the triangular linear system in Eq.~\ref{eq:scm_znu}, which avoids explicit matrix inversion and naturally enforces the DAG constraint via the strictly lower-triangular $\boldsymbol{\Lambda}_{\nu\nu}(\mathbf{u})$.

\paragraph{Variational Posteriors.}
We use a structured variational family that mirrors the decomposition of invariant and perturbation-responsive factors in Eq.~\ref{eq:app_posterior}.

\paragraph{Invariant Posterior $q(\mathbf{z}_\iota\mid \mathbf{x})$.}
The invariant block is parameterized as a diagonal Gaussian:
\begin{equation}
q(\mathbf{z}_\iota \mid \mathbf{x})
= \mathcal{N}\!\big(\boldsymbol{\mu}_\iota'(\mathbf{x}),\;
\mathrm{diag}(\boldsymbol{\sigma}_{\iota}'(\mathbf{x}))\big),
\label{eq:q_ziota}
\end{equation}
where $(\boldsymbol{\mu}_\iota'(\cdot),\boldsymbol{\sigma}_{\iota}'(\cdot))$ are outputs of an inference network.

\paragraph{Variant Posterior $q(\mathbf{z}_\nu\mid \mathbf{x},\mathbf{u})$.}
To capture the causal dependencies among perturbation-responsive variables, we adopt an autoregressive (DAG-ordered) variational posterior:
\begin{equation}
q(\mathbf{z}_\nu \mid \mathbf{x}, \mathbf{u})
= \prod_{i=1}^{d_\nu} q\!\left(z_{\nu,i} \mid \mathbf{z}_{\nu,<i}, \mathbf{x}, \mathbf{u}\right),
\label{eq:q_znu_ar}
\end{equation}
with each conditional factor being Gaussian,
\begin{equation}
q\!\left(z_{\nu,i} \mid \mathbf{z}_{\nu,<i}, \mathbf{x}, \mathbf{u}\right)
= \mathcal{N}\!\Big(\mu_{\nu,i}'(\mathbf{x},\mathbf{u},\mathbf{z}_{\nu,<i}),\;
\sigma_{\nu,i}'(\mathbf{x},\mathbf{u},\mathbf{z}_{\nu,<i})\Big).
\label{eq:q_znu_cond}
\end{equation}
This form allows the posterior to represent nontrivial correlations in $\mathbf{z}_\nu$ induced by the directed dependencies, while maintaining tractable sampling via an ordered reparameterization. Here we can fix a predefined causal ordering over the latent variables because the identifiability result determines the perturbation-responsive factors only up to permutation. We therefore choose a canonical ordering of the latent dimensions and enforce the corresponding triangular structure during learning, without loss of generality. This follows the standard practice in identifiable latent causal models (e.g., \citet{liu2022identifying}). Algorithm~\ref{alg:contrasdagvae_forward} summarizes the training procedure, closely mirroring our implementation.

\begin{algorithm}[t]
\caption{\textbf{Training Procedure of {PerturbedVAE}}}
\label{alg:contrasdagvae_forward}
\small
\begin{algorithmic}[1]
\REQUIRE Dataset $\mathcal{D}$

\STATE $(\mathbf{x}, \mathbf{u}, \mathbf{x}^{(\mathbf{u}_0)}) \sim \mathcal{D}$
\STATE $\mathbf{h}_1 \gets \encbody(\mathbf{x})$;\quad $\mathbf{h}_2 \gets \encbody(\mathbf{x}^{(\mathbf{u}_0)})$

\STATE \textit{--- Step 1: Encode variational posteriors ---}
\STATE $(\boldsymbol{\mu}_\nu', \log{\boldsymbol{\sigma}_\nu'}^{2}) \gets g_\nu(\mathbf{h}_1,\mathbf{u})$
\STATE $(\boldsymbol{\mu}_{\iota,1}', \log{\boldsymbol{\sigma}_{\iota,1}'}^{2}) \gets g_\iota(\mathbf{h}_1)$;\quad
       $(\boldsymbol{\mu}_{\iota,2}', \log{\boldsymbol{\sigma}_{\iota,2}'}^{2}) \gets g_\iota(\mathbf{h}_2)$
\STATE $\boldsymbol{\varepsilon}_\nu, \boldsymbol{\varepsilon}_{\iota,1}, \boldsymbol{\varepsilon}_{\iota,2} \sim \mathcal{N}(\mathbf{0},\mathbf{I})$
\STATE $\tilde{\mathbf{z}}_\nu \gets \boldsymbol{\mu}_\nu' + \boldsymbol{\sigma}_\nu' \odot \boldsymbol{\varepsilon}_\nu$ \hfill {\footnotesize(sample base noise for $\mathbf{z}_\nu$)}
\STATE $\mathbf{z}_{\iota}^{(1)} \gets \boldsymbol{\mu}_{\iota,1}' + \boldsymbol{\sigma}_{\iota,1}' \odot \boldsymbol{\varepsilon}_{\iota,1}$
\STATE $\mathbf{z}_{\iota}^{(2)} \gets \boldsymbol{\mu}_{\iota,2}' + \boldsymbol{\sigma}_{\iota,2}' \odot \boldsymbol{\varepsilon}_{\iota,2}$

\STATE \textit{--- Step 2: DAG-structured transformation for $\mathbf{z}_\nu$ ---}
\STATE $\boldsymbol{\Lambda}_{\nu\nu}(\mathbf{u}) \gets f_{\nu\nu}(\mathbf{u})$ \hfill {\footnotesize(strictly lower-triangular)}
\STATE $\boldsymbol{\Lambda}_{\nu\iota}(\mathbf{u}) \gets f_{\nu\iota}(\mathbf{u})$
\STATE $\boldsymbol{\mu}_\nu(\mathbf{u}) \gets f_{\mu}(\mathbf{u})$
\STATE $\mathbf{z}_\nu \gets \big(\mathbf{I}-\boldsymbol{\Lambda}_{\nu\nu}(\mathbf{u})\big)^{-1}
\Big(\tilde{\mathbf{z}}_\nu + \boldsymbol{\Lambda}_{\nu\iota}(\mathbf{u})\,\mathbf{z}_{\iota}^{(1)} + \boldsymbol{\mu}_\nu(\mathbf{u})\Big)$
\hfill {\footnotesize(implements Eq.~\eqref{eq:scm_znu})}

\STATE \textit{--- Step 3: Reconstruction ---}
\STATE $\hat{\mathbf{x}} \gets \decoder([\mathbf{z}_\nu, \mathbf{z}_{\iota}^{(1)}])$

\STATE \textit{--- Step 4: Losses (ELBO + contrastive alignment) ---}
\STATE $\mathcal{L}_{\text{rec}} \gets \|\mathbf{x} - \hat{\mathbf{x}}\|_2^2$
\STATE $\mathcal{L}_{\text{KL-}\nu} \gets D_{\text{KL}}\!\left(q(\mathbf{z}_\nu \mid \mathbf{x},\mathbf{u}) \,\|\, p(\mathbf{z}_\nu \mid \mathbf{u},\mathbf{z}_\iota)\right)$
\STATE $\mathcal{L}_{\text{KL-}\iota} \gets D_{\text{KL}}\!\left(q(\mathbf{z}_\iota \mid \mathbf{x})\,\|\,p(\mathbf{z}_\iota)\right)
+ D_{\text{KL}}\!\left(q(\mathbf{z}_\iota \mid \mathbf{x}^{(\mathbf{u}_0)})\,\|\,p(\mathbf{z}_\iota)\right)$
\STATE $\mathcal{L}_{\text{contrast}} \gets \|\boldsymbol{\mu}_{\iota,1}'-\boldsymbol{\mu}_{\iota,2}'\|_2^2$

\STATE $(\beta_\nu, \beta_\iota, \alpha) \gets \schedule(t)$ \hfill {\footnotesize(annealing schedule)}
\STATE $\mathcal{L}_{\text{total}} \gets \mathcal{L}_{\text{rec}} + \beta_\nu\mathcal{L}_{\text{KL-}\nu} + \beta_\iota\mathcal{L}_{\text{KL-}\iota} + \alpha\mathcal{L}_{\text{contrast}}$
\STATE Update $\Theta \gets \Theta - \eta\nabla_\Theta \mathcal{L}_{\text{total}}$

\end{algorithmic}
\end{algorithm}


\newpage
\section{Experimental Details}

\subsection{Synthetic Data Experiments}
\label{app:simulation}

\paragraph{Data Generation.}
We sample data following the data-generating process described in Sec.~\ref{sec:dgp}. The concrete simulation parameters are summarized in Table~\ref{app:dgp}.

\renewcommand{\arraystretch}{1.2}
\begin{table}[H]
\centering
\caption{Simulation data generation parameters.}
\begin{tabular}{l c r}
\toprule
\textbf{Quantity} & \textbf{Symbol} & \textbf{Value} \\
\midrule
Observation dimension    & $\mathbf{x}$       & 500  \\
Latent dimension (variant)   & $\mathbf{z}_\nu$   & 4    \\
Latent dimension (invariant) & $\mathbf{z}_\iota$ & 7    \\
Intervention dimension   & $\mathbf{u}$       & 12   \\
Training size            & --                 & 3000 \\
Test size                & --                 & 1000 \\
\bottomrule
\end{tabular}
\label{app:dgp}
\end{table}

\paragraph{Training Setup and Hyperparameters.}
We use the Adam optimizer with the hyperparameters listed in Table~\ref{app:simu_hparams}.

\begin{table*}[t]
\centering
\caption{Simulation hyperparameters.}
\begin{tabular}{l c l c}
\toprule
\textbf{Hyperparameter} & \textbf{Value} & \textbf{Hyperparameter} & \textbf{Value} \\
\midrule
Batch size & 64 & $\mathbf{z}_\nu$ dim & 4 \\
Epochs & 100 & $\mathbf{z}_\iota$ dim & 7 \\
Learning rate & $1\times 10^{-3}$ & $\beta_\nu$ & $1.5\times 10^{-5}$ \\
$\beta_\iota$ & $5\times 10^{-4}$ & $\alpha_{\text{contrast}}$ & 0.1 \\
\bottomrule
\end{tabular}
\label{app:simu_hparams}
\end{table*}

\paragraph{Evaluation Metrics.}
Identifiability of the variant block $\mathbf{z}_\nu$ is quantified using the mean correlation coefficient (MCC), which measures the one-to-one correspondence between each learned latent and its ground-truth. To compute MCC, we follow these steps:

\begin{enumerate}
    \item \textbf{Compute correlation coefficients.} 
    We first compute the pairwise correlation coefficients between the ground-truth latent components $\mathbf{z}_{\nu,i}$ and the learned latent components $\hat{\mathbf{z}}_{\nu,j}$ across the test samples. Specifically, for each pair $(i,j)$ we compute the Pearson correlation coefficient
    \begin{equation}
        \rho_{i,j} \;=\; 
        \frac{\mathrm{Cov}(\mathbf{z}_{\nu,i},\,\hat{\mathbf{z}}_{\nu,j})}
        {\sigma_{\mathbf{z}_{\nu,i}}\,\sigma_{\hat{\mathbf{z}}_{\nu,j}}},
    \end{equation}
    where the covariance and standard deviations are computed over the dataset. We take absolute values $|\rho_{i,j}|$ to account for the inherent sign ambiguity of latent variables.

    \item \textbf{Solve the linear sum assignment problem.} 
    Since the learned components may be permuted relative to the ground-truth latents, we solve a linear sum assignment problem to find the optimal one-to-one matching $\pi$ between ground-truth and learned components that maximizes the total absolute correlation:
    \begin{equation}
        \pi \;=\; \arg\max_{\pi \in \mathcal{S}_d} \sum_{i=1}^{d} |\rho_{i,\pi(i)}|,
    \end{equation}
    where $\mathcal{S}_d$ denotes the set of all permutations of $\{1,\dots,d\}$ and $d$ is the latent dimension.

    \item \textbf{Compute the mean correlation coefficient (MCC).} 
    Given the optimal assignment $\pi$, we define the MCC as the average of the matched absolute correlations:
    \begin{equation}
        \mathrm{MCC} \;=\; \frac{1}{d} \sum_{i=1}^{d} |\rho_{i,\pi(i)}|.
    \end{equation}
\end{enumerate}

To assess block identifiability result over $\mathbf{z}_\iota$, we regress the ground-truth latent variables $(\mathbf{z}_\iota)$ on their learned estimates $( \hat{\mathbf{z}}_\iota)$, and report the coefficient of determination ($R^2$). High $R^2$ values close to one indicate successful block identifiability. To compute the coefficient of determination ($R^2$) for block-wise identifiability, we proceed as follows:

\begin{enumerate}
    \item \textbf{Extract learned and ground-truth latents.} 
We collect the learned latent representations $\hat{\mathbf{z}}_\iota$ and the corresponding ground-truth latent variables $\mathbf{z}_\iota$ on the test set.

\item \textbf{Regression from learned to ground-truth latents.} 
Since identifiability is only defined up to an unknown transformation at the block level, we fit a regression model from the learned latents to the ground-truth latents. Concretely, we fit a function $f$ of the form
\begin{equation}
    \mathbf{z}_\iota = f(\hat{\mathbf{z}}_\iota) + \epsilon.
\end{equation}

\item \textbf{Compute the coefficient of determination.} 
After fitting the regressor, we compute the coefficient of determination
\begin{equation}
    R^2 \;=\; 1 - \frac{\sum_{i=1}^{n} \| \mathbf{z}_{\iota,i} - \hat{\mathbf{z}}^{\,\text{pred}}_{\iota,i}\|^2}
    {\sum_{i=1}^{n} \| \mathbf{z}_{\iota,i} - \bar{\mathbf{z}}_\iota \|^2},
\end{equation}
where $\hat{\mathbf{z}}^{\,\text{pred}}_{\iota,i} = f(\hat{\mathbf{z}}_{\iota,i})$ is the predicted ground-truth latent for sample $i$, $\bar{\mathbf{z}}_\iota$ is the empirical mean of the ground-truth latents, and $n$ is the number of test samples.

    \item \textbf{Interpretation.} 
    A value of $R^2$ close to $1$ indicates that the learned latent block can be (nonlinearly) transformed to accurately recover the true latent block, consistent with block-wise identifiability. Lower $R^2$ indicates loss of information or entanglement across blocks.
\end{enumerate}

\subsection{Real-Data Experiments Details} 
\label{app: realdata}
\paragraph{Dataset from~\citep{norman2019exploring}} For real-world perturbation data, we consider the large-scale Perturb-seq dataset from~\citep{norman2019exploring}. 
It consists of 105{,}528 cells from an erythroleukemia cell line (K562) subjected to CRISPR activation~\citep{gilbert2014genome} targeting 112 genes, resulting in 105 single-gene and 131 double-gene perturbation conditions. Each perturbation condition contains between 50 and 2{,}000 cells. Across all conditions, each cell is represented as a 5,000-dimensional vector $\mathbf{x}$, corresponding to the gene expression levels.

\paragraph{Experimental Protocol for Comparing with Existing Latent Causal Models} In this setting, we follow the experimental protocol of \citet{zhang2023identifiability}. Specifically, we partition the dataset into training and test splits as follows. The training set consists of all unperturbed cells together with the 105 single-gene perturbation datasets $\mathcal{X}_1, \ldots, \mathcal{X}_{105}$. For each single-gene dataset containing more than 800 cells, we randomly hold out 96 cells to form a \emph{single-gene test set}, while the remaining cells are used for training. The \emph{double-gene test set} consists of the 112 double-gene perturbation datasets $\mathcal{X}_{106}, \ldots, \mathcal{X}_{217}$, which are entirely held out from training and used only for evaluation. This setup ensures that models are trained on unperturbed and single-gene perturbation data, but evaluated on both held-out single-gene cells and, more importantly, on combinatorial perturbations.

\paragraph{Evaluation Metrics: RMSE and Population-Level $R^2$} 
In real single-cell perturbation datasets, the ground-truth latent variables are unobservable. We therefore report the root mean squared error (RMSE) between the predicted and observed mean expression vectors, which directly measures the absolute magnitude of prediction errors in gene expression space. In addition, $R^2$ cannot be used to assess latent recovery or identifiability, and is instead used as a measure of predictive utility at the level of gene expression. Concretely, for each perturbation condition, the model first generates a population of ``virtual'' cells conditioned on the perturbation label. We compute the mean gene expression vector of these generated cells and compare it to the mean expression vector of the experimentally observed cells under the same perturbation. A linear regression is then fitted between the predicted and observed mean expression vectors, and the resulting coefficient of determination $R^2$ quantifies how well the model explains the population-level transcriptional response to the perturbation. Under this protocol, $R^2$ measures the accuracy of predicted perturbation effects in terms of explained variance rather than the recovery of latent causal variables, and therefore serves as a proxy for the practical usefulness of the learned representations in predicting gene expression changes under perturbations.

\paragraph{Hyperparameter Settings for Real Data Experiments.}
We use the Adam optimizer with hyperparameters detailed in Table~\ref{app:real_hparams}. 
\begin{table}[H]
\centering
\caption{Real Data Hyperparameters.}
\label{app:real_hparams}
\begin{tabular}{ll|ll}
\toprule
\textbf{Hyperparameter} & \textbf{Value} & \textbf{Hyperparameter} & \textbf{Value} \\
\midrule
Batch size         & 64              & Hidden dimension        & 256 \\
Epochs             & 100             & $\mathbf{z}$ dimension  & 10, 35, 75, 105 \\
Learning rate      & $1 \times 10^{-4}$ & $\alpha_{\text{contrast}}$ & 0.05 \\
$\beta_\nu, \beta_\iota$ & $1 \times 10^{-2}$ &  &  \\
\bottomrule
\end{tabular}
\end{table}

\section{Unperturbed Latent Subspace Analysis}
\label{app:latent space}

In this section, we provide additional visual evidence that the invariant latent block $\mathbf{z}_\iota$ indeed captures perturbation-invariant background transcriptional programs, and does not encode perturbation-specific information.



We examine whether this invariance property holds globally across the full test set, including both single-gene and double-gene perturbations. Figure~\ref{fig:tsne_invariant_overall} shows t-SNE embeddings of $\mathbf{z}_\iota$ for all test samples under single-gene and double-gene perturbation conditions. In both cases, cells from different perturbation regimes remain well mixed and do not form separated clusters, suggesting that $\mathbf{z}_\iota$ generalizes its invariance property beyond the training distribution.

\begin{figure}[H]
    \centering
    \begin{subfigure}[t]{0.45\textwidth}
        \centering
        \includegraphics[width=0.8\textwidth]{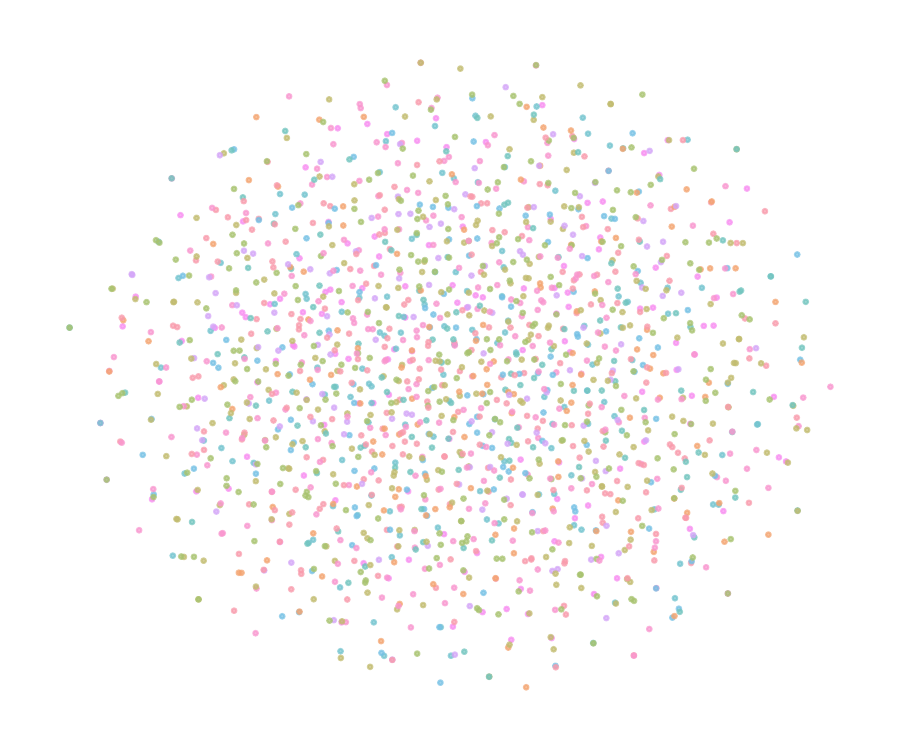}
        \caption{Single-gene test set.}
        \label{fig:tsne_single}
    \end{subfigure}
    \hfill
    \begin{subfigure}[t]{0.45\textwidth}
        \centering
        \includegraphics[width=0.8\textwidth]{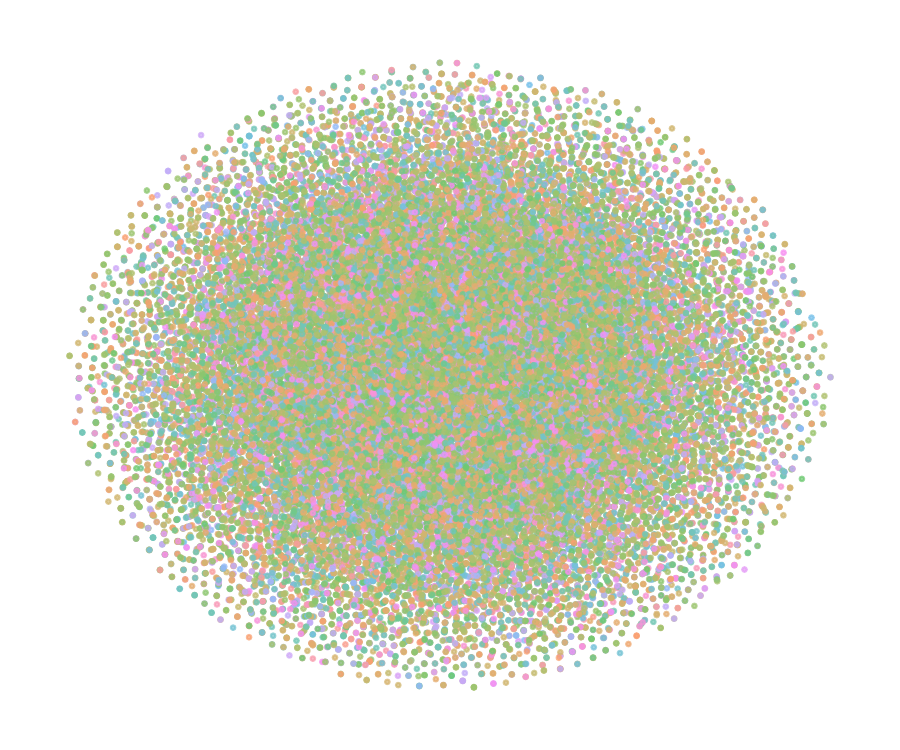}
        \caption{Double-gene test set.}
        \label{fig:tsne_double}
    \end{subfigure}
    \caption{t-SNE visualization of the invariant block $\mathbf{z}_\iota$ for single-gene (a) and double-gene (b) perturbations in the test set.}
    \label{fig:tsne_invariant_overall}
\end{figure}

Together, these results provide additional evidence that $\mathbf{z}_\iota$ captures perturbation-invariant variation in gene expression, validating the disentanglement between the invariant background block $\mathbf{z}_\iota$ and the perturbation-responsive block $\mathbf{z}_\nu$.

\newpage
\section{Perturbed Latent Subspace Analysis}

\subsection{The Learned Latent Causal Representations}\label{app:structure learning}

In this section, we examine whether the latent causal representation learned by {PerturbedVAE} is (i) identifiable and consistently aligned with external perturbations, (ii) structured rather than dense or entangled, and (iii) biologically meaningful. We provide a sequence of visualizations and analyses to validate these properties.

\paragraph{Alignment and Identifiability.}
Following \citet{zhang2023identifiability}, we first present in Figure~\ref{fig:hit} the hit map between perturbed genes and the identifiable latent causal components $\mathbf{z}_\nu(i)$ learned by our model. Columns correspond to perturbed genes, while rows denote individual causal components. Each entry highlights the component most strongly associated with a given perturbation. This visualization assesses whether perturbations are consistently mapped to specific latent components, thereby validating both identifiability and alignment of the learned representation.

\begin{figure}[H]
  \centering
  \includegraphics[width=\textwidth]{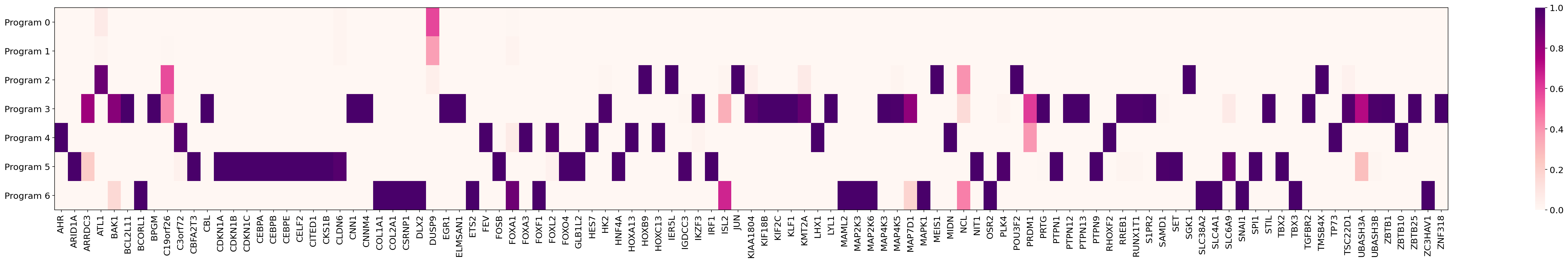} 
  \caption{Perturbed gene hits on identifiable causal components.}
  \label{fig:hit}
\end{figure}

\paragraph{Learned Causal Structure.}
To further examine the structure of the learned latent representation, we visualize the directed dependencies among the identifiable components $\mathbf{z}_\nu$. Figure~\ref{fig:threshold} (left) shows the full adjacency matrix estimated by the model prior to thresholding, where color intensity reflects the signed strength of each estimated causal effect. For interpretability, we apply a threshold ($\tau = 0.25$) to prune weak connections, yielding a sparse graph that highlights the dominant causal relations (Figure~\ref{fig:threshold}, right). This comparison demonstrates that the learned structure is neither dense nor arbitrary, but exhibits a sparse and interpretable causal backbone.

\begin{figure}[H]
    \centering
    \includegraphics[width=0.4\linewidth]{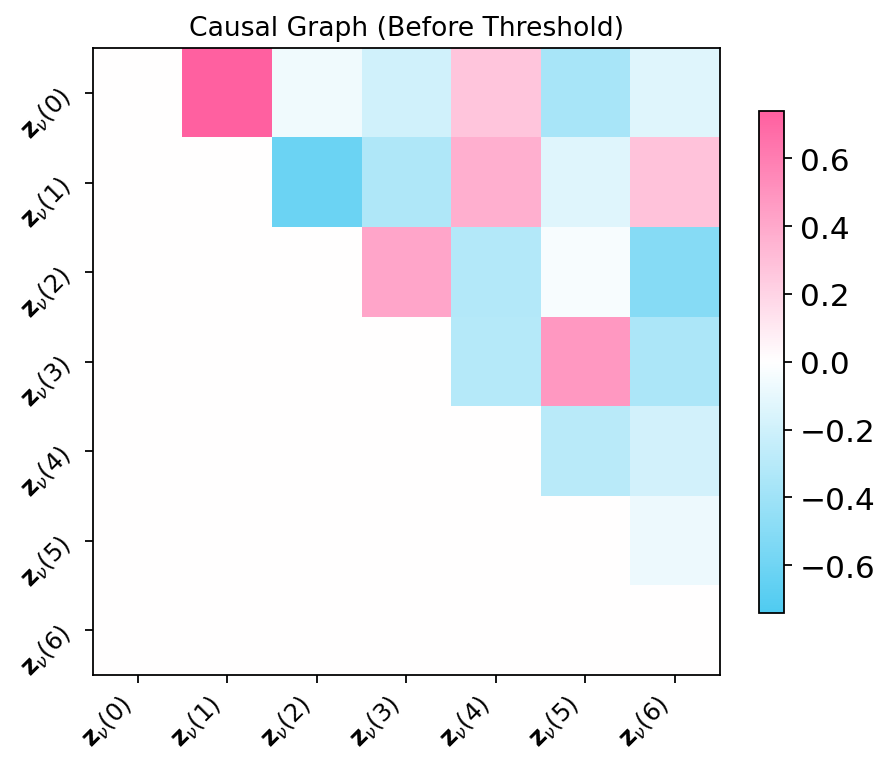}
    \includegraphics[width=0.4\linewidth]{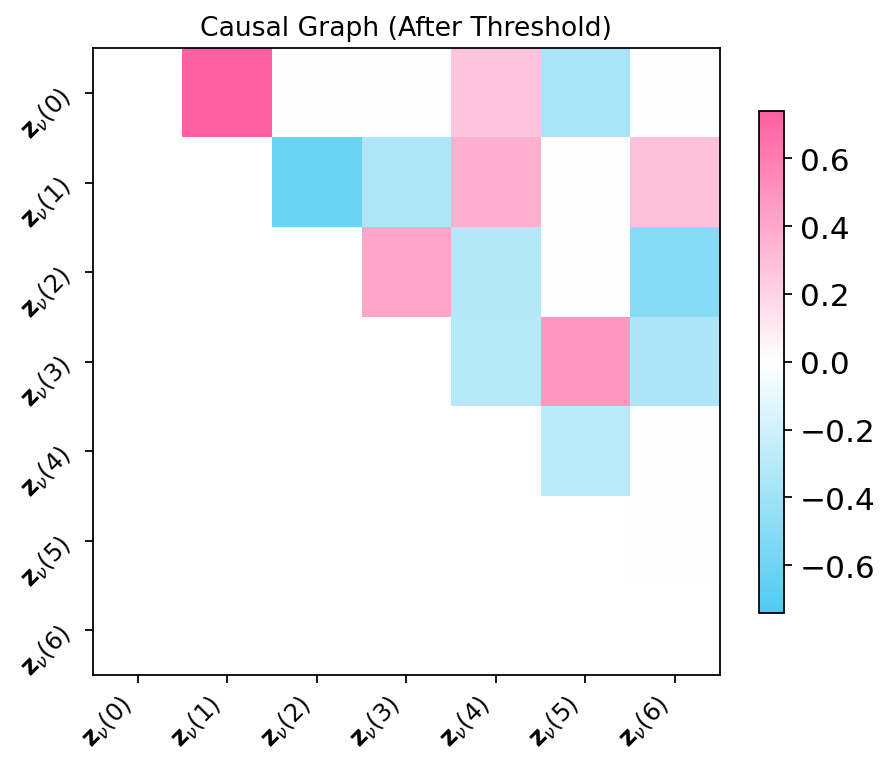}
    \caption{
        Visualization of the learned causal graph among identifiable components $\mathbf{z}_\nu$.
        \textbf{Left:} full adjacency matrix before thresholding. 
        \textbf{Right:} sparse graph after thresholding ($\tau = 0.25$).
    }
    \label{fig:threshold}
\end{figure}

\paragraph{Semantic Grounding of Latent Programs.}
We next assess whether the latent components correspond to coherent biological programs. Figure~\ref{fig:app_structure} visualizes the inferred causal graph among latent components, where each node represents a latent program and each directed edge denotes an inferred causal dependency. To provide gene-level interpretability, we map each latent program back to its associated genes. The complete mapping is reported in Table~\ref{app_tab:program_genes}, demonstrating that each latent component corresponds to a meaningful and coherent gene module.

\begin{figure}[htbp]
  \centering
  \includegraphics[width=0.4\textwidth]{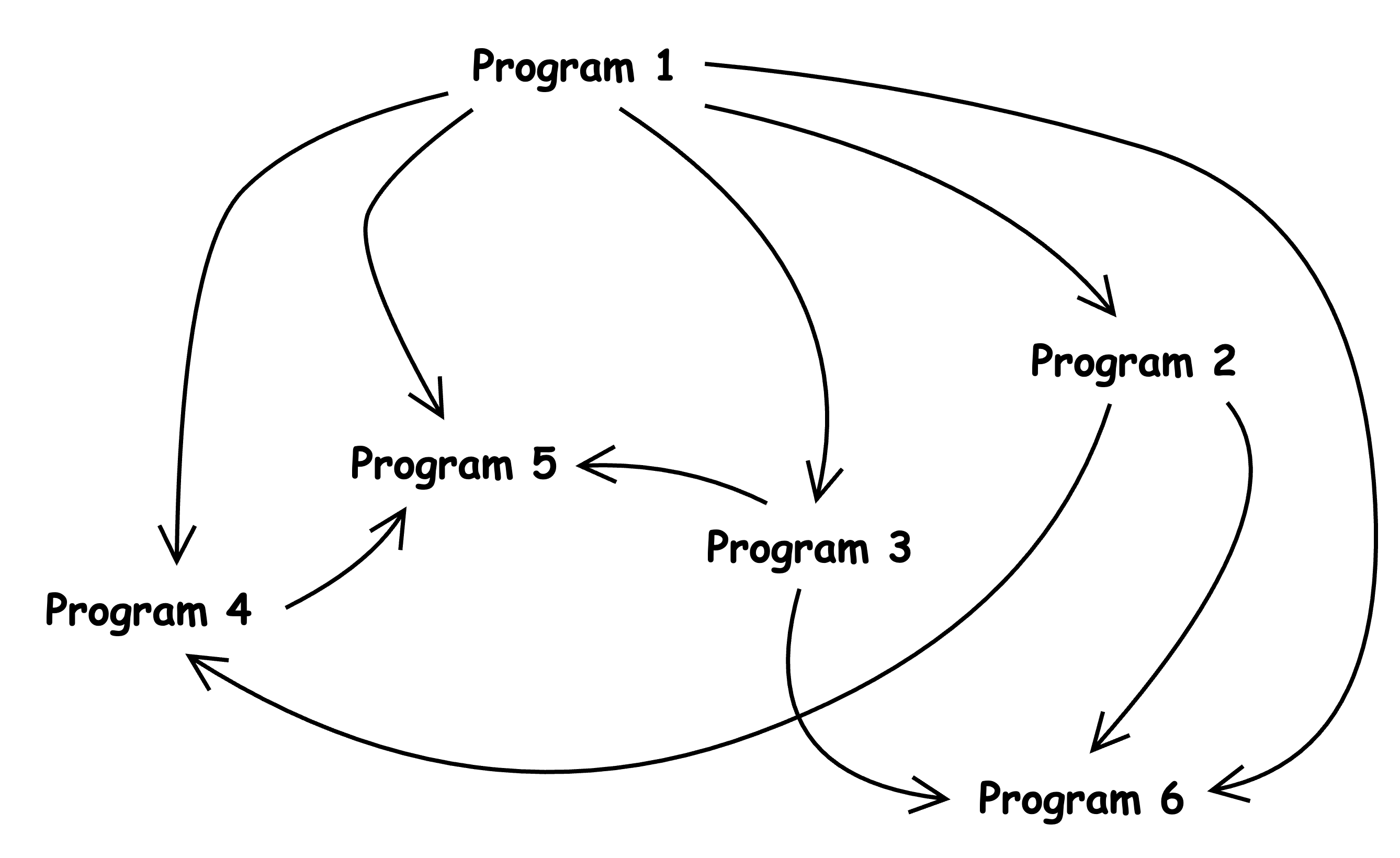} 
  \caption{Inferred causal structure among latent programs.}
  \label{fig:app_structure}
\end{figure}

\paragraph{Biological Plausibility of Inferred Edges.}
Finally, we examine whether representative inferred causal edges are consistent with known biological mechanisms. Table~\ref{app_tab:edges} summarizes several directed edges together with their mechanistic interpretations and supporting references, illustrating that the learned structure aligns with established regulatory pathways rather than reflecting spurious statistical associations.

\begin{table}[H]
\centering
\caption{Program-level representative edges: mechanistic rationale and supporting references.}
\label{app_tab:edges}
\renewcommand{\arraystretch}{1.4}
\begin{tabularx}{0.9\textwidth}{l >{\RaggedRight\arraybackslash}X >{\RaggedRight\arraybackslash}p{0.25\textwidth}}
\toprule
\textbf{Edge} & \textbf{Mechanistic rationale (summary)} & \textbf{Refs.} \\
\midrule
{\footnotesize DUSP9 $\to$ TGFBR2} & TGFBR2 activates ERK through a non-Smad branch; DUSP9 dephosphorylates ERK/JNK, attenuating this output. &
\citep{emanuelli2008overexpression, zhang2009non} \\
{\footnotesize DUSP9 $\to$ TP73} & c-Jun enhances TP73 stability; DUSP9 lowers JNK/ERK$\to$AP-1 signaling, indirectly downregulating TP73. &
\citep{koeppel2011crosstalk, emanuelli2008overexpression} \\
{\footnotesize DUSP9 $\to$ CDKN1A} & ERK$\to$ELK1/EGR1 induces p21 transcription; DUSP9 suppresses ERK phosphorylation, blunting this induction. &
\citep{lim1998stress, ragione2003p21cip1} \\
{\footnotesize DUSP9 $\to$ SNAI1} & EMT induction requires SMAD3--AP-1 cooperation; DUSP9 attenuates AP-1, weakening SNAI1 transcription. &
\citep{sundqvist2013specific, fan2025identification} \\
{\footnotesize JUN $\to$ TP73} & c-Jun stabilizes and potentiates TP73, enhancing apoptosis-related transcription. &
\citep{koeppel2011crosstalk} \\
{\footnotesize JUN $\to$ SNAI1} & AP-1 cooperates with SMAD factors to elevate SNAI1 expression in TGF-$\beta$-driven EMT. &
\citep{sundqvist2013specific, fan2025identification} \\
{\footnotesize TGFBR2 $\to$ CDKN1A} & Canonical SMAD2/3/4 downstream of TGFBR2 transactivates p21, enforcing cytostasis. &
\citep{ikushima2010tgfbeta} \\
\bottomrule
\end{tabularx}
\end{table}

\begin{table}[htbp]
\centering
\caption{Complete list of genes assigned to each latent program inferred from structure learning.}
\label{app_tab:program_genes}
\renewcommand{\arraystretch}{1.2} 
\begin{tabularx}{0.85\textwidth}{c >{\RaggedRight\arraybackslash}X}
\toprule
\textbf{Program} & \textbf{Genes} \\
\midrule
1 & {\small DUSP9} \\
2 & {\small ATL1, C19orf26, HOXB9, IER5L, JUN, MEIS1, POU3F2, SGK1, TMSB4X} \\
3 & {\small ARRDC3, BAK1, BCL2L11, BPGM, CBL, CNN1, CNNM4, EGR1, ELMSAN1, HK2, IKZF3, KIAA1804, KIF18B, KIF2C, KLF1, KMT2A, LYL1, MAP4K3, MAP4K5, MAP7D1, PRDM1, PRTG, PTPN12, PTPN13, RREB1, RUNX1T1, S1PR2, STIL, TGFBR2, TSC22D1, UBASH3A, UBASH3B, ZBTB1, ZBTB25, ZNF318} \\
4 & {\small AHR, C3orf72, FEV, FOXA3, FOXL2, HES7, HOXA13, HOXC13, LHX1, MIDN, RHOXF2, TP73, ZBTB10} \\
5 & {\small ARID1A, CBFA2T3, CDKN1A, CDKN1B, CDKN1C, CEBPA, CEBPB, CEBPE, CELF2, CITED1, CKS1B, CLDN6, FOSB, FOXO4, GLB1L2, HNF4A, IGDCC3, IRF1, NIT1, PLK4, PTPN1, PTPN9, SAMD1, SET, SLC6A9, SPI1, TBX2} \\
6 & {\small BCORL1, COL1A1, COL2A1, CSRNP1, DLX2, ETS2, FOXA1, FOXF1, ISL2, MAML2, MAP2K3, MAP2K6, MAPK1, NCL, OSR2, SLC38A2, SLC4A1, SNAI1, TBX3, ZC3HAV1} \\
\bottomrule
\end{tabularx}
\end{table}

\subsection{Perturbation Information Preservation}
\label{app:de genes}

In this section, we evaluate whether perturbation-specific information is preserved in the perturbed latent subspace rather than being suppressed by the invariant background representation. Following the main text, we use the top 20 differentially expressed (DE) genes as a perturbation-enriched readout to probe the preservation of perturbation-induced variation.

\paragraph{Metric Definitions and Empirical Observations.}
To quantify information preservation, we compute performance metrics on two complementary feature sets for each perturbation condition:
\begin{itemize}[leftmargin=5pt]
    \item\textbf{All genes}: measurements computed using the entire 5{,}000-dimensional gene expression, reflecting the global state.
    \item\textbf{DE genes}: measurements computed using the 20-dimensional sub-vectors corresponding to the top 20 most differentially expressed genes, which form a perturbation-enriched readout of the perturbed subspace.
\end{itemize}

We make the following empirical observations:
\begin{itemize}[leftmargin=5pt]
    \item \textbf{In-distribution (single-gene).} The model achieves high accuracy on both feature sets. The $R^2$ scores on the DE genes are nearly identical to those on all genes, while the RMSE on the DE subset is notably lower (Figure~\ref{fig:de_single}), indicating that perturbation-induced signals are well preserved under single-gene interventions.
    \item \textbf{Out-of-distribution (double-gene).} While the global $R^2$ on all genes remains high (around $\sim 0.98$), the $R^2$ on the DE genes exhibits a mild degradation for a subset of double-gene perturbations, with values in the $0.5$--$0.9$ range (Figure~\ref{fig:de_double}). This reflects the increased difficulty of zero-shot combinatorial extrapolation, where novel, potentially non-additive interactions must be inferred from single-gene training data.
\end{itemize}

\begin{figure}[H]
    \centering
    \begin{subfigure}[t]{0.4\textwidth}
        \centering
        \includegraphics[width=0.9\textwidth]{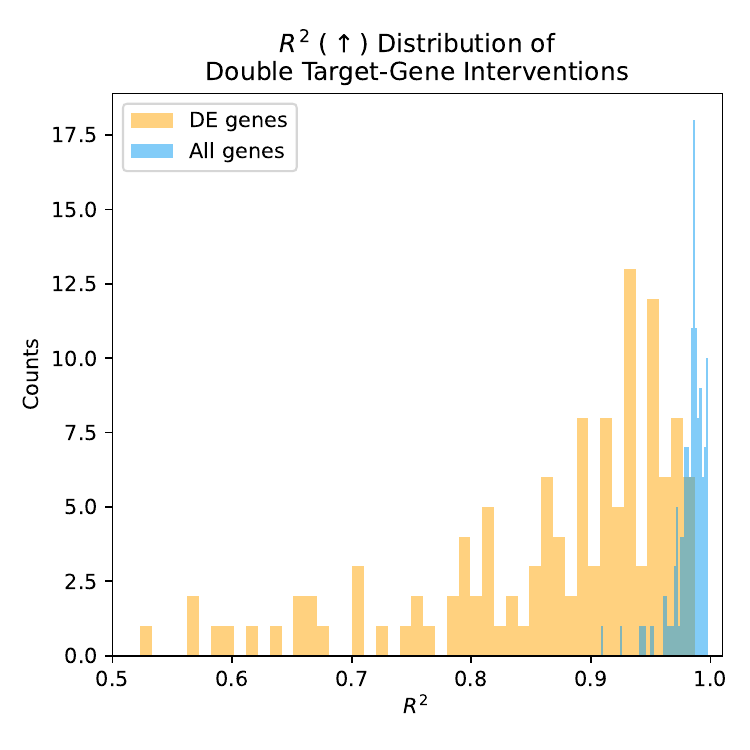}
    \end{subfigure}
    \hfill
    \begin{subfigure}[t]{0.4\textwidth}
        \centering
        \includegraphics[width=0.9\textwidth]{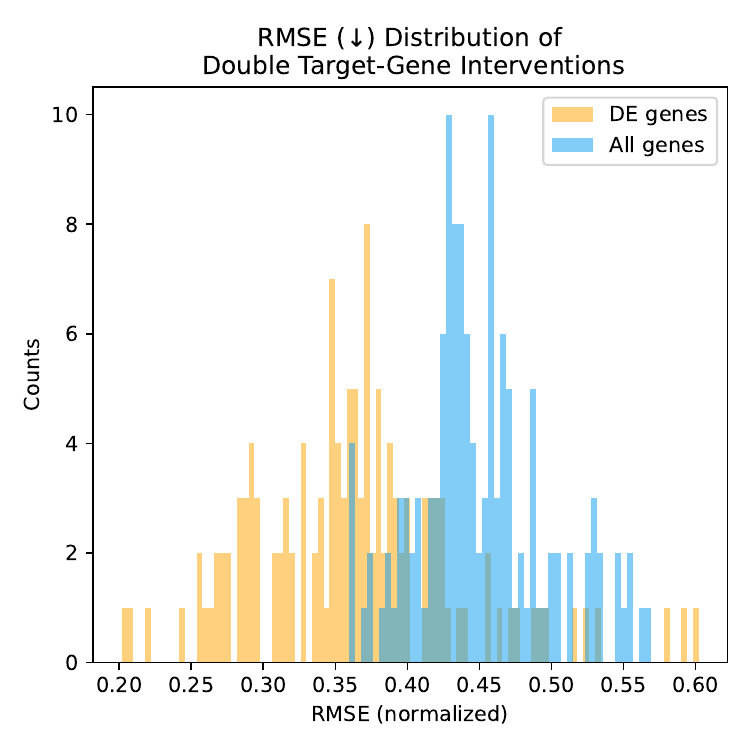}
    \end{subfigure}
    \caption{Performance on DE genes for double-gene perturbations.}
    \label{fig:de_double}
\end{figure}

\paragraph{Interpretation: Information Preservation under Invariance.}
These patterns are consistent with the intended design of the model: perturbation-invariant background variation is stabilized in the invariant block $\mathbf{z}_\iota$, while perturbation-specific information is retained in the perturbed block $\mathbf{z}_\nu$ rather than being suppressed.

\paragraph{Preservation of Perturbation Signals.}
The strong performance on DE genes in the single-gene setting indicates that the model successfully preserves perturbation-induced variation in $\mathbf{z}_\nu$ and propagates it to gene-level predictions.

\paragraph{Limits under Combinatorial Generalization.}
The moderate degradation of $R^2$ on DE genes for some double-gene perturbations reflects the intrinsic difficulty of zero-shot combinatorial causal prediction. Importantly, DE-gene RMSE typically remains low even when $R^2_{\text{DE}}$ decreases, suggesting that the model often predicts the magnitude of key expression changes reasonably well, even when fine-grained variance patterns are harder to match.

\begin{figure}[H]
    \centering
    \begin{subfigure}[t]{0.8\textwidth}
        \centering
        \includegraphics[width=\textwidth]{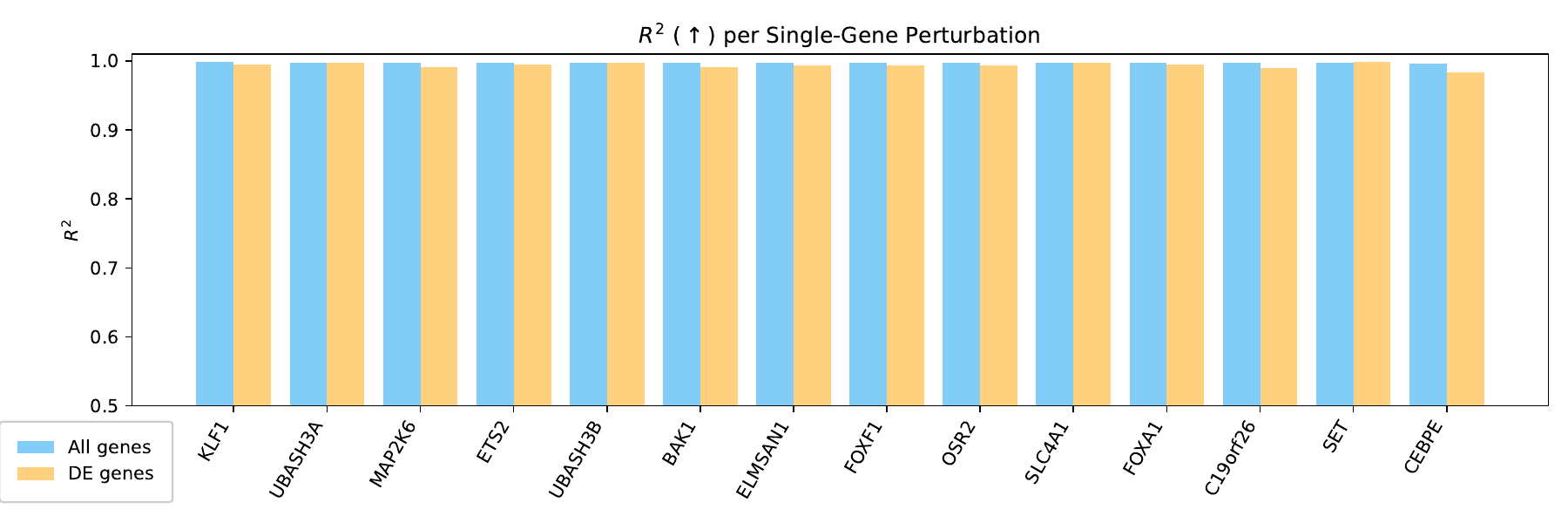}
    \end{subfigure}
    \hfill
    \begin{subfigure}[t]{0.8\textwidth}
        \centering
        \includegraphics[width=\textwidth]{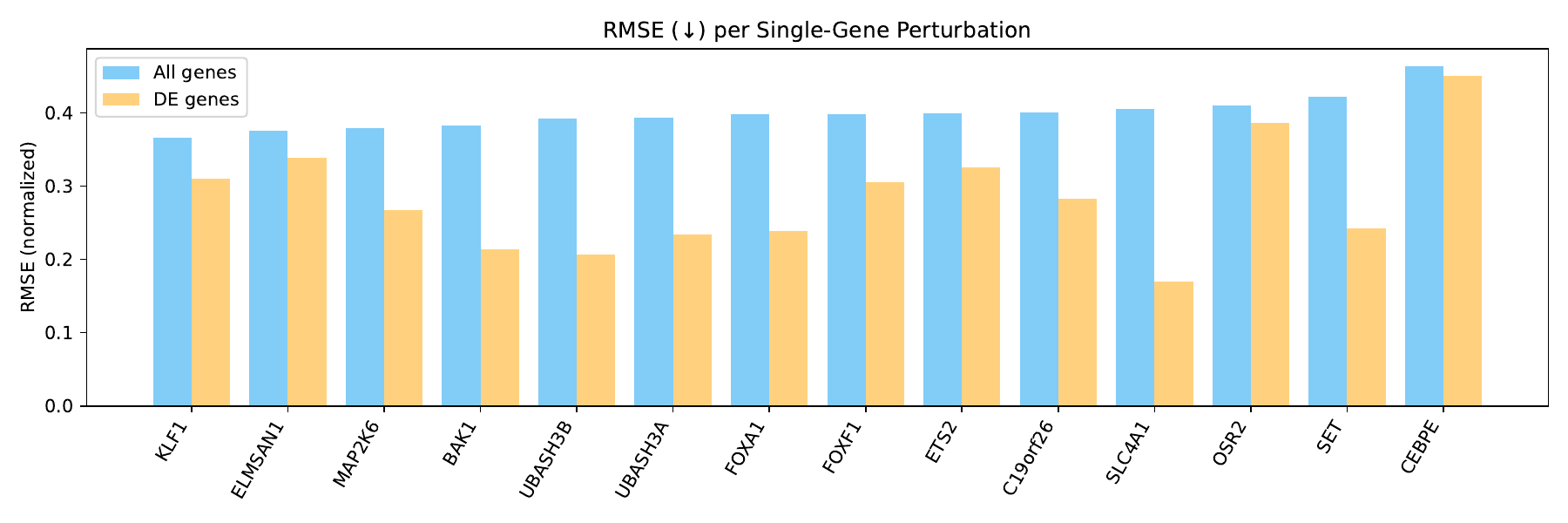}
    \end{subfigure}
    \caption{Performance on DE genes for single-gene perturbations.}
    \label{fig:de_single}
\end{figure}

\newpage
\section{Cross-Dataset Single-Gene I.I.D. Check on \texttt{Replogle2022}}
\label{app:replogle_iid}
\paragraph{Scope.}
We include \citet{replogle2022mapping} as an additional cross-dataset single-gene i.i.d. check for the latent causal representation model comparison. This experiment evaluates whether the proposed perturbation-aware representation learning mechanism remains effective on an independent perturbation dataset. It is not intended to test double-gene OOD extrapolation, which is evaluated in the main \citet{norman2019exploring} combinatorial setting.

\paragraph{Protocol.}
We follow the same latent causal model comparison protocol as in the \citet{norman2019exploring} single-gene i.i.d. evaluation. For each method, we evaluate prediction quality at both the condition level and the cell level. At the condition level, we report L2 distance and $\Delta$Pearson between predicted and observed perturbation responses. At the cell level, we report RMSE. We additionally use $R^2$ as a sanity check to assess whether the predicted single-cell profiles preserve non-degenerate cell-level variation.

\begin{table}[H]
\centering
\caption{
\texttt{Replogle2022} single-gene i.i.d. evaluation for latent causal representation models.
Condition-level metrics evaluate perturbation-level response recovery, while cell-level RMSE evaluates single-cell prediction fidelity.
}
\label{tab:replogle_iid}
\small
\setlength{\tabcolsep}{5pt}
\begin{tabular}{lccc}
\toprule
Method
& L2 $\downarrow$
& $\Delta$Pearson $\uparrow$
& RMSE $\downarrow$ \\
\midrule
SENA
& $10.916 \pm 0.064$
& $0.057 \pm 0.007$
& $0.6761 \pm 0.0038$ \\
Discrepancy-VAE
& $10.862 \pm 0.029$
& $0.056 \pm 0.006$
& $0.6731 \pm 0.0012$ \\
sVAE+
& $10.608 \pm 0.045$
& $0.136 \pm 0.002$
& $0.4905 \pm 0.0001$ \\
SAMS-VAE
& $8.489 \pm 0.061$
& $0.157 \pm 0.014$
& $0.4818 \pm 0.0005$ \\
\textbf{PerturbedVAE}
& $\mathbf{8.296 \pm 0.034}$
& $\mathbf{0.192 \pm 0.001}$
& $\mathbf{0.4815 \pm 0.0004}$ \\
\bottomrule
\end{tabular}
\end{table}

\paragraph{Results.}
As shown in Table~\ref{tab:replogle_iid}, PerturbedVAE achieves the best condition-level performance among the compared latent causal representation models, with the lowest L2 error and the highest $\Delta$Pearson. It also obtains the lowest cell-level RMSE, indicating improved single-cell prediction fidelity on an independent perturbation dataset.

\paragraph{$R^2$ Sanity Check.}
Although not included in the table, we also examine population-level $R^2$ as a sanity check for preserving meaningful single-cell variation. All compared baselines yield negative cell-level $R^2$ values, indicating that their predicted cell-level profiles are worse than the empirical-mean predictor and suffer from severe distortion at the single-cell level. In contrast, PerturbedVAE is the only method that achieves a positive cell-level $R^2$, suggesting that the learned perturbation-aware representation better preserves cell-level structure while maintaining condition-level perturbation accuracy.

\paragraph{Interpretation.}
This cross-dataset result supports the representation-learning aspect of PerturbedVAE: explicitly separating perturbation-responsive variation from dominant invariant structure improves prediction not only on \citet{norman2019exploring}, but also on an independent single-gene perturbation dataset. We emphasize that this experiment is complementary to the main \citet{norman2019exploring} double-gene OOD evaluation. The combinatorial generalization claim is supported by the \citet{norman2019exploring} double-gene benchmark, while \citet{replogle2022mapping} provides an additional i.i.d. check of representation learning robustness.

\newpage
\section{Ablation Studies and Design Validation}
\label{app:implementation}
This subsection reports a series of control and ablation experiments designed to isolate the contributions of key components of {PerturbedVAE}, including the choice of discrepancy metric, latent capacity allocation, and contrastive alignment.

\paragraph{Control: MMD-Based Discrepancy Variant.} 
To rule out the possibility that the observed performance gains are simply due to using a different discrepancy loss, we evaluate a variant of our model in which the contrastive alignment term is replaced by a maximum mean discrepancy (MMD) regularizer, denoted as {PerturbedVAE(MMD)}. This mirrors the MMD-based discrepancy formulation used in \citet{zhang2023identifiability}, enabling a direct, like-for-like comparison with Discrepancy-VAE under the same alignment criterion. This control therefore isolates whether the improvements arise from the proposed causal structure and disentanglement mechanism, rather than from the specific form of the discrepancy loss. Performance on single-gene perturbation prediction is reported in Table~\ref{tab:mmd}.

\begin{table}[H]
\centering
\caption{Evaluation of the PerturbedVAE with MMD variant on single-gene perturbation prediction.}
\resizebox{0.65\textwidth}{!}{ 
\begin{tabular}{cccc}
\toprule
\multirow{2}{*}{\textbf{Method}} & \multicolumn{3}{c}{\textbf{Metrics}}          \\ \cmidrule{2-4} 
                                 & \textbf{RMSE} & $\mathbf{R}^2$ & \textbf{MMD} \\ \midrule
\textbf{Discrepancy-VAE}~\citep{zhang2023identifiability} & $0.5558_{\pm 0.0022}$          & $0.9916_{\pm 0.0014}$          & $0.3243_{\pm 0.0050}$          \\
\textbf{PerturbedVAE} (MMD)  & $\mathbf{{0.5485}_{\pm 0.0013}}$ & 
$\mathbf{{0.9958}_{\pm 0.0003}}$ & 
$\mathbf{{0.3077}_{\pm 0.0036}}$ \\ \bottomrule
\end{tabular}}
\label{tab:mmd}
\end{table}

\paragraph{Ablation on Latent Capacity Allocation.}
We next ablate the allocation of latent capacity between the invariant block $\mathbf{z}_\iota$ and the perturbation-responsive block $\mathbf{z}_\nu$. As reported in Table~\ref{app:capacity}, together with Figure~\ref{fig:gene_r2_pair}, asymmetric allocations in which $\mathbf{z}_\iota$ is assigned substantially more capacity than $\mathbf{z}_\nu$ consistently outperform both equal-split and variant-heavy configurations on both in-distribution (single-gene) and out-of-distribution (double-gene) prediction. The invariant-heavy configuration $(z_\nu, z_\iota) = (20,85)$ achieves the lowest RMSE and highest $R^2$ across settings, indicating that sufficient capacity for modeling background transcriptional programs is critical for accurate generalization.

In contrast, when $\mathbf{z}_\iota$ is under-resourced (e.g., $(85,20)$ or $(50,55)$), performance degrades noticeably and becomes largely indistinguishable across such settings. This suggests that (i) the perturbation-responsive subspace $\mathbf{z}_\nu$ is already adequate at relatively small dimensionalities, and increasing its capacity yields diminishing returns; whereas (ii) the invariant block $\mathbf{z}_\iota$ constitutes the primary performance bottleneck.

\begin{table}[H]
\centering
\caption{Results on single- and double-gene perturbations under different capacity allocations of $\mathbf{z}_\nu$ and $\mathbf{z}_\iota$.}
\resizebox{0.6\textwidth}{!}{ 
\begin{tabular}{ccccl}
\toprule
\multirow{2}{*}{\textbf{Dimension}} & \multicolumn{2}{c}{\textbf{Single-Gene Perturbation}} & \multicolumn{2}{c}{\textbf{Double-Gene Perturbation}} \\ \cmidrule{2-5} 
                                    & \textbf{RMSE}             & $\mathbf{R}^2$            & \textbf{RMSE}             & $\mathbf{R}^2$            \\ \midrule
$\mathbf{z}_\nu = \mathbf{z}_\iota$ & $0.4084_{\pm 0.0011}$     & $0.9875_{\pm 0.0007}$     & $0.4627_{\pm 0.0003}$     & $0.9649_{\pm 0.0003}$     \\
$\mathbf{z}_\nu > \mathbf{z}_\iota$ & $0.4084_{\pm 0.0010}$     & $0.9875_{\pm 0.0007}$     & $0.4627_{\pm 0.0002}$     & $0.9649_{\pm 0.0002}$     \\
$\mathbf{z}_\nu < \mathbf{z}_\iota$ &
  \multicolumn{1}{l}{$\mathbf{0.3995_{\pm 0.0013}}$} &
  \multicolumn{1}{l}{$\mathbf{0.9977_{\pm 0.0002}}$} &
  \multicolumn{1}{l}{$\mathbf{0.4474_{\pm 0.0007}}$} &
  $\mathbf{0.9865_{\pm 0.0009}}$ \\ \bottomrule
\end{tabular}}
\label{app:capacity}
\end{table}

\paragraph{Ablation on Contrastive Alignment.}
Finally, we ablate the contrastive alignment term by comparing {PerturbedVAE} with and without the alignment loss ($\alpha=0.05$ vs.\ $\alpha=0$) under a fixed total latent dimensionality ($z=105$). As shown in Table~\ref{app:alignloss}, removing the alignment loss leads to consistent performance degradation on both single- and double-gene prediction.

Empirically, when $\alpha=0$, the invariant block $\mathbf{z}_\iota$ collapses and carries little information (with $\mathrm{KL}_\iota \to 0$), causing the effective representation to be dominated by the perturbation-responsive block $\mathbf{z}_\nu$. As a result, performance resembles that of capacity splits with $\mathbf{z}_\nu \ge \mathbf{z}_\iota$, in which the model effectively ignores the invariant subspace. In contrast, with alignment enabled, $\mathbf{z}_\iota$ remains informative and stable, preventing leakage of perturbation-specific effects into the invariant block and yielding substantially better generalization, particularly on out-of-distribution double-gene perturbations.

These results indicate that contrastive alignment is a key mechanism for sustaining the informativeness of the invariant block and maintaining the intended disentanglement, and is therefore critical for the robust performance of {PerturbedVAE}.

\begin{table}[H]
\centering
\caption{Single- and double-gene performance under contrastive alignment ablation.}
\resizebox{0.6\textwidth}{!}{ 
\begin{tabular}{ccccl}
\toprule
\multirow{2}{*}{\textbf{\begin{tabular}[c]{@{}c@{}}Contrastive\\ Alignment\end{tabular}}} &
  \multicolumn{2}{c}{\textbf{Single-Gene Perturbation}} &
  \multicolumn{2}{c}{\textbf{Double-Gene Perturbation}} \\ \cmidrule{2-5} 
       & \textbf{RMSE}         & $\mathbf{R}^2$        & \textbf{RMSE}         & $\mathbf{R}^2$        \\ \midrule
\xmark & $0.4083_{\pm 0.0011}$ & $0.9875_{\pm 0.0007}$ & $0.4626_{\pm 0.0002}$ & $0.9650_{\pm 0.0002}$ \\
\cmark &
  \multicolumn{1}{l}{$\mathbf{0.3995_{\pm 0.0013}}$} &
  \multicolumn{1}{l}{$\mathbf{0.9977_{\pm 0.0002}}$} &
  \multicolumn{1}{l}{$\mathbf{0.4474_{\pm 0.0007}}$} &
  $\mathbf{0.9865_{\pm 0.0009}}$ \\ \bottomrule
\end{tabular}}
\label{app:alignloss}
\end{table}

\section{Perspective on Additive Linear Baselines for Double-Gene Perturbation}
\label{app:perspective}

This section details the main-paper discussion on the classical additive linear baseline for combinatorial perturbation prediction. We aim to clarify (i)~why additive models can be highly competitive on standard pseudobulk benchmarks on Norman2019, as also observed by \citet{ahlmann2025deep}, and (ii)~what aspects of the prediction problem are not well captured by purely additive or regression-centric objectives. We then report supplementary results that compare {PerturbedVAE} with the additive baseline and GEARS under a strict single-gene $\rightarrow$ double-gene OOD protocol, using evaluation criteria that separately probe average responses and single-cell heterogeneity.

\paragraph{Why Additive Baselines Can Be Strong on \texttt{Norman2019}.}
Recent benchmarking results~\citep{ahlmann2025deep} highlight an important empirical characteristic of Norman2019: when evaluated on condition-level pseudobulk responses (i.e., averages over cells), even sophisticated models (including GEARS~\citep{roohani2023gears} and several foundation-model variants) may not consistently outperform a simple additive predictor under squared-error metrics. A plausible explanation is that, for many gene pairs and for many high-expression targets, the dominant component of the double-perturbation response is well approximated by a near-linear superposition of single-gene effects. In such regimes, the additive baseline benefits from a strong inductive bias that is directly aligned with the benchmark objective.

\paragraph{Perspective: Latent Causal Modeling Beyond Pseudobulk Regression.}
Our work targets a different modeling goal from regression-centric predictors that directly map perturbation labels to pseudobulk profiles. We aim to learn a structured latent causal representation that supports mechanism-level disentanglement and generalization to combinatorial interventions without using any double-perturbation supervision during training. Concretely, we model single-cell observations as arising from low-dimensional latent causal variables $\mathbf{z}$ whose dynamics are modulated by interventions $\mathbf{u}$ and corrupted by biologically meaningful stochasticity $\mathbf{n}$. In this formulation, $\mathbf{z}$ is not gene expression itself; rather, it can be interpreted as latent cellular programs, pathway activities, or regulatory modules that mediate perturbation effects. The observed expression $\mathbf{x}$ is treated as a nonlinear projection of these latent factors through the VAE decoder. 

The explicit noise term $\mathbf{n}$ reflects substantial cell-to-cell stochasticity in single-cell transcriptomics (e.g., transcriptional bursting and technical variability) that is typically suppressed by pseudobulk averaging. Modeling this stochasticity, rather than collapsing it, is intended to help separate perturbation-driven signals from unstructured variation and to make the learned latent mechanisms more suitable for downstream interpretation and OOD generalization. This perspective motivates evaluating models not only on average effects, but also on how well they capture perturbation-conditioned single-cell variability.

\paragraph{Feature Space and Evaluation Granularity.}
To balance standard comparability with causal/biological validity, we report results on two complementary gene sets and at two levels of evaluation granularity.

\textbf{Gene sets.}
(i)~\textbf{High-expression benchmark subset:} following \citet{ahlmann2025deep}, we compute metrics on the 1,000 most highly expressed genes in control cells. This subset provides a high–signal-to-noise regime and reflects the commonly used pseudobulk benchmark setting.  
(ii)~\textbf{Genome-wide profile:} we also evaluate on the full set of 5,000 genes. This is more aligned with the design goal of {PerturbedVAE}, since background cellular programs may manifest as subtle, distributed signals that are not restricted to high-expression genes and may be missed by top-expression filtering.

\textbf{Evaluation levels.}
(i)~\textbf{Condition-level pseudobulk:} we average single-cell profiles within each perturbation condition to form a pseudobulk vector and report standard metrics (Delta Pearson\footnote{We report \textit{Pearson Delta} at the pseudobulk level, defined as the Pearson correlation (across genes) between predicted and observed perturbation-induced changes relative to control.}, $L_2$, RMSE, $R^2$). These metrics quantify how well a method recovers the average conditional response associated with each perturbation and provide a direct comparison to additive baselines.  
(ii)~\textbf{Perturbation-conditioned single-cell evaluation:} for each perturbation label $\mathbf{u}$, the model produces a predicted mean expression vector, interpreted as a deterministic summary of $p_\theta(\mathbf{x}\mid \mathbf{u})$. We then compare this predicted mean against the ensemble of observed single-cell profiles under the same $\mathbf{u}$, computing RMSE and $R^2$ at the single-cell level and finally averaging across held-out double-perturbation conditions. Unlike pseudobulk metrics, this evaluation probes how well the model explains perturbation-conditioned variability in single-cell states. In {PerturbedVAE}, strong performance here is consistent with the intended disentanglement: $\mathbf{z}_{\iota}$ captures perturbation-invariant background programs, while $\mathbf{z}_{\nu}$ captures perturbation-responsive mechanisms that reshape the single-cell expression landscape (see also App.~\ref{app:de genes} for DE-gene–enriched probes).

\begin{table}[H]
\centering
\caption{Supplementary evaluation on genome-wide expression profiles (single-gene $\rightarrow$ double-gene OOD). \textbf{\textit{Note.}} A dash (--) indicates that the corresponding cell-level $R^2$ is negative (worse than a trivial condition-mean predictor) and is therefore not informative for comparing methods in this setting.}
\resizebox{\textwidth}{!}{ 
\begin{tabular}{ccccccc}
\hline
\multirow{2}{*}{\textbf{Method}} & \multicolumn{4}{c}{\textbf{Condition-level}} & \multicolumn{2}{c}{\textbf{Cell-level}} \\ \cline{2-7} 
 &
  \textbf{Prediction error ($L2$)} &
  \textbf{Delta Pearson} &
  \textbf{RMSE} &
  \textbf{$\mathbf{R}^2$} &
  \textbf{RMSE} &
  $\mathbf{R}^2$ \\ \hline
\textbf{Additive} & $2.5407_{\pm 0.0000}$ & $0.9076_{\pm 0.0000}$ & $0.0887_{\pm 0.0000}$ & $0.6431_{\pm 0.0000}$ & $0.4424_{\pm 0.0000}$ & $-$ \\
\textbf{GEARS}    & $4.6797_{\pm 0.2620}$ & $0.4631_{\pm 0.0644}$ & $0.1514_{\pm 0.0086}$ & $0.9730_{\pm 0.0032}$ & $0.5861_{\pm 0.0031}$ & $-$ \\
\textbf{PerturbedVAE}  & $3.7238_{\pm 0.0012}$ & $0.6869_{\pm 0.0005}$ & $0.1285_{\pm 0.0015}$ & $0.9965_{\pm 0.0005}$ & $0.4494_{\pm 0.0008}$ & $0.9840_{\pm 0.0011}$ \\ \hline
\end{tabular}}
\label{tab:genome-wide}
\medskip
\raggedright
\end{table}

\begin{table}[H]
\centering
\caption{Supplementary evaluation on the high-expression gene subset (single-gene $\rightarrow$ double-gene OOD).}
\resizebox{\textwidth}{!}{ 
\begin{tabular}{ccccccc}
\hline
\multirow{2}{*}{\textbf{Method}} & \multicolumn{4}{c}{\textbf{Condition-level}} & \multicolumn{2}{c}{\textbf{Cell-level}} \\ \cline{2-7} 
 &
  \textbf{Prediction error ($L2$)} &
  \textbf{Delta Pearson} &
  \textbf{RMSE} &
  \textbf{$\mathbf{R}^2$} &
  \textbf{RMSE} &
  $\mathbf{R}^2$ \\ \hline
\textbf{Additive} & $2.4906_{\pm 0.0000}$ & $0.9101_{\pm 0.0000}$ & $0.0870_{\pm 0.0000}$ & $0.6470_{\pm 0.0000}$ & $0.4332_{\pm 0.0000}$ & $-$ \\
\textbf{GEARS}    & $4.2649_{\pm 0.2044}$ & $0.5068_{\pm 0.0710}$ & $0.1381_{\pm 0.0065}$ & $0.9682_{\pm 0.0065}$ & $0.5746_{\pm 0.0018}$ & $-$ \\
\textbf{PerturbedVAE}  & $3.6491_{\pm 0.0010}$ & $0.6936_{\pm 0.0004}$ & $0.1259_{\pm 0.0013}$ & $0.9951_{\pm 0.0005}$ & $0.4411_{\pm 0.0007}$ & $0.9758_{\pm 0.0011}$ \\ \hline
\end{tabular}}
\label{tab:high-expression}
\end{table}

\paragraph{Discussion.}
Tables~\ref{tab:genome-wide}-\ref{tab:high-expression} summarize results for {PerturbedVAE}, the additive baseline, and GEARS under the strict single-gene $\rightarrow$ double-gene OOD protocol. Among the deep learning models, {PerturbedVAE} attains higher $R^2$ and lower RMSE than GEARS on both gene sets, suggesting that conditioning prediction on a learned causal latent representation can support stronger OOD generalization than directly learning a perturbation-to-expression mapping with the graph neural network baseline.

Consistent with \citet{ahlmann2025deep}, the additive baseline remains highly competitive on condition-level pseudobulk metrics, achieving the lowest $L_2$ error and highest Delta Pearson, especially on the high-expression subset on which the benchmark is commonly defined. This behavior is expected when the dominant component of the average response is approximately linear and weakly interacting, matching the inductive bias of the additive model. In such settings, more flexible models must recover this near-linearity from data while also accommodating residual non-additive effects, which can lead to a small gap on average-effect metrics even when the model is useful for other aspects of the problem.

However, condition-level metrics alone can obscure differences in how methods relate to single-cell variability. In our evaluation, both the additive baseline and GEARS yield negative cell-level $R^2$ on held-out double perturbations, indicating that their predicted condition means do not improve over a trivial condition-mean predictor when assessed against the distribution of single-cell profiles. In contrast, {PerturbedVAE} achieves substantially higher cell-level $R^2$ while remaining competitive on pseudobulk metrics. This pattern is consistent with the modeling objective of {PerturbedVAE}: learning disentangled latent factors that capture both shared background programs and perturbation-responsive mechanisms, thereby providing a coherent generative account of how perturbations reshape single-cell state distributions. From a CRL perspective, such single-cell–level fidelity is particularly relevant for downstream tasks such as mechanism interpretation, causal structure discovery, and robust OOD generalization, including settings where linear compositionality assumptions may not hold.

\section{Supplementary PCA/scVI-Based Additive Controls}
\label{app:pca_scvi_additive}

Our main comparison with simple baselines follows the additive linear model emphasized by recent large-scale benchmarking work~\citep{ahlmann2025deep}. In addition, related benchmarking studies have reported that simple representation-extractor pipelines, such as PCA-based representations combined with linear predictors, can be competitive in single-cell perturbation prediction~\citep{bendidi2024benchmarking}. We therefore include PCA/scVI-based additive controls as supplementary analyses. These experiments test whether generic low-dimensional representation extraction followed by additive extrapolation is sufficient for the double-gene OOD setting considered in this work.

We evaluate two representation-extractor controls: PCA+Additive and scVI+Additive. PCA provides a linear low-dimensional representation of the expression matrix, while scVI provides a generic nonlinear variational representation. In both cases, the extracted representation is used with the same additive prediction pipeline. These controls differ from PerturbedVAE in that they do not explicitly separate perturbation-invariant background variation from perturbation-responsive effects, nor do they learn a perturbation-conditioned latent mechanism for unseen combinatorial interventions.

\begin{table}[h]
\centering
\caption{
Supplementary PCA/scVI-based additive controls for double-gene OOD prediction.
PCA and scVI are used as generic representation extractors followed by an additive prediction pipeline.}
\label{tab:pca_scvi_additive}
\small
\setlength{\tabcolsep}{6pt}
\begin{tabular}{lcc}
\toprule
Method & RMSE $\downarrow$ & $R^2$ $\uparrow$ \\
\midrule
PCA + Additive
& $0.4907 \pm 0.0002$
& -- \\
scVI + Additive
& $0.4735 \pm 0.0002$
& -- \\
\textbf{PerturbedVAE}
& $\mathbf{0.4474 \pm 0.0007}$
& $\mathbf{0.9865 \pm 0.0009}$ \\
\bottomrule
\end{tabular}
\end{table}

As shown in Table~\ref{tab:pca_scvi_additive}, PCA+Additive and scVI+Additive provide competitive simple controls, confirming that generic low-dimensional representations can capture useful structure for perturbation prediction. However, both remain below PerturbedVAE in double-gene OOD prediction. These supplementary results do not diminish the value of simple PCA-based pipelines; rather, they clarify their role as controls for generic representation compression. In our double-gene OOD setting, generic compression followed by additive extrapolation remains insufficient to match PerturbedVAE, supporting the need for perturbation-aware modeling beyond standard low-dimensional representation extraction.

\end{document}